\documentclass{article} 
\usepackage{iclr2027_conference,times}

\usepackage{amsthm}
\usepackage{amsmath}
\usepackage{amssymb}
\usepackage{graphicx}
\usepackage{mathtools}
\usepackage{enumerate}
\usepackage{enumitem}
\usepackage{footnote}
\usepackage{float}
\usepackage{xspace}
\usepackage{multirow}
\usepackage{nicefrac}
\usepackage{wrapfig}
\usepackage{framed}
\usepackage{url}
\usepackage{units}
\usepackage[colorlinks=true, linkcolor=black, citecolor=blue, urlcolor=magenta]{hyperref}

\usepackage{blkarray}
\PassOptionsToPackage{round}{natbib}

\usepackage{tikz}
\usetikzlibrary{arrows,chains,matrix,positioning,scopes,arrows.meta}

\let\hat\widehat

\newtheorem{lemma}{Lemma}

\newtheorem{theorem}{Theorem}
\newtheorem{corollary}{Corollary}
\newtheorem{remark}{Remark}

\newtheorem{assumption}{Assumption}

\newcommand{\DefinedAs}[0]{\mathrel{\mathop:}=}

\newcommand{\bx}{\boldsymbol{x}}

\DeclareMathOperator*{\argmin}{argmin}

\DeclareFontFamily{OMX}{MnSymbolE}{}
\DeclareFontShape{OMX}{MnSymbolE}{m}{n}{
    <-6>  MnSymbolE5
   <6-7>  MnSymbolE6
   <7-8>  MnSymbolE7
   <8-9>  MnSymbolE8
   <9-10> MnSymbolE9
  <10-12> MnSymbolE10
  <12->   MnSymbolE12}{}
\DeclareSymbolFont{mnlargesymbols}{OMX}{MnSymbolE}{m}{n}
\SetSymbolFont{mnlargesymbols}{bold}{OMX}{MnSymbolE}{b}{n}
\DeclareMathDelimiter{\llangle}{\mathopen}{mnlargesymbols}{'164}{mnlargesymbols}{'164}
\DeclareMathDelimiter{\rrangle}{\mathclose}{mnlargesymbols}{'171}{mnlargesymbols}{'171}

\usepackage{prettyref}

\newcommand{\savehyperref}[2]{\texorpdfstring{\hyperref[#1]{#2}}{#2}}
\newrefformat{eq}{\savehyperref{#1}{\textup{(\ref*{#1})}}}
\newrefformat{eqn}{\savehyperref{#1}{Equation~\eqref{#1}}}
\newrefformat{lem}{\savehyperref{#1}{Lemma~\ref*{#1}}}
\newrefformat{def}{\savehyperref{#1}{Definition~\ref*{#1}}}
\newrefformat{line}{\savehyperref{#1}{line~\ref*{#1}}}
\newrefformat{thm}{\savehyperref{#1}{Theorem~\ref*{#1}}}
\newrefformat{corr}{\savehyperref{#1}{Corollary~\ref*{#1}}}
\newrefformat{cor}{\savehyperref{#1}{Corollary~\ref*{#1}}}
\newrefformat{sec}{\savehyperref{#1}{Section~\ref*{#1}}}
\newrefformat{app}{\savehyperref{#1}{Appendix~\ref*{#1}}}
\newrefformat{ap}{\savehyperref{#1}{Appendix~\ref*{#1}}}
\newrefformat{assum}{\savehyperref{#1}{Assumption~\ref*{#1}}}
\newrefformat{as}{\savehyperref{#1}{Assumption~\ref*{#1}}}
\newrefformat{ex}{\savehyperref{#1}{Example~\ref*{#1}}}
\newrefformat{fig}{\savehyperref{#1}{Figure~\ref*{#1}}}
\newrefformat{alg}{\savehyperref{#1}{Algorithm~\ref*{#1}}}
\newrefformat{rem}{\savehyperref{#1}{Remark~\ref*{#1}}}
\newrefformat{conj}{\savehyperref{#1}{Conjecture~\ref*{#1}}}
\newrefformat{prop}{\savehyperref{#1}{Proposition~\ref*{#1}}}
\newrefformat{proto}{\savehyperref{#1}{Protocol~\ref*{#1}}}
\newrefformat{prob}{\savehyperref{#1}{Problem~\ref*{#1}}}
\newrefformat{claim}{\savehyperref{#1}{Claim~\ref*{#1}}}
\newrefformat{que}{\savehyperref{#1}{Question~\ref*{#1}}}
\usepackage{bbm}
\usepackage{titletoc}
\usepackage{graphicx}
\usepackage{wrapfig}
\usepackage{subcaption}
\usepackage{booktabs}
\usepackage[utf8]{inputenc} 
\usepackage[T1]{fontenc}    
\usepackage{algorithm}     
\usepackage{algpseudocode} 
\usepackage{hyperref}       
\usepackage{url}            
\usepackage{amsfonts}       
\usepackage{nicefrac}       
\usepackage{microtype}      
\usepackage{xcolor}         
\usepackage{multirow}
\usepackage{algorithm}
\usepackage{algpseudocode}
\usepackage{subcaption}

\usepackage{xcolor}
\usepackage{bm}
\usepackage{amsmath,amssymb,amsthm,mathtools}
\usepackage{bm}

\usepackage{hyperref}
\usepackage{url}
\usepackage{algorithm}
\usepackage{algpseudocode}
\usepackage{subcaption}

\definecolor{mycyan}{RGB}{0,204,204}
\definecolor{mygreen}       {rgb} {0.10,0.50,0.50}
\definecolor{mylightgreen}  {RGB} {50,168,82}
\definecolor{lightred}      {rgb} {1.00,0.80,0.80}
\definecolor{lightblue}     {rgb} {0.80,0.80,1.00}
\definecolor{mediumred}     {rgb} {1.00,0.60,0.60}
\definecolor{mediumblue}    {rgb} {0.40,0.50,0.70}
\definecolor{pennblue}{cmyk}{1,0.65,0,0.30}
\definecolor{pennred}{cmyk}{0,1,0.65,0.34}
\definecolor{myred}{RGB}{199,107,102}

\title{ConicFlow: CONstrained Inference-time\\ optimal Control for FLOW matching Models}

\author{%
\mdseries
\begin{tabular}[t]{@{}p{0.47\textwidth}p{0.47\textwidth}@{}}
\textbf{Shervin Khalafi}\thanks{Equal contribution.} &
\textbf{Kasra Fallah}\footnotemark[1] \\
University of Pennsylvania & Columbia University \\
{\small\texttt{shervink@engineering.upenn.edu}} &
{\small\texttt{kasra.fallah@columbia.edu}} \\[2ex]
\textbf{Sung Cho} & \textbf{Dongsheng Ding} \\
University of Pennsylvania & University of Tennessee, Knoxville \\
{\small\texttt{sungcho9@engineering.upenn.edu}} &
{\small\texttt{dongshed@utk.edu}} \\[2ex]
\textbf{James Anderson} & \textbf{Alejandro Ribeiro} \\
Columbia University & University of Pennsylvania \\
{\small\texttt{ja3451@columbia.edu}} &
{\small\texttt{aribeiro@engineering.upenn.edu}}
\end{tabular}}

\iclrfinalcopy 
\begin{document}

\maketitle

\lhead{}
\renewcommand{\headrulewidth}{0pt}

\begin{abstract}

We propose ConicFlow, which frames constrained sampling from a flow matching model at inference time as an optimal control problem that balances deviation from the pretrained trajectory against reward, subject to hard terminal constraints. We discretize this problem and adapt receding-horizon control, optimizing only a short window of the trajectory at each step for computational tractability. We address the resulting sequence of receding-horizon subproblems with a receding-horizon primal–dual algorithm. Its dual multipliers capture the difficulty of satisfying each constraint and are refined throughout sampling, steering generation toward feasible final samples. We further propose a terminal-estimate variant that avoids accumulating local truncation errors by evaluating the reward and constraints at the predicted terminal point rather than at the end of the current horizon. Under standard regularity assumptions, we establish finite-horizon guarantees on suboptimality. Under convex-affine structure, we derive a sharper bound that scales linearly with omitted horizon length, and we show that the terminal-estimate variant is provably tighter when the flow dynamics change slowly.  Experiments in both low- and high-dimensional settings with a recent large-scale flow matching model show that ConicFlow achieves high reward and strong constraint satisfaction while deviating less from the pretrained flow than recent baselines. \footnote{Our implementation code can be found at: \url{https://github.com/shervinkhalafi/conicflow}}

\end{abstract}
\section{Introduction}

Diffusion models and flow matching models generate high-quality samples across a variety of modalities~\citep{ho2020denoising,lipman2022flow,rombach2022high}. However, sample quality alone is often insufficient: many applications require samples that both achieve high reward and satisfy hard constraints. Rewards can encode user preferences, such as image
aesthetics evaluated by models trained on human
feedback~\citep{wu2023human}, consistency with observations in inverse
problems~\citep{ben2024d}, or safety-related
criteria~\citep{schramowski2023safe}.
Optimizing these rewards alone, however, does not guarantee satisfaction of hard constraints. The challenge is to
achieve high reward and terminal feasibility while limiting deviations
from the pretrained dynamics.

In both theory and practice, ensuring that samples generated by these models satisfy hard constraints is underexplored. Existing approaches fall broadly into two categories. One approach is to adapt model parameters through
fine-tuning~\citep{black2024training,domingo2025adjoint}, but it can be prohibitively expensive. Furthermore, such an approach would require finetuning again for each new reward/constraint. Guidance, inference-time steering, and projection based methods for constraint satisfaction~\citep{ben2024d,christopher2026constraint}, avoid the computational overhead of fine-tuning, usually at the cost of additional overhead during sampling. Our approach is an inference-time approach. We provide extended discussion on related approaches in Appendix~\ref{app: related work}.

A pre-trained flow matching model describes the dynamics of the sample evolving across time. We need to deviate from these dynamics in order to satisfy constraints and high reward at the endpoint, meanwhile ideally we would want these deviations to be small, because we still want to end up reasonably close to the data manifold. Given this description, the problem naturally lends itself to an optimal control formulation, where the deviation from the flow dynamics can be viewed as the control input that we are optimizing over. This leads to the continuous-time constrained inference problem~\eqref{eq:continuous-problem} where $r$: $\mathbb{R}^d\rightarrow\mathbb{R}$ encodes a soft preference over
the generated sample, while
$h$: $\mathbb{R}^d\rightarrow\mathbb{R}^m$ represents $m$ hard terminal
requirements. 
\begin{equation}
\tag{P-C}
\label{eq:continuous-problem}
\begin{aligned}
    \min_{u(t)}
    \quad &
    J_C(u) 
    \; \DefinedAs \;
    \int_0^1
    \frac{\|u(t)\|^2}{2\lambda}\,dt
    -
    r(x_1^u)
    \\
    \mathrm{s.t.}
    \quad &
    \dot{x}_t^u
    \; = \;
    b_t(x_t^u)+u(t),
    \quad
    \text{given } x_0,
    \\
    &
    h(x_1^u) \;\leq\; 0.
\end{aligned}
\end{equation}
Full-horizon constrained control can be computationally demanding,
particularly for high-dimensional flow models. We therefore adopt
receding-horizon control (RHC), which replaces the full problem by a sequence
of shorter optimizations and replans after each intervention.
\citep{garcia1989model,mayne2000constrained,mattingley2011receding,bemporad2000explicit}.
This improves tractability at the cost of horizon-truncation error, motivating
our formulations~\eqref{eq:local-terminal-rh} and~\eqref{eq:terminal-rollout-rh} in Section~\ref{sec:methodology}.

Specifically, in RHC, rather than optimizing the full trajectory at once, at each step of the dynamics (in the case of flow matching, at each sampling step) we optimize over a short window and then apply the first optimized control action before shifting the window forward by one step and repeating this procedure until we reach terminal time. Note that in~\eqref{eq:continuous-problem}, both the reward and the constraints are evaluated only at the terminal time. Classical receding-horizon control introduced in~\eqref{eq:local-terminal-rh} instead evaluates them at the end of each truncated window, which creates a mismatch with the original objective. We therefore propose~\eqref{eq:terminal-rollout-rh}, which uses the pretrained flow map~\citep{boffi2025consistency} to predict the terminal state reached from the end of each window and evaluates the reward and constraints at this prediction. To the best of our knowledge, this is the first work to use flow maps within receding-horizon control, allowing each subproblem to account for the full-horizon objective without extending the control horizon. Our contributions are fourfold.

\begin{enumerate}
    \item[(i)] \textbf{Constrained-control formulation and theory.} We formulate reward-guided generation with hard terminal constraints as an optimal control problem that penalizes deviations from pretrained dynamics, and bound the errors introduced by discretization and horizon truncation.
    \item[(ii)] \textbf{Terminal prediction through flow maps.} We propose a variant of receding-horizon control~\eqref{eq:terminal-rollout-rh} that uses the pretrained flow map to predict the uncontrolled continuation beyond each optimization window. This enables evaluation of rewards and constraints at the final sampling time without extending the optimized control horizon. 
    \item[(iii)] \textbf{Primal-dual algorithm and convergence guarantees.} We develop a receding-horizon primal-dual algorithm and, under affine dynamics, establish sublinear convergence rates. We further quantify how local optimization errors propagate and establish convergence to the exact receding-horizon trajectory under uniform bounds. 
    \item[(iv)] \textbf{Experiments.} On synthetic data and image editing, ConicFlow improves the reward-control trade-off and achieves near-$100\%$ empirical feasibility with up to a $20$-fold reduction in control energy relative to the
    baseline. 
    
\end{enumerate}

\paragraph{Preliminaries on flow matching and flow maps.}
A flow matching model generates samples by evolving an initial state $x_0\sim p_0$ according to $\dot{x}_t=b_t(x_t)$ for $t\in[0,1]$, where $b$: $[0,1]\times\mathbb{R}^d\to\mathbb{R}^d$ is a learned velocity field. Training aims to transport the source distribution $p_0$ to the target distribution $p_{\mathrm{data}}$ at $t=1$, and can be formulated using stochastic interpolants~\citep{albergo2025stochastic}. Assuming the ODE has unique solutions over the time interval, its \emph{flow map} $X_{s,t}:\mathbb{R}^d\to\mathbb{R}^d$ sends a state at time $s$ to the corresponding state at time $t$: $X_{s,t}(x_s)=x_t$, with $X_{s,s}(x)=x$ for $0\leq s\leq t\leq1$. Learned approximations of these solution operators enable direct prediction across time intervals, accelerating sampling without explicitly integrating every intermediate step~\citep{boffi2025consistency}.

\section{Problem Formulation}
\label{sec:problem_formulation}

\paragraph{Problem statement.} We consider the problem of steering a pretrained
flow model toward samples that achieve a desired terminal reward while
simultaneously satisfying hard output constraints.
Rather than modifying or retraining the underlying flow model, we introduce an
inference-time control $u_t$ that perturbs the nominal vector field.
This leads to the continuous-time constrained inference problem~\eqref{eq:continuous-problem}. The reward $r$ and the constraint $h$ play distinct roles: The constraint specifies strict properties that the generated output must satisfy, while the reward ranks feasible terminal samples.
The quadratic energy term penalizes deviation from the pretrained generative
dynamics. In particular, smaller $\lambda$ places a larger penalty on
intervention and keeps the controlled trajectory closer to the nominal flow,
while larger $\lambda$ permits stronger inference-time steering.

\paragraph{Discrete-time constrained inference.}
In practice, flow-based inference is carried out using finitely many
vector-field evaluations. We introduce a uniform forward-Euler
discretization of the flow dynamics in~\eqref{eq:continuous-problem} over 
$t_0<t_1<\cdots<t_N$: $x_{k+1}
=
x_k
+
\Delta t
\bigl(
b_{t_k}(x_k)+u_k
\bigr)$, where $t_0=0$, $t_N=1$, and
$\Delta t = 1/N$. By applying Euler discretization, we approximate~\eqref{eq:continuous-problem} as a discrete-time constrained inference problem, 
\begin{equation}
\tag{P-D}
\label{eq:discrete-problem}
\begin{aligned}
\min_{u_{0:N-1}}
\quad &
J_D(u) 
    \; \DefinedAs \;
\sum_{k\,=\,0}^{N-1}
\Delta t
\frac{\|u_k\|^2}{2\lambda}
-
r(x_N^u)
\\
\mathrm{s.t.}
\quad &
x_{k+1}
\;=\;
x_k+
\Delta t
\bigl(
b_{t_k}(x_k)+u_k
\bigr), \quad
    \text{given } x_0,
\\
&
h(x_N^u) \;\leq\; 0.
\end{aligned}
\end{equation}
The constrained inference problem~\eqref{eq:discrete-problem} is finite-dimensional. To capture the approximation error, we prove that the optimal value of the discrete-time problem~\eqref{eq:discrete-problem} approximates that of the continuous-time problem~\eqref{eq:continuous-problem} with error $O(1/N)$ in Lemma~\ref{lem:discretization-error-main}. Let $u_C^\star$ be an optimal control for~\eqref{eq:continuous-problem}, and $u_D^\star(N)$ be an optimal control
for the $N$-step discretized problem~\eqref{eq:discrete-problem}. 

\begin{assumption}[Regularity, feasibility, and repairability]
\label{ass:main-discretization}
\textnormal{(i)} On the relevant energy-bounded reachable set, $b_t(x)$ is
$L_b$-Lipschitz in state and $L_b^t$-Lipschitz in time, while $r$ and $h$
are $L_r$- and $L_h$-Lipschitz, respectively;
\textnormal{(ii)} there exist $u^0,u_C^\star\in L^2$ such that
$h(x_1^{u^0})\leq0$, $J_C(u^0)<\infty$, and
$u_C^\star\in\argmin_{u:\,h(x_1^u)\leq0}J_C(u)$;
\textnormal{(iii)} there exist $\kappa,\rho_0>0$, independent of $N$, such
that any relevant continuous- or discrete-time control with terminal violation at
most $\rho_0$ admits a feasible perturbation whose control norm is at most
$\kappa$ times the violation.
\end{assumption}

We state Assumption~\ref{ass:main-discretization} in a compact form with its precise statement deferred to Appendix~\ref{app:proof-discretization}.

\begin{lemma}[Discretization error]
\label{lem:discretization-error-main}
Let Assumption~\ref{ass:main-discretization} hold. Then, there exist constants
$N_1\in\mathbb{N}$, and $C_{\mathrm{disc}}>0$ that is independent of $N$, such that
for all $N\geq N_1$,
\begin{equation}
\label{eq:disc-main-bound}
\left|
J_D\bigl(u_D^\star(N)\bigr)
-
J_C(u_C^\star)
\right|
\leq
\frac{C_{\mathrm{disc}}}{N}.
\end{equation}
\end{lemma}
Lemma~\ref{lem:discretization-error-main}'s proof is provided in Appendix~\ref{app:proof-discretization}. In Lemma~\ref{lem:discretization-error-main}, we show that the objective gap between the discrete- and continuous-time problems,~\eqref{eq:discrete-problem} and~\eqref{eq:continuous-problem}, decays to zero at a rate $O(1/N)$. 

\paragraph{Path-space constrained inference.} Although~\eqref{eq:discrete-problem} is finite-dimensional, solving it
directly requires a full-horizon constrained nonlinear optimization over all
$N$ control inputs, with repeated evaluations of the pretrained flow map.
For high-dimensional flow maps, this can be computationally prohibitive at
inference time. We therefore rewrite Problem~\eqref{eq:discrete-problem} in terms of the generated
trajectory rather than the control sequence. The control input $u_k$ in~\eqref{eq:discrete-problem} is directly expressed by $u_k
= (x_{k+1}-x_k)/{\Delta t}
-
b_{t_k}(x_k)$. Thus, we can remove the explicit flow dynamics in~\eqref{eq:discrete-problem} and introduce a discrete-time path-space problem as follows.
\begin{equation}
\tag{P-DP}
\label{eq:path-space-problem}
\begin{aligned}
\min_{x_{1:N}}
\quad & J_D(u) 
    \; \DefinedAs \;
\sum_{k\,=\,0}^{N-1}
\frac{1}{2\lambda\Delta t}
\left\|
x_{k+1}
-
x_k
-
\Delta t\,b_{t_k}(x_k)
\right\|^2
-
r(x_N)
\\
\mathrm{s.t.}
\quad &
h(x_N)\leq 0, \quad
    \text{given } x_0.
\end{aligned}
\end{equation}
Like Problem~\eqref{eq:discrete-problem},~\eqref{eq:path-space-problem} remains finite-dimensional, but it does not include constraints from the flow dynamics. Each stage cost of~\eqref{eq:path-space-problem} measures the discrepancy between the optimized displacement
$x_{k+1}-x_k$ and the pretrained flow drift
$\Delta t\,b_{t_k}(x_k)$.
Thus, path-space optimization can be interpreted as deforming the nominal
generative trajectory while penalizing the magnitude of that
deformation. However, solving~\eqref{eq:path-space-problem} still requires optimizing the entire $N$-step generation
trajectory. For flow matching models where the vector field is typically given by a deep neural network, repeatedly optimizing such a full path requires
a large number of evaluations of the high-dimensional vector field. This becomes more
costly as the number of sampling steps grows. To address this challenge, we will establish a receding-horizon primal-dual approach that allows us to retain the reward and constraint structure of the full inference
problem~\eqref{eq:continuous-problem} while optimizing a short portion of the future trajectory at each
time step during inference.

\section{Receding-horizon primal-dual}
\label{sec:methodology}

To make~\eqref{eq:path-space-problem} more tractable, we first propose to replace the whole
trajectory with a local subset of it.
At generation step $j$, starting from current state $x_j=\bar{x}_j$, we
select next $H_j$ states and optimize over them: $(x_{j+1}, \ldots, x_{j + H_j})$, where $H_j \ll N$. We apply the first induced control, and then repeat the optimization
from the resulting state.
This produces a receding-horizon procedure analogous to model predictive
control~\citep{garcia1989model,mayne2000constrained,alonso2022distributed}, but expressed directly in the path space of the pretrained flow. 

\paragraph{Where should the terminal objective and constraint be evaluated?} Denote a local path $ x_{j+1:j+H_j} = (x_{j+1},\ldots,x_{j+H_j})$ and abbreviate it as $\mathbf{x}_j$. Restricting the optimization horizon introduces an important design choice.
A direct truncation evaluates the reward and constraint at the end of the
local optimization window:
\begin{equation}
\tag{RH-H}
\label{eq:local-terminal-rh}
\begin{aligned}
\min_{\mathbf{x}_j}
\quad &
\sum_{\ell \,=\, j}^{j+H_j - 1}
\frac{1}{2\lambda\Delta t}
\left\|
x_{\ell+1}
-
x_\ell
-
\Delta t\,
b_{t_{\color{pennred}\ell}}(x_\ell)
\right\|^2
-
r(x_{j + H_j})
\\
\mathrm{s.t.}
\quad &
h(x_{j + H_j})\leq 0 .
\end{aligned}
\end{equation}
Formulation~\eqref{eq:local-terminal-rh} is computationally attractive, since it involves shorter paths to optimize. Meanwhile,~\eqref{eq:local-terminal-rh} relaxes the
original inference problem~\eqref{eq:continuous-problem}: the terminal reward and feasibility requirement
are now imposed at time $t_{j+H_j}$ rather than at the final generation time. Once the local cost has been optimized at $\mathbf{x}_j^\star$, the first residual intervention is
recovered as $u_j^\star
=
(x_{j + 1}^\star-\bar{x}_j)/\Delta t
-
b_{t_j}(\bar{x}_j)$. Only this first intervention is executed before shifting the window forward by one step and replanning~\citep{mayne2000constrained,alonso2022distributed}.

A distinctive feature of our setting is that the pretrained flow itself provides a natural prediction of how a locally optimized state would evolve up to the true terminal time. Rather than evaluating reward and feasibility prematurely at $x_{j+H_j}$, we can propagate this state under the uncontrolled flow and evaluate the original terminal objective at the resulting prediction. This brings each truncated subproblem closer to the full-horizon problem~\eqref{eq:path-space-problem} without extending the control horizon.

To make this precise, we introduce the discrete flow map induced by the pretrained flow, $X_{i,j}:\mathbb{R}^d\to\mathbb{R}^d$, defined recursively by $X_{i,i}(x)=x$ and $X_{i,k+1}(x) = X_{i,k}(x) + \Delta t\, b_{t_k}\bigl(X_{i,k}(x)\bigr)$ for $k \ge i$, so that $X_{i,j}(x)$ is the state obtained by propagating $x$ from $t_i$ to $t_j$ under the pretrained flow. Given the endpoint $x_{j+H_j}$ of the local window, the terminal prediction is then $\widehat{x}_N = X_{j+H_j,N}(x_{j+H_j})$. Evaluating the terminal objective at $\widehat{x}_N$ in place of $x_{j+H_j}$ yields our terminal-rollout formulation, RH-T, described below. This gives the terminal-rollout formulation
\begin{equation}
\tag{RH-T}
\label{eq:terminal-rollout-rh}
\begin{aligned}
\min_{\mathbf{x}_j}
\quad &
\sum_{\ell \,=\, j}^{j+{H_j} - 1}
\frac{1}{2\lambda\Delta t}
\left\|
x_{\ell+1}
-
x_\ell
-
\Delta t\,
b_{t_{\ell}}(x_\ell)
\right\|^2
-
r\bigl(\widehat{x}_N\bigr)
\\
\mathrm{s.t.}
\quad &
h\bigl(\widehat{x}_N\bigr)\leq 0 .
\end{aligned}
\end{equation}
Compared with~\eqref{eq:local-terminal-rh},~\eqref{eq:terminal-rollout-rh}
truncates only the control horizon while retaining the terminal objective and
constraint. Our analysis in Section~\ref{subsec:theoretical-analysis_suboptimality} identifies when this yields a tighter guarantee; the
comparison is detailed in Remark~\ref{rem:rhh-rht-comparison}.
\paragraph{Receding-Horizon Primal-Dual Approach.}
Directly solving a new hard-constrained nonlinear program at every generation
step itself is expensive.
We therefore introduce a primal--dual approach in which terminal constraints
are represented by dual variables
$\mu\in\mathbb{R}_+^m$. To unify our analysis, we represent an approximate terminal state by $\Phi_j(x)$, which is $x$ or $X_{j+{H_j},N}(x)$ for~\eqref{eq:local-terminal-rh} and~\eqref{eq:terminal-rollout-rh}, respectively. 
{Denote $\mathcal{E}_j(\mathbf{x}_j) \DefinedAs \sum_{\ell \,=\, j}^{j+{H_j} - 1}
\frac{1}{2\lambda\Delta t}
\left\|
x_{\ell+1}
-
x_\ell
-
\Delta t\,
b_{t_{\ell}}(x_\ell)
\right\|^2$ and $g_j(\mathbf{x}_j) \DefinedAs h\bigl(\Phi_j(x_{j + H_j})\bigr)$.} {For each $j$, we have a frozen local problem,
\begin{equation}\tag{$\text{RH}_j$}
    p_j^\star(\bar x_j) 
    \;\, \DefinedAs \;\, 
    \min_{\mathbf{x}_j} \;\; 
    \big\{
    \varphi_j(\mathbf{x}_j; \bar x_j)
    \DefinedAs \mathcal{E}_j(x_{j:j +H_j}) - r\bigl(\Phi_j(x_{j + H_j})\bigr)
    \big\}
    \quad
    \mathrm{s.t.}
    \quad
    g_j(\mathbf{x}_j) \leq 0.
 \label{eq_RH_j}
\end{equation}
The Lagrangian for~\eqref{eq_RH_j} is given by $\mathcal{L}_j(\mathbf{x}_j,\mu; \bar x_j)
=
\varphi_j(\mathbf{x}_j;\bar x_j)+ \mu^\top g_j(\mathbf{x}_j)$. Define the dual function by $d_j(\mu;\bar x_j) 
    \DefinedAs 
    \min_{\mathbf{x}_j} \mathcal{L}_j(\mathbf{x}_j,\mu;\bar x_j)$. The associated dual problem is
$\max_{\mu\,\geq\,0} d_j(\mu;\bar x_j)$,
whose solution set is denoted by
$\mathcal{M}_j^\star(\bar x_j)$. At each receding-horizon window, we hold the current state
$\bar x_j$ fixed and alternate primal minimization with projected
dual ascent: $\mathbf{x}_j^{(s)}
\in \arg\min_{\mathbf{x}_j}
\mathcal{L}_j(\mathbf{x}_j,\mu_j^{(s)};\bar x_j)$ and $\mu_j^{(s+1)}
=
\bigl[\mu_j^{(s)}
+\eta_\mu g_j(\mathbf{x}_j^{(s)})\bigr]_+$, where $\eta_\mu>0$ is the dual step size. The primal step optimizes the planned path for the current
multipliers, while the dual step increases the weight of violated
constraints and decreases that of strictly satisfied constraints. After $K$ dual updates, we recompute the primal response at
$\mu_j^{(K)}$ and execute only its first transition.
We then advance the window and warm-start its multiplier using
$\mu_{j+1}^{(0)}=\mu_j^{(K)}$.
This transfers the learned constraint weights between successive
windows without assuming that their optimal multipliers coincide.
We summarize the receding-horizon inference procedure in
Algorithm~\ref{alg:rh-path-space} and establish its theoretical guarantees for solving the constrained inference problem~\eqref{eq:continuous-problem} in Section~\ref{sec:theoretical-analysis}.

\begin{algorithm}[t]
\caption{Receding-Horizon Primal-Dual Algorithm}
\label{alg:rh-path-space}
\begin{algorithmic}[1]
\Require Pretrained flow $b_t$, reward $r$, constraints $h$, 
horizon $H$, local optimization steps $K\geq 1$, dual step size $\eta_\mu$, a choice of~\eqref{eq:local-terminal-rh} and~\eqref{eq:terminal-rollout-rh} through the terminal map $\Phi_j$.

\State Initialize $x_0$ and $\mu_0^{(0)} \geq 0$

\For{$j=0,\ldots,N-1$}

    \State Set $H_j \gets \min\{H,N-j\}$, $\bar{x}_j \gets x_j$ and initial $\mu_j^{(0)}\geq 0$.

    \For{$s=0,1,\ldots,K-1$}

        \State Compute the primal response: $\mathbf{x}_j^{(s)} 
        \;\gets\; 
        \argmin_{\mathbf{x}_j} \mathcal{L}_j (\mathbf{x}_j, \mu_j^{(s)}; \bar x_j) $.

        \State Perform the projected dual update: $\mu_j^{(s+1)} 
        \;\gets\;
        \Big[
        \mu_j^{(s)}+\eta_\mu
        g_j(\mathbf{x}_j^{(s)})
        \Big]_+$.
    \EndFor
    \State Compute the primal response at $\mu_j^{(K)}$: $\mathbf{x}_j^{(K)} 
    \;\gets\; 
    \argmin_{\mathbf{x}_j} \mathcal{L}_j (\mathbf{x}_j, \mu_j^{(K)}; \bar x_j)
    $.
    
    \State Execute the first state of $\mathbf{x}_j^{(K)}$: $x_{j+1} 
    \gets
    x_{j+1}^{(K)}$, and recover the residual control: $u_j 
    \gets
    ({x_{j+1}^{(K)}-\bar{x}_j})/{\Delta t}
    -
    b_{t_j}(\bar{x}_j)$.
   
    \State If $j<N-1$, warm start the next physical window with $\mu_{j+1}^{(0)} = \mu_j^{(K)}$. 
\EndFor

\noindent\Return The executed path $(x_0, \ldots, x_N)$ and residual controls $(u_0,\ldots,u_{N-1})$.
\end{algorithmic}
\end{algorithm}

\section{Theoretical Analysis}
\label{sec:theoretical-analysis}
To bridge the original continuous-time problem and the implementable
Algorithm~\ref{alg:rh-path-space}, we proceed in two steps.
First, we bound the approximation error of~\eqref{eq:local-terminal-rh} and~\eqref{eq:terminal-rollout-rh}
relative to the full discretized problem~\eqref{eq:discrete-problem}.
We then bound the error introduced by finitely solving the local
primal--dual problem~\eqref{eq_RH_j}. Combined with
Lemma~\ref{lem:discretization-error-main}, these results connect
Algorithm~\ref{alg:rh-path-space} to the original continuous formulation
~\eqref{eq:continuous-problem}.

\subsection{Bounding the Suboptimality}
\label{subsec:theoretical-analysis_suboptimality}

We first quantify the error introduced by restricting the control horizon.
Specifically, we compare the idealized constrained trajectories of~\eqref{eq:local-terminal-rh} and ~\eqref{eq:terminal-rollout-rh}
with the optimal solution $\bx^\star$ of the full discretized
problem~\eqref{eq:discrete-problem}. Both approximations optimize only $H_j$
future steps, but differ in where the terminal reward and constraint are
evaluated.

Throughout this section, Assumption~\ref{ass:main-discretization} remains in
force. For the receding-horizon comparison, we additionally impose the
following structure. Unlike the continuous-to-discrete analysis, terminal
feasibility here is handled through a uniform Slater condition rather than
local repairability.

\begin{assumption}[Affine flow and bounded amplification]
\label{ass:rh-affine}
The discretized flow is affine,
$x_{k+1}=x_k+\Delta t(A_kx_k+c_k+u_k)$.
On the relevant reachable set, $\|A_kx+c_k\|\leq M$, and further let 
$
\kappa_A
\; \DefinedAs \;
\sup_{0 \, \leq \, k \, \leq \, \ell \, \leq \, N}\;
\left\|
\prod_{s=k}^{\ell-1}(I+\Delta t A_s)
\right\|
\; < \;
\infty,
$
with $M$ and $\kappa_A$ independent of the horizon.
\end{assumption}

\begin{assumption}[Convexity and uniform feasibility]
\label{ass:rh-constrained}
The reward $r$ is concave and $L_r$-Lipschitz, and $h$ is componentwise
convex and $L_h$-Lipschitz. Either subproblem~\eqref{eq:local-terminal-rh} or~\eqref{eq:terminal-rollout-rh} admits a strictly
feasible trajectory with uniform margin $\gamma>0$ and objective at most
$\overline J_D$.
\end{assumption}

Let $\underline J_D$ be a uniform lower bound on the objectives of \eqref{eq:local-terminal-rh} or~\eqref{eq:terminal-rollout-rh} and define
$
\chi_h
\DefinedAs
{L_h(\overline {J_D}-\underline {J_D})}/{\gamma}
$ 
and 
$
G_h
\DefinedAs
L_r+\chi_h.
$
Here, $G_h$ represents the effective terminal sensitivity, with $L_r$
capturing reward variation and $\chi_h$ the additional sensitivity required
to maintain feasibility.

We first consider ~\eqref{eq:local-terminal-rh}, which evaluates reward and feasibility at the end of
the truncated horizon. Consequently, its approximation error depends on how
far the nominal flow can move over the omitted portion of the trajectory,
captured by the drift bound $M$.

\begin{lemma}[Constrained RH-H bound]
\label{lem:c-rhh-affine}
Let Assumptions~\ref{ass:main-discretization}--\ref{ass:rh-constrained} hold, and 
$M<\infty$, $G_h<\infty$, and $\gamma>0$. Then, the optimal trajectory $\bx^{H,\star}$ of~\eqref{eq:local-terminal-rh} satisfies
$$
0
\; \leq \;
J(\bx^{H,\star})-J(\bx^\star)
\; \leq \;
C_H^{\rm C}\frac{N-H}{N}, 
\qquad
\text{ where } \; 
C_H^{\rm C}
\DefinedAs
2\max\left\{
G_hM+ \frac{1}{2}{\lambda G_h^2},
\frac{M^2}{2\lambda}
\right\}.
$$
\end{lemma}

Thus,~\eqref{eq:local-terminal-rh} incurs a horizon error linear in the omitted fraction
$(N-H)/N$. Its constant depends on the absolute motion of the pretrained flow
through $M$, reflecting that reward and feasibility are evaluated before the
true terminal time.

\eqref{eq:terminal-rollout-rh} instead propagates the local endpoint to the original terminal time $N$ using
the pretrained flow map. This removes dependence on the absolute nominal
motion after the control horizon; the relevant quantity is instead how strongly
the remaining flow amplifies perturbations, measured by $\kappa_A$.

\begin{lemma}[Constrained RH-T bound]
\label{lem:c-rht-affine}
Let Assumptions~\ref{ass:main-discretization}--\ref{ass:rh-constrained} hold, and assume
$\kappa_A<\infty$, $G_h<\infty$, $\gamma>0$. Then, the optimal trajectory $\hat\bx^{H,\star}$ of~\eqref{eq:terminal-rollout-rh} satisfies
$$
0
\; \leq \;
J(\widehat{\bx}^{H,\star})-J(\bx^\star)
\; \leq \;
C_T^{\rm C}\frac{N-H}{N}, 
\qquad
\text{ where } \; 
C_T^{\rm C}
\DefinedAs
\frac{1}{2}
\kappa_A^2G_h^2 {\lambda}.
\vspace{-0.2cm}
$$
\end{lemma}


Both formulations thus achieve an $O((N-H)/N)$ approximation to the
full discretized problem, but their constants capture different properties of
the pretrained dynamics: \eqref{eq:local-terminal-rh} depends on nominal flow motion, whereas \eqref{eq:terminal-rollout-rh}
depends on perturbation amplification. The resulting comparison, including
conditions under which \eqref{eq:terminal-rollout-rh} yields the tighter guarantee, is detailed in
Remark~\ref{rem:rhh-rht-comparison}.

\begin{remark}[When is terminal rollout beneficial?]
\label{rem:rhh-rht-comparison}
For our constrained problem,
$\kappa_A\leq\sqrt{2}$ implies
$C_T^{\rm C}\leq C_H^{\rm C}$.
Thus, when the remaining pretrained flow does not strongly amplify
intermediate-state perturbations, terminal rollout gives the tighter guarantee.
The analogous unconstrained comparison is given in
Appendix~\ref{app:unconstrained-rh-bounds}.
Intuitively,~\eqref{eq:terminal-rollout-rh} is most useful when the state continues to evolve after
the control horizon, but this evolution remains predictable.
\end{remark}

\subsection{Convergence of Dynamic Primal--Dual}
\label{subsec:theoretical-analysis_convergence}

In Lemmas~\ref{lem:c-rhh-affine}--\ref{lem:c-rht-affine}, we assume that each
subproblem of~\eqref{eq_RH_j} is solved, exactly. We next quantify the error introduced by a
finite number of primal--dual updates in Algorithm~\ref{alg:rh-path-space}.  For the $j$th subproblem~\eqref{eq_RH_j}, we denote by $B_j$ a lower block diagonal matrix with off-diagonals $-(I_d+\Delta t A_{j+1})$, $S_j$ a row vector with first block $I_d+\Delta t A_j$, and $Q_j$ the linear part of the terminal map $\Phi_j$ ($Q_j=I$ for~\eqref{eq:local-terminal-rh}). 
Using the strong convexity and dual smoothness, we first show that the inner-loop primal-dual iterates (Lines~4--7 of~Algorithm~\ref{alg:rh-path-space}) converge to the optimal, sublinearly.

\begin{lemma}[Local primal--dual convergence]
\label{thm: subproblem convergence}
    Let Assumptions~\ref{ass:rh-affine}--\ref{ass:rh-constrained} hold for the $j$th frozen subproblem problem~\eqref{eq_RH_j}. 
    Then, if $0<\eta_\mu<2/L_{d,j}$, then the inner-loop primal-dual iterates (Lines~4--7 of~Algorithm~\ref{alg:rh-path-space}) converge their optima, i.e., $\mu_j^{(s)}\to\bar\mu_j\in\mathcal M_j^\star$ and $\mathbf{x}_j^{(s)}\to\mathbf{x}_j^\star$. Moreover, if $0<\eta_\mu\leq1/L_{d,j}$, then the gaps between the iterates $(\mathbf{x}_j^{(s)}, \mu_j^{(s)})$ and the optima are bounded by
    $$
    p_j^\star-d_j(\mu_j^{(s)})
    \;\leq\;
    \frac{R_j^2}{2\eta_\mu s},
    \;\qquad\;
    \|\mathbf{x}_j^{(s)}-\mathbf{x}_j^\star\|
    \;\leq\;
    \frac{R_j}{\sqrt{\sigma_j\eta_\mu s}},
    \;\qquad\;
    \|[g_j(\mathbf{x}_j^{(s)})]_+\|
    \;\leq\;
    \frac{L_{g,j}R_j}{\sqrt{\sigma_j\eta_\mu s}}
    $$
    where $R_j\DefinedAs\operatorname{dist}(\mu_j^{(0)},\mathcal M_j^\star)$, $\sigma_j \DefinedAs {\sigma_{\min}^2(B_j)}/(\lambda\Delta t)$ is the strong convexity modulus, $L_{g,j} \DefinedAs L_h\|Q_j\|$ and $L_{d,j} \DefinedAs {L_{g,j}^2}/{\sigma_j}$ are the Lipschitz constants of the constraint and dual functions. 
\end{lemma}

We defer the proof of Lemma~\ref{thm: subproblem convergence} to Appendix~\ref{app:proof-subproblem-convergence}. 
In Lemma~\ref{thm: subproblem convergence}, we show that the primal-dual gap $p_j^\star-d_j(\mu_j^{(s)})$, and the primal gap $\|\mathbf{x}_j^{(s)}-\mathbf{x}_j^\star\|$ and constraint violation $\|[g_j(\mathbf{x}_j^{(s)})]_+\|$ decay to zero, sublinearly. This result shows that the primal-dual update (Lines~4--7 of~Algorithm~\ref{alg:rh-path-space}) solves the
subproblem~\eqref{eq_RH_j} up to an error $\mathcal{O}(1/\sqrt{s})$. We next quantify how these errors propagate through the receding-horizon trajectory, and establish the global convergence of~Algorithm~\ref{alg:rh-path-space} in Theorem~\ref{thm:finite-rh-error}; see Appendix~\ref{app:proof-finite-rh-error} for proofs. Denote the first state of the solution $\mathbf{x}_j^\star(\xi)$ to~\eqref{eq_RH_j}
initialized at $\xi$ as $\pi_j(\xi)$. We let the exact trajectory generated by $\pi_j$ be $x_{j+1}^{\rm RH} = \pi_j(x_j^{\rm RH})$ with $x_0^{\rm RH} = x_0$.

\begin{theorem}[Global convergence]
\label{thm:finite-rh-error}
    Let Assumptions~\ref{ass:rh-affine}--\ref{ass:rh-constrained} hold for the $j$th frozen subproblem problem~\eqref{eq_RH_j}. Then, if $0<\eta_\mu\leq1/\bar L_d$ and each inner-loop primal-dual update of Algorithm~\ref{alg:rh-path-space} take $K$ iterations, then the generated trajectory $x_j$ for $j=1,\ldots,N$ satisfies
$$
\|x_j-x_j^{\rm RH}\|
\leq
\frac{\bar R\,\Gamma_j(\bar\kappa)}
{\sqrt{\underline\sigma\,\eta_\mu K}},
\;\qquad\;
\|[h(x_N)]_+\|
\leq
\frac{L_h\bar R}
{\sqrt{\underline\sigma\,\eta_\mu K}}
$$
where
$\bar L_d \DefinedAs \max_jL_{d,j}$, 
$\bar R \DefinedAs \max_j\operatorname{dist}
(\mu_j^{(0)},\mathcal M_j^\star(x_j))$,
$\underline\sigma \DefinedAs \min_j\sigma_j$,
$
\kappa_j \DefinedAs
\frac{\|B_j^\top S_j\|}
{\lambda\Delta t\,\sigma_j}
$ 
$
\bar\kappa \DefinedAs \max_j\kappa_j
$, and
$
\Gamma_j(\bar\kappa) \DefinedAs \sum_{q\,=\,0}^{j-1}\bar\kappa^q.
$
\end{theorem}

In Theorem~\ref{thm:finite-rh-error}, we prove that the trajectories generated by Algorithm~\ref{alg:rh-path-space} converge to the exact trajectory induced by the policy $\pi_j$, in terms of the primal gap $\|x_j-x_j^{\rm RH}\|$ and the constraint violation $\|[h(x_N)]_+\|$. This implies that a finite iterations of primal-dual optimization contribute
an $O(1/\sqrt{K})$ error, modulated by the sensitivity
$\Gamma_j(\bar\kappa)$ of the exact policy.

\begin{corollary}[Error decomposition]
\label{thm:error-decomposition}
Let the assumptions
of Theorem~\ref{thm:finite-rh-error} hold. Suppose $J$ is $L_J$-Lipschitz on the relevant set. Then, the generated trajectory $\textbf{x}^{(K)}$ by Algorithm~\ref{alg:rh-path-space} satisfies 
$$
J(\bx^{(K)})- J_C(u_C^\star)
\; \leq \;
\frac{C_{\mathrm{disc}}}{N}
+
C_{\rm hor}\frac{N-H}{N}
+
\frac{L_J\sqrt{N}\,\bar R\,\Gamma_N(\bar\kappa)}
{\sqrt{\underline\sigma\,\eta_\mu K}},
$$
where $C_{\rm hor} \DefinedAs C_H^{\rm C}$ for \eqref{eq:local-terminal-rh} and
$C_{\rm hor} \DefinedAs C_T^{\rm C}$ for \eqref{eq:terminal-rollout-rh}.
\end{corollary}

The proof of Corollary~\ref{thm:error-decomposition} is straightforward by adding and subtracting $J_D\bigl(u_D^\star(N)\bigr)$ and $J(\bx^{\rm RH})$, and applying Lemmas~\ref{lem:discretization-error-main}--\ref{lem:c-rht-affine} and
Theorem~\ref{thm:finite-rh-error}. The bound separates the discretization error $O(1/N)$ and horizon bias $O((N-H)/N)$ from the
$O(1/\sqrt{K})$ finite-solver error. 

Taking $H=N$ eliminates the horizon bias $O((N-H)/N)$, and letting $K\to\infty$ for fixed $N$ eliminates the finite-solver error, leaving only the discretization error bounded by $O(1/N)$, which vanishes as $N\to\infty$, yielding global convergence to the continuous-time optimum.

\section{Experiments}\label{sec:experiments}
We evaluate ConicFlow in two regimes: \emph{unconstrained} reward-guided generation, where the objective trades terminal reward against control energy, and \emph{constrained} generation, where the sample must additionally satisfy hard terminal constraints. We compare against the state-of-the-art approach in each regime: (i) FMRG~\citep{huang2026guide} and (ii) HardFlow~\citep{li2026hardflow}. FMRG is a training-free, single-trajectory method that uses soft reward guidance, and it is theoretically near-optimal only when small deviations from the pretrained flow suffice. HardFlow enforces constraints at terminal time, but with a different methodology compared to our approach, and it has been shown to outperform projection-based alternatives. See Appendix~\ref{app: related work} for more discussion on related approaches.

\textbf{Unconstrained low-dimensional setting.} We first use a flow map model trained on 2D samples from the synthetic two moons dataset and consider two rewards, both taking values in $[0,1]$. The first reward increases closer to the upper moon. As shown in Figure~\ref{fig: pareto-lowdim} (left), all methods saturate the reward at small $\lambda$ (i.e., large control-energy penalties). On average, small deviations from the pretrained flow suffice to land trajectories on the upper moon. Both ConicFlow variants: \eqref{eq:local-terminal-rh} and~\eqref{eq:terminal-rollout-rh} achieve the best trade-off between terminal reward and control energy, closely matched by FMRG, which is theoretically near-optimal in the small-$\lambda$ regime~\citep{huang2026guide}.

\begin{figure}[H]
    \centering
    \includegraphics[width=0.37\textwidth]{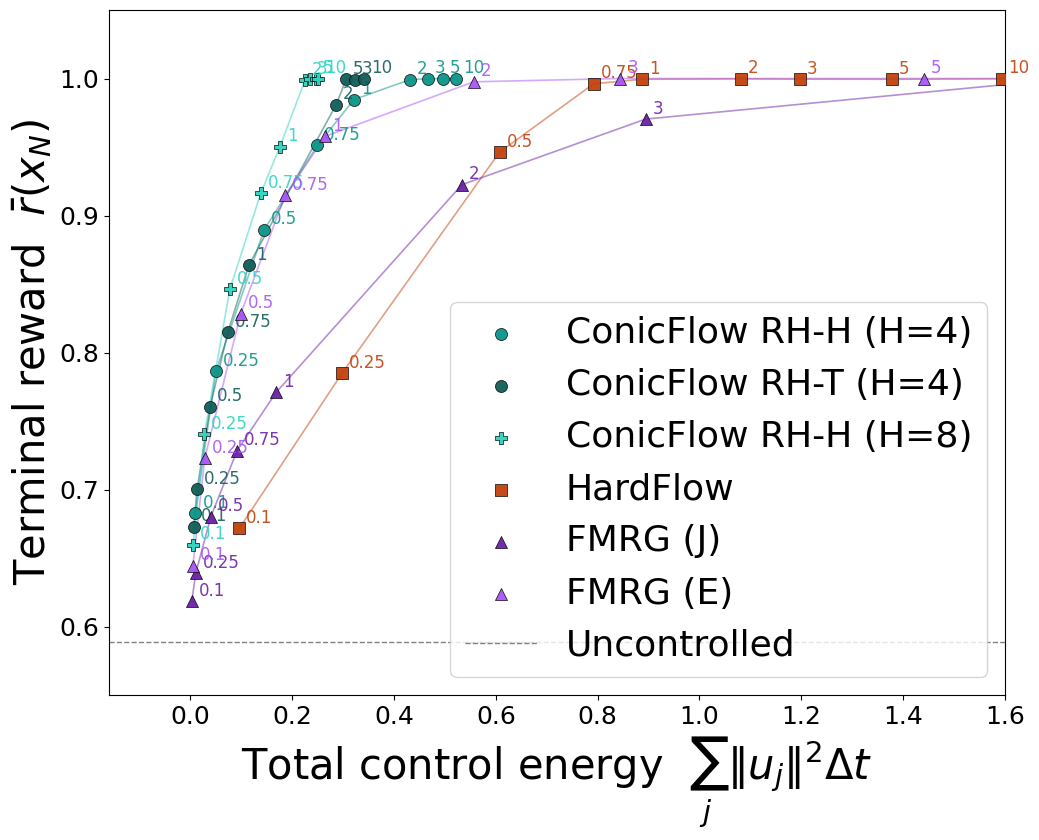}
    \includegraphics[width=0.37\textwidth]{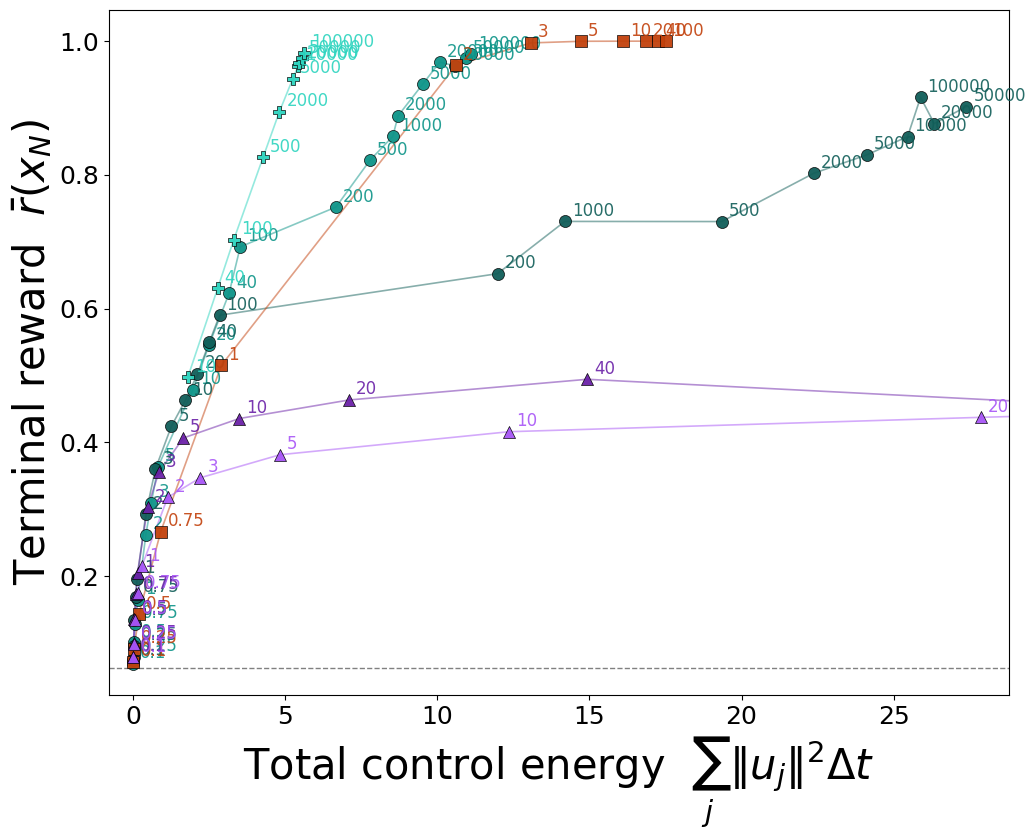}
    \raisebox{3.35mm}{\includegraphics[width=0.23\textwidth]{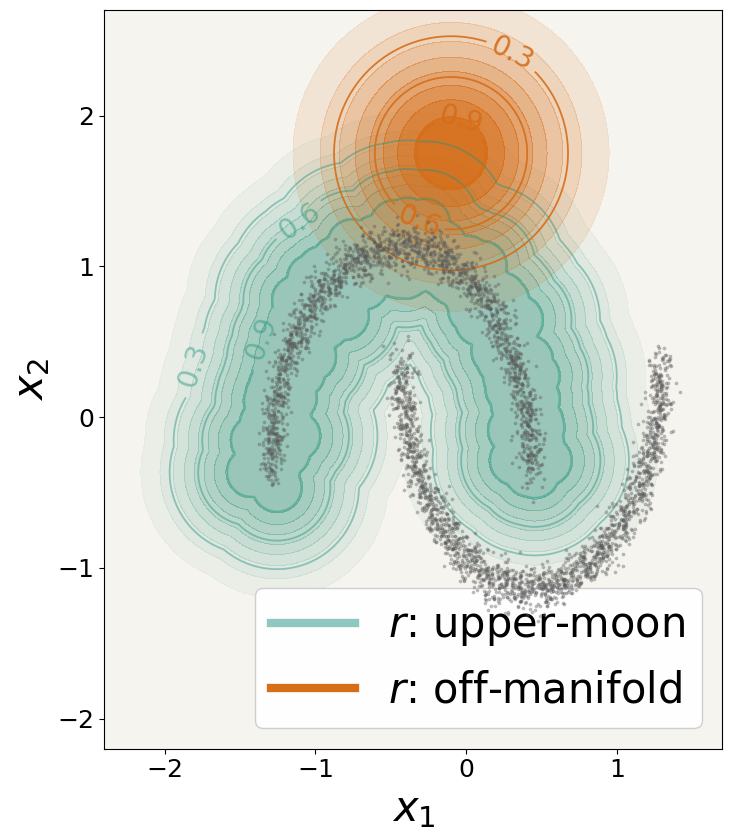}}
    \vspace{-0.5cm}
    \caption{Trade-off between terminal reward and total control cost, traced by varying the control-cost weight $\lambda$ (upper left is better). $\lambda$ value displayed on each point. \textbf{Left:} easy reward (upper moon). \textbf{Middle:} difficult reward (off-manifold). \textbf{Right:} uncontrolled sample distribution and both rewards.}
    \label{fig: pareto-lowdim}
\end{figure}

To study settings that need large deviations (thus large $\lambda$), the second reward peaks off the pretrained data manifold (middle in Figure~\ref{fig: pareto-lowdim}). FMRG fails outside the small-$\lambda$ regime as its control energy increases without noticeable reward gains. ConicFlow \eqref{eq:local-terminal-rh} with $H=4$ matches HardFlow (and improves on it for small deviations), while increasing the horizon to $H=8$ improves the trade-off across the board, outperforming all baselines and saturating the reward with the least control energy. Neither \eqref{eq:local-terminal-rh} nor \eqref{eq:terminal-rollout-rh} consistently outperforms the other.

In short, FMRG does well for small $\lambda$ and HardFlow for larger $\lambda$. \emph{ConicFlow with a short horizon combines the strengths of FMRG and HardFlow across all $\lambda$, and improves further as $H$ increases.}

\textbf{Unconstrained image setting.} In high-dimensional, prompt-conditioned image generation, we use FLUX.1-dev with the flow-map LoRA released by~\citep{huang2026guide}, distilled from the base model as our flow matching model~\citep{boffi2025consistency}. We consider a well-established task: reward-guided image editing~\citep{li2026hardflow,kim2022diffusionclip,patashnik2021styleclip}. A base prompt $c$ generates the sample, while the reward $r(x;c')$ is evaluated against a \emph{different} edit prompt $c'$. The control must introduce the attribute named by $c'$ while deviating as little as possible from the original image. We use three commonly used rewards CLIP~\citep{radford2021learning}, ImageReward~\citep{xu2023imagereward}, and PickScore~\citep{kirstain2023pickapic}. The RH-T variant of ConicFlow was used across all image experiments as it led to better performance.

\begin{figure}[h]
  \centering
    \includegraphics[width=0.45\linewidth]{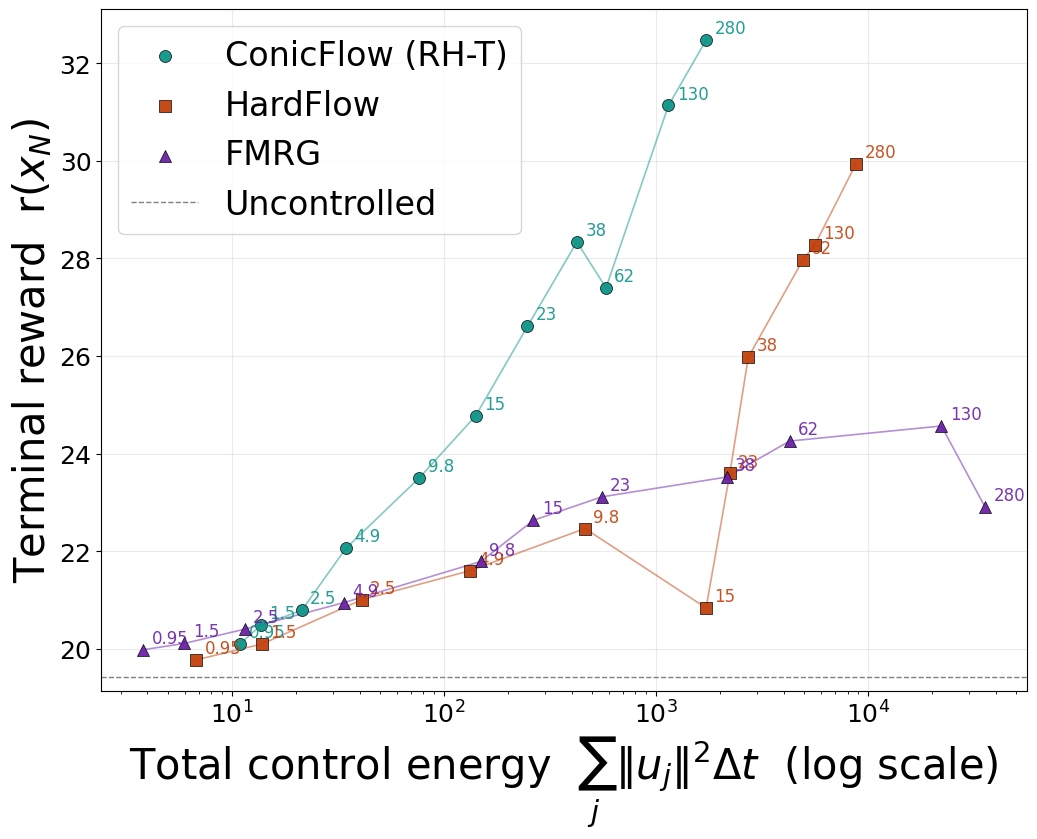}
    \qquad
    \raisebox{4.25mm}{\includegraphics[width=0.45\linewidth]{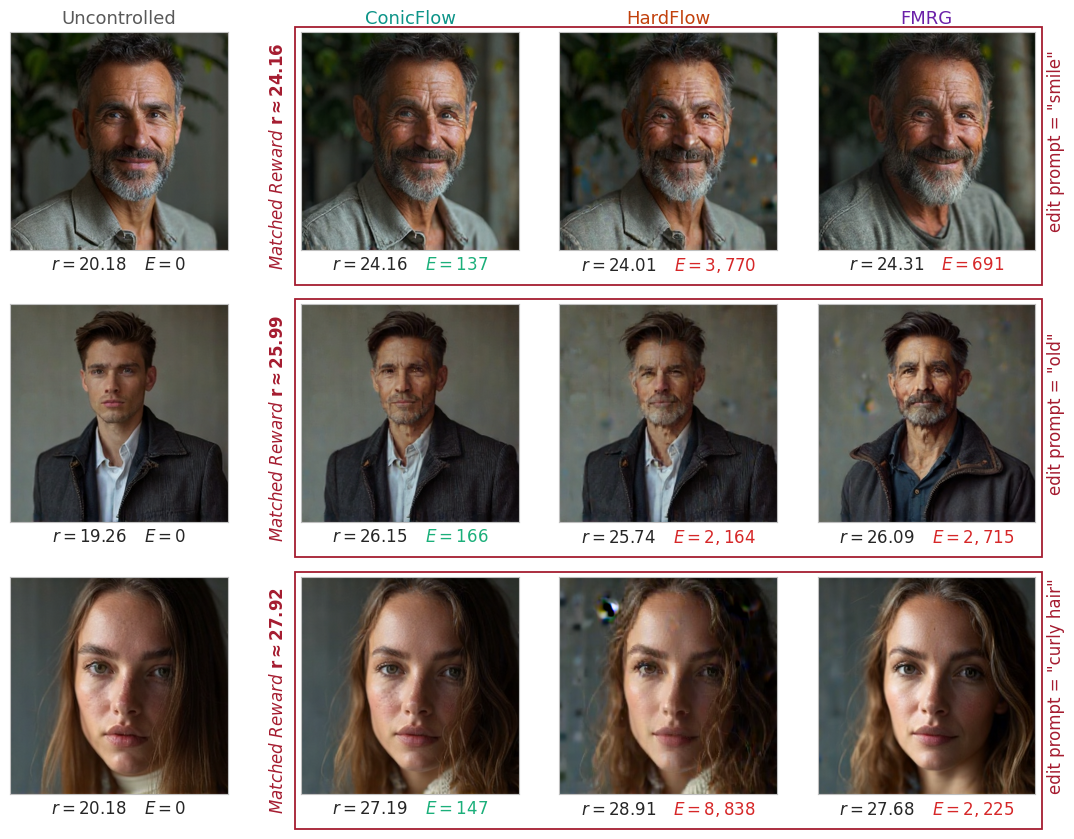}
    \label{fig:edit-rows-editing}}
    \vspace{-0.5cm}
  \caption{Trade-off between terminal reward and total control cost. \textbf{Left:} Reward vs.\ control energy on the editing task, traced by sweeping $\lambda$ (upper left is better). \textbf{Right:} Edits from each method at approximately matched reward; ConicFlow reaches the same reward with much lower control energy.
  }
  \label{fig:pareto-edit-pickscore}
\end{figure}

ConicFlow achieves the best trade-off between reward and control energy (Figure~\ref{fig:pareto-edit-pickscore}, left). At matched reward, its edits stay closer to the original image (Figure~\ref{fig:pareto-edit-pickscore}, right). Table~\ref{tab:objective-improvement} quantifies this across multiple rewards, seeds and prompts, each describing a new optimization problem.

\begin{table}[h]
\caption{{Mean reward gain over a shared energy budget.} Each method's
    reward--energy frontier, averaged over the control-energy range attained by
    all three methods and measured relative to the uncontrolled sample (units are those of rewards). Mean
    over $n{=}4$ (prompt, seed) pairs per reward.}
    \vspace{-0.2cm}
  \centering
  \small
  \begin{tabular}{lccc}
    \toprule
    Method & CLIP $(\uparrow)$ & ImageReward $(\uparrow)$ & PickScore $(\uparrow)$ \\
    \midrule
    ConicFlow & $\mathbf{0.130} \pm 0.018$ & $\mathbf{2.048} \pm 0.455$ & $\mathbf{6.158} \pm 0.728$ \\
    HardFlow  & $0.118 \pm 0.021$ & $1.279 \pm 0.374$ & $1.736 \pm 0.472$ \\
    FMRG      & $0.065 \pm 0.008$ & $1.278 \pm 0.504$ & $2.670 \pm 0.073$ \\
    \bottomrule
  \end{tabular}
  
  \label{tab:objective-improvement}
\end{table}

\paragraph{Constrained image setting.} We consider the same image editing task, but now enforce the desired attribute through a hard constraint $r(x;c') \geq \tau$, i.e., $h(x) = \tau - r(x;c') \leq 0$, using PickScore as $r$. In the multi-constraint setting, we impose one such constraint per edit prompt, $h_k(x) = \tau_k - r(x;c'_k) \leq 0$ for $k = 1,\dots,K$, requiring the sample to exhibit all $K$ attributes simultaneously. For each problem, the threshold is set to $\tau = m \cdot r(x^{\text{unc}};c')$, a multiple of the score achieved by the uncontrolled sample, with $m$ ranging from $1.1$ (easier) to $1.5$ (harder). We report the control energy and the feasibility rate, i.e., the fraction of (prompt, seed) problems whose final sample satisfies the constraint. Since FMRG has no mechanism for enforcing hard constraints, we do not include it in this comparison.

\begin{figure*}[h]
  \centering
  \begin{subfigure}[t]{0.3\textwidth}
    \vspace{0pt}
    \centering
    \includegraphics[width=\linewidth]{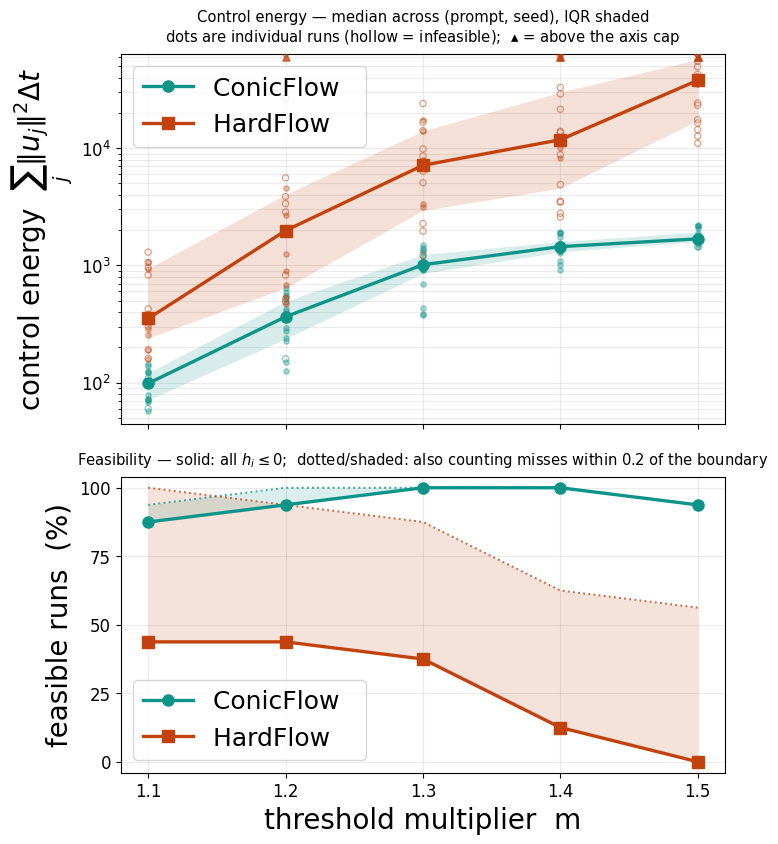}
    \phantomcaption\label{fig:edit-rows-left}
  \end{subfigure}\hfill
  \begin{subfigure}[t]{0.34\textwidth}
    \vspace{0pt}
    \centering
    \includegraphics[width=\linewidth]{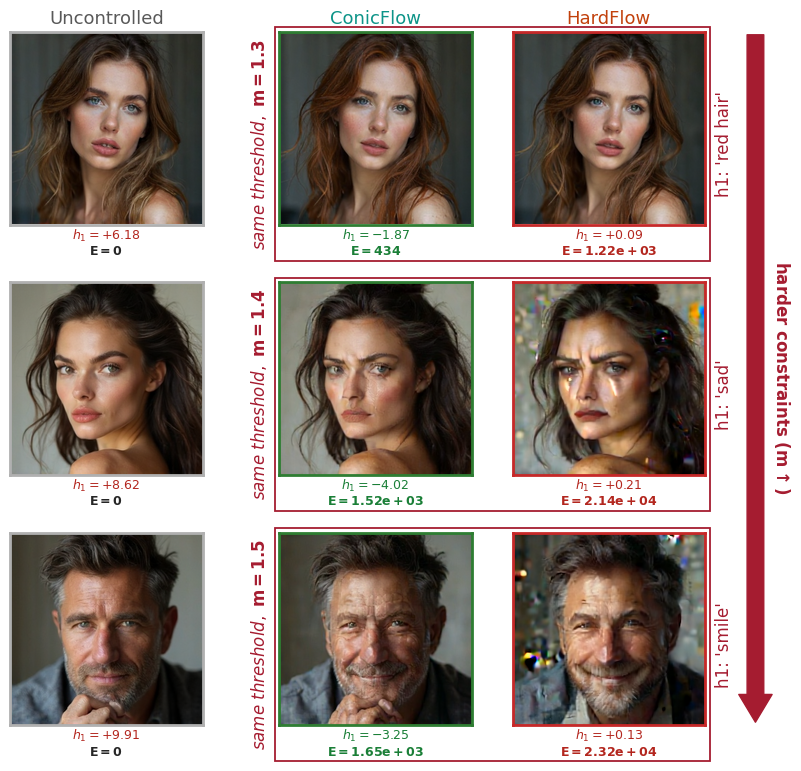}
    \phantomcaption\label{fig:edit-rows-middle}
  \end{subfigure}\hfill
  \begin{subfigure}[t]{0.34\textwidth}
    \vspace{0pt}
    \centering
    \includegraphics[width=\linewidth]{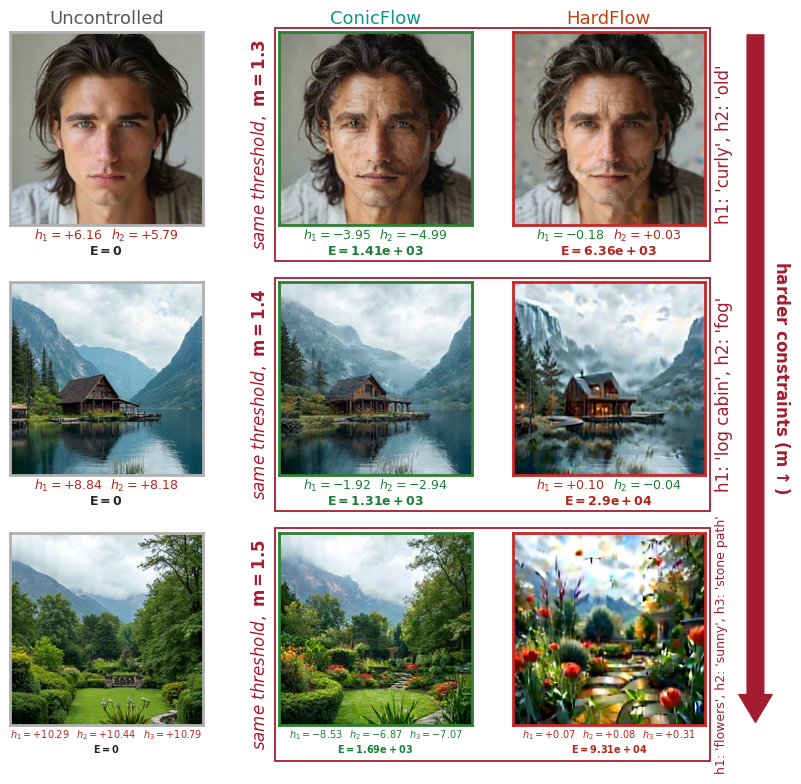}
    \phantomcaption\label{fig:edit-rows-right}
  \end{subfigure}
  \vspace{-0.7cm}
  \caption{ Constrained image editing comparison. \textbf{Top left:} control energy vs.\ threshold multiplier $m$, median over (prompt, seed) pairs.
  \textbf{Bottom left:} feasibility rate vs.\ $m$.
  \textbf{Middle:} single-constraint edits with increasing $m$.
  \textbf{Right:} multi-constraint edits with increasing $m$.}
  \label{fig:edit-rows}
\end{figure*}

As $m$ increases, both methods expend more control energy (Figure~\ref{fig:edit-rows}, top left), but ConicFlow satisfies the constraints with substantially less energy (up to $20\times$). ConicFlow also maintains a feasibility rate close to $100\%$, whereas HardFlow's drops as the threshold grows (bottom left). The examples demonstrate the same message: as $m$ increases, both methods express the edit concept more strongly, but HardFlow deviates increasingly from the uncontrolled image, eventually producing visible noise and artifacts. This is most pronounced in the multi-constraint setting (right). See Appendix~\ref{app: experiments} for additional implementation details and results.

\section{Conclusion}
We introduced ConicFlow, a receding-horizon primal-dual framework for constrained inference with pretraned flow matching models. This framework combines reward optimization, control-energy regularization, and explicit terminal constraints. We adapt receding-horizon control and propose a terminal-estimate variant that uses the pretrained flow map to evaluate the reward and constraints at the predicted terminal state rather than at the end of the truncated window. Our analysis quantifies discretization and horizon truncation  errors and establishes primal-dual convergence guarantees under convex-affine assumptions, including the propagation of local optimization errors across successive windows. Experiments on synthetic data and image editing demonstrate that this approach can improve the reward–control trade-off while maintaining high empirical feasibility.

\clearpage
\section*{Acknowledgment}
James Anderson is partially funded by NSF grants EECS 2144634 and CNS 2535097, and the Center of AI Technology (CAIT) in collaboration with Amazon.

\bibliography{dd-bib}
\bibliographystyle{iclr2027_conference}

\clearpage   
\appendix

\appendix

\startcontents[appendix]

\section*{Appendix Table of Contents}
\printcontents[appendix]{}{1}{
  \setcounter{tocdepth}{2}
}

\newpage
\section{Related Work}\label{app: related work}

Inference-time control of pretrained diffusion and flow matching models has developed along two largely separate lines: \emph{soft} guidance methods, which trade reward against fidelity to the pretrained dynamics but offer no feasibility guarantee, and \emph{hard}-constraint methods, which enforce feasibility, most often by projecting onto the constraint set during sampling. ConicFlow belongs to the second line. It is a hard-constrained sampling method whose terminal constraints are enforced through dual multipliers, and whose objective additionally accommodates a reward and a control-energy penalty so that feasible samples are ranked by preference and stay close to the pretrained dynamics. To evaluate it we compare against a recent state-of-the-art representative from each line: Flow Map Reward Guidance (FMRG) \citep{huang2026guide} for soft guidance, and HardFlow \citep{li2026hardflow} for hard-constrained sampling in the tradition of projection-based methods. We review both families below and then discuss the control-theoretic tools we build on.

\paragraph{Flow matching and flow maps.}
The pretrained models we steer are trained with flow matching \citep{lipman2022flow,liu2022flow} or, equivalently, stochastic interpolants \citep{albergo2025stochastic}, which learn a velocity field $b_t$ whose ODE transports a simple base distribution to the data distribution. Flow maps \citep{boffi2024flow,boffi2025consistency}, which directly predict the ODE solution between two arbitrary times, were introduced to accelerate sampling; we use a pretrained flow map both to integrate the dynamics and, in the RH-T variant, to predict the terminal sample from an intermediate state.

\paragraph{Soft guidance and reward-tilted sampling.}
Classifier guidance \citep{dhariwal2021diffusion} and classifier-free guidance \citep{ho2022classifier} steer diffusion models by adding a (learned or amortized) gradient of a conditional likelihood to the score. Training-free variants replace the classifier with the gradient of a reward evaluated at a one-step estimate of the clean sample, as in diffusion posterior sampling \citep{chung2023diffusion} and its flow-based counterparts. A second line differentiates through the sampling ODE itself: FlowGrad \citep{liu2023flowgrad} and D-Flow \citep{ben2024d} optimize the initial noise or the whole trajectory by backpropagating a terminal loss, and OC-Flow \citep{wang2025ocflow} recasts this as an optimal control problem solved via the Pontryagin maximum principle, with convergence guarantees in Euclidean space and on $\mathrm{SO}(3)$. OC-Flow is the closest precursor to our unconstrained formulation, and it shows that the KL divergence between the controlled and uncontrolled terminal distributions is bounded by the expected control cost, which justifies the quadratic energy term in (P-C); it does not, however, handle constraints, and solving the full-horizon adjoint problem is considerably more expensive than a receding-horizon scheme. Most recently, FMRG \citep{huang2026guide} reformulates guidance as \emph{deterministic} optimal control and shows that the flow map arises naturally in the optimal feedback, yielding a single-trajectory, training-free method that matches or surpasses multi-particle guidance schemes with as few as three function evaluations at text-to-image scale. FMRG is provably near-optimal in the small-$\lambda$ regime (small deviations from the pretrained flow), which is exactly the regime in which we observe it performing best (Section~\ref{sec:experiments}); it has no mechanism for enforcing hard terminal constraints, and its greedy, myopic control degrades when the reward lies far from the data manifold. We therefore use FMRG as the representative soft-guidance baseline, and, since it subsumes DPS-style methods as special cases, we do not compare against those separately.

\paragraph{Fine-tuning for reward alignment.}
An orthogonal approach adapts the model weights to a reward, via reinforcement learning \citep{black2024training,fan2023dpok} or stochastic optimal control \citep{domingo2025adjoint,uehara2024finetuning}. Fine-tuning amortizes the cost of steering over many samples but must be repeated for every new reward or constraint, can be prohibitively expensive for large models, and offers no per-sample feasibility guarantee. Our method is inference-time only and complementary to these approaches.

\paragraph{Hard-constrained sampling via projection.}
The dominant approach to hard constraints projects the iterate onto the feasible set at every (or every late) sampling step. Projected diffusion models \citep{christopher2024projected} interpret reverse diffusion as projected Langevin dynamics and require a projection operator for the constraint set at each step; \citet{zampini2025trainingfree} extend this to latent (Stable Diffusion) models and, notably, to copyright-style constraints on images, but still require a function measuring distance to the feasible set and a projection in latent space. Projected coupled diffusion \citep{luan2026pcd} applies the same per-step projection to the joint generation of correlated samples from several pretrained models. In robotics, diffusion predictive control with constraints \citep{romer2025dpcc} and DDAT \citep{bouvier2025ddat} project generated trajectories onto sets that are dynamically admissible for the robot, either at every step or only late in sampling; these methods are tailored to state--action constraints and system dynamics and are less directly comparable to the general terminal-constraint problem we consider. As several authors have noted \citep{li2026hardflow,kurtz2026ldf}, projecting the \emph{entire} sampling path is overly restrictive: intermediate iterates need not be feasible, exact projection onto a general nonlinear set is itself an optimization problem, and projecting inside an ODE produces a discontinuous, projected dynamical system that degrades sample quality.

\paragraph{Hard constraints via penalties and dual dynamics.}
Several recent works instead enforce terminal constraints through the sampling dynamics themselves. TOCFlow \citep{gao2026tocflow} formulates terminally constrained sampling as an optimal control problem, characterizes the value function through a Hamilton--Jacobi--Bellman equation, and shows that, as the control penalty vanishes, the terminal law converges to a generalized Wasserstein projection onto the constraint manifold $\{x:h(x)=0\}$. Algorithmically, however, constraints enter through a quadratic terminal penalty $\|h(x)\|^2$, so exact feasibility is recovered only in the vanishing-penalty limit; the theoretical analysis is formulated around a smooth zero-level-set constraint manifold. Lagrangian dual flows \citep{kurtz2026ldf} augment the generative ODE with a dual co-state (and a slack variable for inequalities) so that constraint violation is driven to zero along the flow without projections or inner optimization; the authors also show why quadratic penalties of the TOCFlow type are a poor substitute for genuine dual variables. This is conceptually close to our dynamic primal--dual scheme, but LDF flows the multiplier continuously alongside the sample with no notion of look-ahead, and has so far been demonstrated only on low-dimensional and MNIST-scale problems. Constrained diffusers \citep{zhang2025constrained} bring projected, primal--dual, and augmented-Lagrangian updates to reverse Langevin sampling for safe planning; they target a different, \emph{distribution-level} problem, $\max_{\lambda\ge 0}\min_q \mathrm{KL}(q\,\|\,p) + \lambda^\top \mathbb{E}_{x\sim q}[g(x)]$, whose dual update requires an expectation over samples at every step, and their setting is diffusion-based trajectory planning rather than deterministic flows.

\paragraph{Hard constraints via trajectory optimization.}
HardFlow \citep{li2026hardflow} is the work most closely related to ours. It reformulates hard-constrained sampling from a flow matching model as a trajectory optimization problem in which constraints are imposed only at the terminal time, uses model predictive control ideas to decompose the long-horizon problem into a sequence of single-step subproblems, allows integral costs that limit distribution shift, and provides a control-theoretic bound on the gap between its surrogate and the ideal problem. It substantially outperforms projection-based methods on robotics, PDE boundary-control, and text-guided image-editing tasks, and it is therefore our representative of the hard-constraint family. The differences from ConicFlow are, first, that HardFlow optimizes a single step at a time and solves each constrained subproblem with an off-the-shelf solver, whereas we optimize over a horizon $H\ge 1$ with a dynamic primal--dual scheme whose multipliers are warm-started across windows; second, that HardFlow evaluates the constraint at the end of the current step, whereas our RH-T variant evaluates it at the predicted terminal point; and third, that we provide truncation and finite-iteration guarantees for both variants. Empirically, HardFlow is competitive precisely when large deviations from the pretrained flow are required, complementing FMRG's strength in the small-deviation regime, and ConicFlow combines the strengths of both (Section~\ref{sec:experiments}).

\paragraph{Receding-horizon and constrained control.}
Receding-horizon control provides a classical framework for handling
finite-horizon objectives and hard constraints through repeated online
optimization~\citep{garcia1989model,mayne2000constrained}.
Several works have addressed its computational and robustness limitations.
In particular, Kerrigan and collaborators studied move blocking to reduce the
online decision dimension while retaining feasibility and stability guarantees
\cite{cagienard2007move}, convex parametrizations of feedback policies for
robust constrained finite-horizon control~\citep{goulart2006optimization},
and constrained output-feedback receding-horizon control under process and
measurement uncertainty~\citep{goulart2007output}.
Related work has developed high-speed RHC solvers
\cite{mattingley2011receding} and scalable localized MPC architectures
\cite{alonso2022distributed}.
Our setting differs in that the controlled dynamics are given by a pretrained
generative flow and the objective and hard constraints are imposed on the
generated terminal sample.

\clearpage
\section{Proofs}
\subsection{Proof of Lemma~\ref{lem:discretization-error-main}}
\label{app:proof-discretization}

\begin{proof}

\medskip
\noindent\textbf{Proof strategy.}
We prove the result by comparing the continuous and discretized optimal
values in both directions. The argument proceeds in four steps.

First, the quadratic control penalty together with the Lipschitz terminal
reward yields a coercive lower bound on the objective. Using the finite-cost
feasible trajectory from Assumption~\ref{ass:main-discretization} i, this gives a
uniform energy bound on all controls relevant to the comparison and,
consequently, a common bounded region containing their trajectories.

Second, we show that replacing a continuous control by its block averages
and propagating the dynamics with Euler discretization changes the terminal
state by $O(\Delta t)$. Conversely, the piecewise-constant interpolation of
a discrete control incurs the same order of terminal-state error.

Third, the resulting $O(\Delta t)$ terminal constraint violation is corrected
using Assumption~\ref{ass:main-discretization} ii, requiring only an
$O(\Delta t)$ perturbation of the control. This construction also establishes
feasibility and existence of the discrete optimum for sufficiently large $N$.

Finally, we apply these constructions in both directions. Discretizing and
repairing the continuous optimum gives an upper bound on the discrete optimal
value, while interpolating and repairing the discrete optimum gives the
reverse inequality. Combining the two yields
$$
\left|
J_D\bigl(U_D^\star(N)\bigr)
-
J_C(u_C^\star)
\right|
=
O(\Delta t)
=
O\!\left(\frac{1}{N}\right).
$$

We next make the quantitative content of
Assumptions~\ref{ass:main-discretization} from i to iii
explicit and introduce the constants used in the proof.

\medskip
\noindent\textbf{Quantitative form of the assumptions.}
By Assumption~\ref{ass:main-discretization} i, there exist constants
$L_b,L_b^t,L_r,L_h>0$ such that
$$
\|b_t(x)-b_s(y)\|
\leq
L_b\|x-y\|+L_b^t|t-s|,
$$
and
$$
|r(x)-r(y)|
\leq
L_r\|x-y\|,
\qquad
\|h(x)-h(y)\|
\leq
L_h\|x-y\|.
$$
We also define
$$
B_0
:=
\sup_{t\in[0,1]}
\|b_t(x_0)\|
<\infty.
$$

By Assumption~\ref{ass:main-discretization} ii, there exists a feasible
continuous-time control $u^0$ satisfying
$$
h(x_1^{u^0})\leq 0,
\qquad
J_C(u^0)<\infty,
$$
and the continuous problem admits an optimal control $u_C^\star$.

Assumption~\ref{ass:main-discretization} iii is used in the following
quantitative form. On the energy ball introduced below, there exist
constants $\kappa,\rho_0>0$, independent of $N$, such that any continuous
or discrete control whose terminal violation is at most $\rho_0$ can be
made feasible by a control perturbation whose norm is at most $\kappa$
times the violation. For continuous controls we use the $L^2$ norm, while
for discrete controls we use
$$
\|U\|_{\ell^2_{\Delta t}}^2
:=
\sum_{k=0}^{N-1}
\Delta t\,\|u_k\|^2.
$$

\medskip
\noindent\textbf{A priori energy bound.}
We first derive the energy radius on which the remainder of the proof is
carried out. For any continuous control $u$, the dynamics and
Assumption~\ref{ass:main-discretization} i imply
$$
\|x_t^u-x_0\|
\leq
B_0t
+
L_b\int_0^t\|x_s^u-x_0\|\,ds
+
\int_0^t\|u_s\|\,ds.
$$
By Cauchy--Schwarz and Gr\"onwall's inequality,
$$
\sup_{0\leq t\leq 1}
\|x_t^u-x_0\|
\leq
e^{L_b}
\left(
B_0+\|u\|_{L^2}
\right).
$$
Hence, by Lipschitz continuity of $r$,
$$
r(x_1^u)
\leq
|r(x_0)|
+
L_re^{L_b}
\left(
B_0+\|u\|_{L^2}
\right).
$$
Define
$$
a_r:=L_re^{L_b},
\qquad
\beta_r
:=
|r(x_0)|
+
a_rB_0
+
\lambda a_r^2.
$$
Young's inequality gives
$$
a_r\|u\|_{L^2}
\leq
\frac{1}{4\lambda}\|u\|_{L^2}^2
+
\lambda a_r^2.
$$
Therefore,
$$
J_C(u)
\geq
\frac{1}{4\lambda}\|u\|_{L^2}^2
-
\beta_r.
$$

The same argument applied to the Euler recursion gives
$$
J_D(U)
\geq
\frac{1}{4\lambda}
\|U\|_{\ell^2_{\Delta t}}^2
-
\beta_r.
$$

Define
$$
R_\star
:=
4\lambda
\bigl(
J_C(u^0)+\beta_r
\bigr),
\qquad
\overline R:=R_\star+1.
$$
Since $u_C^\star$ is optimal and $u^0$ is feasible,
$$
J_C(u_C^\star)
\leq
J_C(u^0).
$$
Using the coercivity estimate above,
$$
\|u_C^\star\|_{L^2}^2
\leq
R_\star
<
\overline R.
$$

\medskip
\noindent\textbf{Step 1: Uniform confinement of relevant trajectories.}

Consider any continuous control satisfying
$$
\|u\|_{L^2}^2
\leq
\overline R.
$$
The preceding trajectory estimate gives
$$
\sup_{0\leq t\leq 1}
\|x_t^u-x_0\|
\leq
\bigl(B_0+\sqrt{\overline R}\bigr)e^{L_b}.
$$
Define
$$
\rho_{\overline R}
:=
\bigl(B_0+\sqrt{\overline R}\bigr)e^{L_b},
$$
and
$$
M_{\overline R}
:=
B_0+L_b\rho_{\overline R}.
$$
Then
$$
\|b_t(x_t^u)\|
\leq
M_{\overline R}.
$$

The same conclusion holds for Euler trajectories. Let
$U=(u_0,\ldots,u_{N-1})$ satisfy
$$
\|U\|_{\ell^2_{\Delta t}}^2
\leq
\overline R.
$$
The Euler recursion gives
$$
\|x_{k+1}^U-x_0\|
\leq
(1+L_b\Delta t)
\|x_k^U-x_0\|
+
B_0\Delta t
+
\Delta t\,\|u_k\|.
$$
Moreover,
$$
\sum_{k=0}^{N-1}
\Delta t\,\|u_k\|
\leq
\left(
\sum_{k=0}^{N-1}
\Delta t\,\|u_k\|^2
\right)^{1/2}
\leq
\sqrt{\overline R}.
$$
Iterating the recursion and using
$(1+L_b\Delta t)^N\leq e^{L_b}$ gives
$$
\max_{0\leq k\leq N}
\|x_k^U-x_0\|
\leq
\rho_{\overline R}.
$$
Thus, both continuous and Euler trajectories generated by controls in the
energy ball $\overline R$ remain uniformly bounded, and the drift along
all such trajectories is bounded by $M_{\overline R}$.

\medskip
\noindent\textbf{Step 2: Continuous--discrete approximation.}

We next establish the approximation estimates required in the remainder of
the proof.

\smallskip
\noindent\emph{Continuous-to-discrete approximation.}
Let $u\in L^2([0,1];\mathbb R^d)$ satisfy
$$
\|u\|_{L^2}^2
\leq
\overline R,
$$
and define its block averages by
$$
\overline u_k
:=
\frac{1}{\Delta t}
\int_{t_k}^{t_{k+1}}u_s\,ds.
$$
By Jensen's inequality,
$$
\sum_{k=0}^{N-1}
\Delta t\,\|\overline u_k\|^2
\leq
\int_0^1\|u_t\|^2\,dt
\leq
\overline R.
$$

Let $x^u(t)$ denote the exact continuous trajectory and let $x_k^\Delta$
denote the Euler trajectory driven by
$\overline U=(\overline u_0,\ldots,\overline u_{N-1})$:
$$
x_{k+1}^\Delta
=
x_k^\Delta
+
\Delta t
\bigl(
b_{t_k}(x_k^\Delta)+\overline u_k
\bigr).
$$
Set
$$
e_k
:=
x^u(t_k)-x_k^\Delta.
$$
By construction,
$$
\int_{t_k}^{t_{k+1}}
\bigl(
u_s-\overline u_k
\bigr)\,ds
=
0.
$$
Thus the integrated control contribution is reproduced exactly over each
grid interval. Subtracting the Euler recursion from the exact dynamics gives
$$
e_{k+1}
=
e_k
+
\Delta t
\bigl(
b_{t_k}(x^u(t_k))
-
b_{t_k}(x_k^\Delta)
\bigr)
+
\tau_k,
$$
where
$$
\tau_k
:=
\int_{t_k}^{t_{k+1}}
\left[
b_s(x^u(s))
-
b_{t_k}(x^u(t_k))
\right]ds.
$$

Define
$$
q_k
:=
\int_{t_k}^{t_{k+1}}
\|u_s\|^2\,ds.
$$
For $s\in[t_k,t_{k+1}]$, Step~1 and Cauchy--Schwarz imply
$$
\|x^u(s)-x^u(t_k)\|
\leq
M_{\overline R}(s-t_k)
+
\sqrt{s-t_k}\,\sqrt{q_k}.
$$
Therefore,
$$
\|\tau_k\|
\leq
L_bM_{\overline R}\Delta t^2
+
\frac{1}{2}L_b^t\Delta t^2
+
L_b\Delta t^{3/2}\sqrt{q_k}.
$$
Moreover,
$$
\sum_{k=0}^{N-1}\sqrt{q_k}
\leq
\sqrt{N}
\left(
\sum_{k=0}^{N-1}q_k
\right)^{1/2}
\leq
\frac{\sqrt{\overline R}}{\sqrt{\Delta t}}.
$$
Consequently,
$$
\sum_{k=0}^{N-1}\|\tau_k\|
\leq
\left(
L_bM_{\overline R}
+
\frac{1}{2}L_b^t
+
L_b\sqrt{\overline R}
\right)
\Delta t.
$$

Since
$$
\|e_{k+1}\|
\leq
(1+L_b\Delta t)\|e_k\|
+
\|\tau_k\|,
$$
discrete Gr\"onwall yields
$$
\max_{0\leq k\leq N}
\|x^u(t_k)-x_k^\Delta\|
\leq
C_0\Delta t,
$$
where
$$
C_0
:=
e^{L_b}
\left(
L_bM_{\overline R}
+
\frac{1}{2}L_b^t
+
L_b\sqrt{\overline R}
\right).
$$

\smallskip
\noindent\emph{Discrete-to-continuous interpolation.}
Conversely, let $U=(u_0,\ldots,u_{N-1})$ satisfy
$$
\|U\|_{\ell^2_{\Delta t}}^2
\leq
\overline R,
$$
and define its piecewise-constant interpolation
$$
\widehat u(t)
:=
u_k,
\qquad
t\in[t_k,t_{k+1}).
$$
Then
$$
\|\widehat u\|_{L^2}^2
=
\|U\|_{\ell^2_{\Delta t}}^2.
$$
Since the block average of $\widehat u$ on each interval is exactly $u_k$,
the preceding argument applies directly and gives
$$
\max_{0\leq k\leq N}
\|\widehat x(t_k)-x_k^U\|
\leq
C_0\Delta t,
$$
where $\widehat x$ is the exact continuous trajectory generated by
$\widehat u$.

\smallskip
\noindent\emph{Sensitivity to control perturbations.}
Let $u,v\in L^2([0,1];\mathbb R^d)$. Then
$$
\|x_t^u-x_t^v\|
\leq
L_b
\int_0^t
\|x_s^u-x_s^v\|\,ds
+
\int_0^t
\|u_s-v_s\|\,ds.
$$
Gr\"onwall and Cauchy--Schwarz imply
$$
\sup_{0\leq t\leq 1}
\|x_t^u-x_t^v\|
\leq
e^{L_b}
\|u-v\|_{L^2}.
$$
Similarly, for two discrete controls $U$ and $V$,
$$
\max_{0\leq k\leq N}
\|x_k^U-x_k^V\|
\leq
e^{L_b}
\|U-V\|_{\ell^2_{\Delta t}}.
$$

Thus Step~2 provides the two estimates used below:
$$
\text{Euler endpoint error}=O(\Delta t),
\qquad
\text{control perturbation }\delta
\Longrightarrow
\text{state perturbation }O(\delta).
$$

\medskip
\noindent\textbf{Step 3: Discrete feasibility, existence, and energy bound.}

Let $\overline U^0$ denote the block average of the feasible reference
control $u^0$. From the coercivity estimate,
$$
\|u^0\|_{L^2}^2
\leq
R_\star,
$$
and Jensen's inequality gives
$$
\|\overline U^0\|_{\ell^2_{\Delta t}}^2
\leq
R_\star
<
\overline R.
$$
By Step~2,
$$
\|x_N^{\overline U^0}-x_1^{u^0}\|
\leq
C_0\Delta t.
$$
Since $u^0$ is feasible and $h$ is $L_h$-Lipschitz,
$$
\|[h(x_N^{\overline U^0})]_+\|
\leq
L_hC_0\Delta t.
$$

Define
$$
N_0
:=
\left\lceil
\max\left\{
1,
\frac{L_hC_0}{\rho_0}
\right\}
\right\rceil.
$$
For $N\geq N_0$, Assumption~\ref{ass:main-discretization} iii provides a
feasible discrete control $\widetilde U^0$ satisfying
$$
\|\widetilde U^0-\overline U^0\|_{\ell^2_{\Delta t}}
\leq
\delta_N,
$$
where
$$
\delta_N
:=
\kappa L_hC_0\Delta t.
$$

The change in quadratic control cost satisfies
$$
\frac{1}{2\lambda}
\left(
\|\widetilde U^0\|_{\ell^2_{\Delta t}}^2
-
\|\overline U^0\|_{\ell^2_{\Delta t}}^2
\right)
\leq
\frac{\sqrt{\overline R}}{\lambda}\delta_N
+
\frac{1}{2\lambda}\delta_N^2.
$$
Moreover, the control-sensitivity estimate and Lipschitz continuity of $r$
give
$$
\left|
r(x_N^{\widetilde U^0})
-
r(x_N^{\overline U^0})
\right|
\leq
L_re^{L_b}\delta_N.
$$
Using also the Euler endpoint estimate and the fact that block averaging
cannot increase the quadratic control cost,
$$
J_D(\widetilde U^0)
\leq
J_C(u^0)
+
L_rC_0\Delta t
+
\left(
\frac{\sqrt{\overline R}}{\lambda}
+
L_re^{L_b}
\right)
\delta_N
+
\frac{1}{2\lambda}\delta_N^2.
$$

Since $\Delta t^2\leq\Delta t$, define
$$
C_{\mathrm{disc}}
:=
C_0
\left[
L_r
+
\left(
\frac{\sqrt{\overline R}}{\lambda}
+
L_re^{L_b}
\right)
\kappa L_h
+
\frac{\kappa^2L_h^2C_0}{2\lambda}
\right].
$$
Then
$$
J_D(\widetilde U^0)
\leq
J_C(u^0)
+
C_{\mathrm{disc}}\Delta t.
$$
Thus the discrete feasible set is nonempty for all $N\geq N_0$.

For fixed $N$, the map
$$
U\mapsto x_N^U
$$
is continuous, and therefore the feasible set is closed. Moreover,
$$
J_D(U)
\geq
\frac{1}{4\lambda}
\|U\|_{\ell^2_{\Delta t}}^2
-
\beta_r,
$$
so $J_D$ is coercive. Since the discrete control space is
finite-dimensional, the discrete problem admits a global minimizer
$U_D^\star(N)$.

By optimality,
$$
J_D\bigl(U_D^\star(N)\bigr)
\leq
J_D(\widetilde U^0)
\leq
J_C(u^0)
+
C_{\mathrm{disc}}\Delta t.
$$
Using the coercivity estimate again,
$$
\|U_D^\star(N)\|_{\ell^2_{\Delta t}}^2
\leq
R_\star
+
4\lambda C_{\mathrm{disc}}\Delta t.
$$
Define
$$
N_2
:=
\left\lceil
\max\left\{
1,
4\lambda C_{\mathrm{disc}}
\right\}
\right\rceil,
\qquad
N_1:=\max\{N_0,N_2\}.
$$
Then, for every $N\geq N_1$,
$$
\|U_D^\star(N)\|_{\ell^2_{\Delta t}}^2
\leq
R_\star+1
=
\overline R.
$$

\medskip
\noindent\textbf{Step 4: Two-sided comparison of the optimal values.}

We now apply the preceding constructions to the two optimal controls.

\smallskip
\noindent\emph{Continuous optimum to discrete feasible control.}
Since
$$
\|u_C^\star\|_{L^2}^2
\leq
R_\star
<
\overline R,
$$
let $\overline U^\star$ denote the block average of $u_C^\star$. By
Jensen's inequality,
$$
\|\overline U^\star\|_{\ell^2_{\Delta t}}^2
\leq
R_\star.
$$
Step~2 gives
$$
\|x_N^{\overline U^\star}-x_1^{u_C^\star}\|
\leq
C_0\Delta t.
$$
Since $u_C^\star$ is feasible,
$$
\|[h(x_N^{\overline U^\star})]_+\|
\leq
L_hC_0\Delta t.
$$
For $N\geq N_0$, terminal repairability provides a feasible discrete
control $\widetilde U^\star$ satisfying
$$
\|\widetilde U^\star-\overline U^\star\|_{\ell^2_{\Delta t}}
\leq
\delta_N.
$$
Repeating the preceding objective estimate gives
$$
J_D(\widetilde U^\star)
\leq
J_C(u_C^\star)
+
C_{\mathrm{disc}}\Delta t.
$$
Since $\widetilde U^\star$ is feasible for the discrete problem,
$$
J_D\bigl(U_D^\star(N)\bigr)
-
J_C(u_C^\star)
\leq
C_{\mathrm{disc}}\Delta t.
$$

\smallskip
\noindent\emph{Discrete optimum to continuous feasible control.}
For $N\geq N_1$, Step~3 gives
$$
\|U_D^\star(N)\|_{\ell^2_{\Delta t}}^2
\leq
\overline R.
$$
Define its piecewise-constant interpolation by
$$
\widehat u_D^\star(t)
:=
u_{D,k}^\star,
\qquad
t\in[t_k,t_{k+1}).
$$
The quadratic control energy is preserved exactly:
$$
\|\widehat u_D^\star\|_{L^2}^2
=
\|U_D^\star(N)\|_{\ell^2_{\Delta t}}^2.
$$
Let $\widehat x_D^\star$ denote the corresponding exact continuous
trajectory. Step~2 gives
$$
\|\widehat x_D^\star(1)-x_N^{U_D^\star(N)}\|
\leq
C_0\Delta t.
$$
Since $U_D^\star(N)$ is feasible for the discrete problem,
$$
\|[h(\widehat x_D^\star(1))]_+\|
\leq
L_hC_0\Delta t.
$$
For $N\geq N_0$, Assumption~\ref{ass:main-discretization} iii provides a
feasible continuous control $\widetilde u_D^\star$ satisfying
$$
\|\widetilde u_D^\star-\widehat u_D^\star\|_{L^2}
\leq
\delta_N.
$$
The piecewise-constant interpolation preserves the quadratic running cost
exactly, and the same endpoint and repair estimates therefore yield
$$
J_C(\widetilde u_D^\star)
\leq
J_D\bigl(U_D^\star(N)\bigr)
+
C_{\mathrm{disc}}\Delta t.
$$
Since $\widetilde u_D^\star$ is feasible for the continuous problem,
optimality of $u_C^\star$ implies
$$
J_C(u_C^\star)
-
J_D\bigl(U_D^\star(N)\bigr)
\leq
C_{\mathrm{disc}}\Delta t.
$$

Combining the two directions gives
$$
\left|
J_D\bigl(U_D^\star(N)\bigr)
-
J_C(u_C^\star)
\right|
\leq
C_{\mathrm{disc}}\Delta t
=
\frac{C_{\mathrm{disc}}}{N},
$$
for every $N\geq N_1$. This proves the claim.

\end{proof}

\begin{remark}[What determines the discretization rate]
The proof separates the error into two mechanisms. Regularity of the flow,
together with block averaging of the additive control, gives the first-order
Euler endpoint estimate
$$
\|x_1^u-x_N^{\overline U}\|
=
O(\Delta t).
$$
Local terminal repairability then converts the resulting
$O(\Delta t)$ constraint violation into an $O(\Delta t)$ perturbation of
the control. Consequently, enforcing exact terminal feasibility does not
change the first-order discretization rate.
\end{remark}

\clearpage
\subsection{Unconstrained Receding-Horizon Bounds}
\label{app:unconstrained-rh-bounds}

For completeness, we record the unconstrained counterparts of the
constrained receding-horizon bounds in the main text. These results are
special cases in which the terminal feasibility constraint is absent.

\begin{lemma}[Unconstrained RH-H bound]
\label{lem:uc-rhh-affine}
Suppose the discretized dynamics are affine,
$$
x_{k+1}
=
x_k+\Delta t(A_kx_k+c_k+u_k),
\qquad
\Delta t=\frac{1}{N},
$$
with $\|A_kx+c_k\|\leq M$ on the relevant reachable set, and let $r$ be
concave and $L_r$-Lipschitz. For $1\leq H\leq N$, let
$\bx^{H,\star}$ denote the exact unconstrained RH-H closed-loop trajectory
with horizon $H$, and let $\bx^\star$ denote the full-horizon optimum.
Then
$$
0\leq
J(\bx^{H,\star})-J(\bx^\star)
\leq
C_H^{\rm U}\frac{N-H}{N},
$$
where
$$
C_H^{\rm U}
:=
2\max\left\{
L_rM+\frac{\lambda L_r^2}{2},
\frac{M^2}{2\lambda}
\right\}.
$$
In particular, $C_H^{\rm U}$ is independent of $N$ and $H$.
\end{lemma}

\begin{proof}
Recall that
$$
J_{j,L}^\star(x)
:=
\min_{u_{j:j+L-1}}
\left\{
\sum_{k=j}^{j+L-1}
\frac{\Delta t}{2\lambda}\|u_k\|^2
-
r(x_{j+L})
\right\},
$$
subject to
$$
x_{k+1}
=
x_k+\Delta t(A_kx_k+c_k+u_k),
\qquad
x_j=x.
$$
We write $J_j^\star(x):=J_{j,N-j}^\star(x)$ for the full
remaining-horizon value.

Under the affine dynamics, the terminal state is affine in the control
sequence. Since $r$ is concave, $-r$ is convex, while the quadratic control
cost is coercive. Hence all finite-horizon unconstrained problems considered
below attain their minima.

\medskip
\noindent\textbf{Step 1: Local sensitivity to the terminal horizon.}

Fix $j,L,q$ such that $j+L+q\leq N$, and set
$$
\tau:=q\Delta t.
$$
We first show that
$$
\left|
J_{j,L+q}^\star(x)-J_{j,L}^\star(x)
\right|
\leq
C_{\mathrm{loc}}\tau,
$$
where
$$
C_{\mathrm{loc}}
:=
\max\left\{
L_rM+\frac{\lambda L_r^2}{2},
\frac{M^2}{2\lambda}
\right\}.
$$

Let $u_{j:j+L+q-1}^\star$ be optimal for
$J_{j,L+q}^\star(x)$, with associated trajectory $x_k^\star$, and denote
the energy of its last $q$ controls by
$$
E
:=
\sum_{k=j+L}^{j+L+q-1}
\Delta t\,\|u_k^\star\|^2.
$$
Using only the first $L$ controls gives an admissible candidate for
$J_{j,L}^\star(x)$. Therefore,
$$
\begin{aligned}
J_{j,L}^\star(x)-J_{j,L+q}^\star(x)
&\leq
r(x_{j+L+q}^\star)-r(x_{j+L}^\star)
-\frac{E}{2\lambda}.
\end{aligned}
$$
By the $L_r$-Lipschitz continuity of $r$,
$$
J_{j,L}^\star(x)-J_{j,L+q}^\star(x)
\leq
L_r
\|x_{j+L+q}^\star-x_{j+L}^\star\|
-\frac{E}{2\lambda}.
$$

Using the dynamics and the drift bound
$\|A_kx_k+c_k\|\leq M$ from
Assumption~\ref{ass:rh-affine},
$$
\begin{aligned}
\|x_{j+L+q}^\star-x_{j+L}^\star\|
&\leq
\sum_{k=j+L}^{j+L+q-1}
\Delta t
\left(
M+\|u_k^\star\|
\right) \\
&\leq
M\tau
+
\sqrt{\tau E},
\end{aligned}
$$
where the second inequality follows from Cauchy--Schwarz. Hence
$$
J_{j,L}^\star(x)-J_{j,L+q}^\star(x)
\leq
L_rM\tau
+
L_r\sqrt{\tau E}
-
\frac{E}{2\lambda}.
$$
Completing the square in $\sqrt{E}$ gives
$$
L_r\sqrt{\tau E}
-
\frac{E}{2\lambda}
\leq
\frac{\lambda L_r^2}{2}\tau,
$$
and therefore
$$
J_{j,L}^\star(x)-J_{j,L+q}^\star(x)
\leq
\left(
L_rM+\frac{\lambda L_r^2}{2}
\right)\tau.
$$

For the reverse direction, let
$u_{j:j+L-1}^\star$ be optimal for $J_{j,L}^\star(x)$ and let
$\bar x:=x_{j+L}^\star$ be its terminal state. Extend this trajectory over
the next $q$ steps by choosing
$$
u_k
=
-\bigl(A_k\bar x+c_k\bigr),
\qquad
k=j+L,\ldots,j+L+q-1.
$$
Then
$$
x_{k+1}
=
\bar x,
$$
so the state, and hence the terminal reward, remains unchanged throughout
the extension. Since $\|A_k\bar x+c_k\|\leq M$, the additional control
cost is at most
$$
\sum_{k=j+L}^{j+L+q-1}
\frac{\Delta t}{2\lambda}
\|u_k\|^2
\leq
\frac{M^2}{2\lambda}\tau.
$$
Thus,
$$
J_{j,L+q}^\star(x)-J_{j,L}^\star(x)
\leq
\frac{M^2}{2\lambda}\tau.
$$
Combining the two directions yields
$$
\left|
J_{j,L+q}^\star(x)-J_{j,L}^\star(x)
\right|
\leq
C_{\mathrm{loc}}q\Delta t.
$$

\medskip
\noindent\textbf{Step 2: Accumulation along the RH-H trajectory.}

Set
$$
m:=N-H,
$$
so that $m\Delta t=(N-H)/N$. For $j=0,\ldots,m$, define
$$
W_j(x):=J_{j,H}^\star(x).
$$
At $j=m$, the horizon $H$ reaches the true terminal time, and therefore
$$
W_m(x)=J_m^\star(x).
$$

Let $x_j^{H,\star}$ denote the RH-H closed-loop state and
$u_j^{H,\star}$ the first control returned by the $H$-step problem at
time $j$. By optimality of the truncated problem,
$$
W_j(x_j^{H,\star})
=
\frac{\Delta t}{2\lambda}
\|u_j^{H,\star}\|^2
+
J_{j+1,H-1}^\star(x_{j+1}^{H,\star}),
$$
for $j=0,\ldots,m-1$. Therefore,
$$
\begin{aligned}
\frac{\Delta t}{2\lambda}
\|u_j^{H,\star}\|^2
&=
W_j(x_j^{H,\star})
-
W_{j+1}(x_{j+1}^{H,\star}) \\
&\quad+
\Bigl[
W_{j+1}(x_{j+1}^{H,\star})
-
J_{j+1,H-1}^\star(x_{j+1}^{H,\star})
\Bigr].
\end{aligned}
$$
Applying the local horizon estimate with $q=1$ gives
$$
W_{j+1}(x_{j+1}^{H,\star})
-
J_{j+1,H-1}^\star(x_{j+1}^{H,\star})
\leq
C_{\mathrm{loc}}\Delta t.
$$
Hence
$$
\frac{\Delta t}{2\lambda}
\|u_j^{H,\star}\|^2
\leq
W_j(x_j^{H,\star})
-
W_{j+1}(x_{j+1}^{H,\star})
+
C_{\mathrm{loc}}\Delta t.
$$
Summing from $j=0$ to $m-1$ yields
$$
\sum_{j=0}^{m-1}
\frac{\Delta t}{2\lambda}
\|u_j^{H,\star}\|^2
\leq
W_0(x_0)
-
W_m(x_m^{H,\star})
+
mC_{\mathrm{loc}}\Delta t.
$$

From time $m=N-H$ onward, the remaining horizon is at most $H$, so RH-H
coincides with full remaining-horizon re-optimization. Consequently, its
remaining cost from $x_m^{H,\star}$ is exactly
$J_m^\star(x_m^{H,\star})=W_m(x_m^{H,\star})$. Thus
$$
J(\bx^{H,\star})
\leq
W_0(x_0)
+
mC_{\mathrm{loc}}\Delta t.
$$

Finally, $W_0(x_0)=J_{0,H}^\star(x_0)$, while
$J(\bx^\star)=J_{0,N}^\star(x_0)$. Applying the local horizon estimate once
more, now with $L=H$ and $q=m$, gives
$$
W_0(x_0)-J(\bx^\star)
\leq
C_{\mathrm{loc}}m\Delta t.
$$
Therefore,
$$
J(\bx^{H,\star})-J(\bx^\star)
\leq
2C_{\mathrm{loc}}m\Delta t.
$$
Since $\bx^\star$ is globally optimal,
$$
J(\bx^{H,\star})-J(\bx^\star)\geq0.
$$
Using $m=N-H$ and $\Delta t=1/N$,
$$
0
\leq
J(\bx^{H,\star})-J(\bx^\star)
\leq
2C_{\mathrm{loc}}\frac{N-H}{N}.
$$
Defining
$$
C_H^{\rm U}
:=
2C_{\mathrm{loc}}
=
2\max\left\{
L_rM+\frac{\lambda L_r^2}{2},
\frac{M^2}{2\lambda}
\right\}
$$
proves the result.
\end{proof}

\begin{lemma}[Unconstrained RH-T bound]
\label{lem:uc-rht-affine}
Suppose the hypotheses of Lemma~\ref{lem:uc-rhh-affine} hold and, in
addition,
$$
\kappa_A
:=
\sup_{0\leq k\leq\ell\leq N}
\left\|
\prod_{s=k}^{\ell-1}(I+\Delta tA_s)
\right\|
<\infty.
$$
For $1\leq H\leq N$, let $\widehat{\bx}^{H,\star}$ denote the exact
unconstrained RH-T closed-loop trajectory obtained by optimizing controls
over an $H$-step window and propagating the local endpoint to time $N$
under the uncontrolled pretrained flow. Then
$$
0\leq
J(\widehat{\bx}^{H,\star})-J(\bx^\star)
\leq
C_T^{\rm U}\frac{N-H}{N},
$$
where
$$
C_T^{\rm U}
:=
\frac{\lambda}{2}\kappa_A^2L_r^2.
$$
In particular, $C_T^{\rm U}$ is independent of $N$ and $H$.
\end{lemma}

\begin{proof}
As in the proof of Lemma~\ref{lem:uc-rhh-affine}, the quadratic control
penalty and Lipschitz terminal reward ensure that all finite-dimensional
optimization problems below attain their minima. We therefore focus only
on the additional structure introduced by terminal rollout.

For $j+L\leq N$, define the rollout value
$$
\widehat J_{j,L}^\star(x)
:=
\min_{u_{j:j+L-1}}
\left\{
\sum_{k=j}^{j+L-1}
\frac{\Delta t}{2\lambda}\|u_k\|^2
-
r(\widehat x_N)
\right\},
$$
where the state follows
$$
x_{k+1}
=
x_k+\Delta t(A_kx_k+c_k+u_k)
$$
for $k=j,\ldots,j+L-1$, and is subsequently propagated to time $N$
under the uncontrolled flow, i.e., $u_k=0$ for $k\geq j+L$.

\medskip
\noindent\textbf{Step 1: Terminal rollout does not accumulate local horizon errors.}

Increasing the control horizon only enlarges the feasible set of control
sequences. Hence
$$
J_j^\star(x)
\leq
\widehat J_{j,L+1}^\star(x)
\leq
\widehat J_{j,L}^\star(x).
$$

Let
$$
m:=N-H,
$$
and define
$$
\widehat W_j(x)
:=
\widehat J_{j,H}^\star(x),
\qquad
j=0,\ldots,m.
$$
At $j=m$, exactly $H$ steps remain until the true terminal time, and
therefore
$$
\widehat W_m(x)
=
J_m^\star(x).
$$

Let $\widehat x_j^{H,\star}$ denote the RH-T closed-loop state and
$\widehat u_j^{H,\star}$ the first control selected at time $j$. By
optimality of the $H$-step rollout problem,
$$
\widehat W_j(\widehat x_j^{H,\star})
=
\frac{\Delta t}{2\lambda}
\|\widehat u_j^{H,\star}\|^2
+
\widehat J_{j+1,H-1}^\star
(\widehat x_{j+1}^{H,\star}).
$$
Since allowing one additional control action cannot increase the optimal
value,
$$
\widehat W_{j+1}(\widehat x_{j+1}^{H,\star})
=
\widehat J_{j+1,H}^\star
(\widehat x_{j+1}^{H,\star})
\leq
\widehat J_{j+1,H-1}^\star
(\widehat x_{j+1}^{H,\star}).
$$
Consequently,
$$
\frac{\Delta t}{2\lambda}
\|\widehat u_j^{H,\star}\|^2
\leq
\widehat W_j(\widehat x_j^{H,\star})
-
\widehat W_{j+1}(\widehat x_{j+1}^{H,\star}).
$$
Summing from $j=0$ to $m-1$ gives
$$
\sum_{j=0}^{m-1}
\frac{\Delta t}{2\lambda}
\|\widehat u_j^{H,\star}\|^2
\leq
\widehat W_0(x_0)
-
\widehat W_m(\widehat x_m^{H,\star}).
$$

From time $m$ onward, the control horizon reaches the true terminal time,
so the remaining closed-loop cost is
$J_m^\star(\widehat x_m^{H,\star})
=\widehat W_m(\widehat x_m^{H,\star})$. Therefore,
$$
J(\widehat{\bx}^{H,\star})
\leq
\widehat J_{0,H}^\star(x_0).
$$
Since the full-horizon problem has a larger control set,
$$
J(\bx^\star)
=
J_0^\star(x_0)
\leq
\widehat J_{0,H}^\star(x_0).
$$
It is therefore sufficient to bound the single initial rollout gap
$\widehat J_{0,H}^\star(x_0)-J_0^\star(x_0)$.

\medskip
\noindent\textbf{Step 2: Bounding the initial rollout gap.}

Let $u_0^\star,\ldots,u_{N-1}^\star$ denote the controls associated with
the full-horizon optimum $\bx^\star$. Define
$$
\tau_H
:=
(N-H)\Delta t
=
\frac{N-H}{N},
$$
and let
$$
E_H
:=
\sum_{k=H}^{N-1}
\Delta t\,\|u_k^\star\|^2
$$
be the control energy of its omitted tail.

Construct a feasible candidate for the $H$-step rollout problem by keeping
the first $H$ optimal controls and setting all subsequent controls to zero:
$$
\widetilde u_k
=
\begin{cases}
u_k^\star, & k=0,\ldots,H-1,\\
0, & k=H,\ldots,N-1.
\end{cases}
$$
Let $\widetilde x_k$ denote the resulting trajectory. Since the two
trajectories coincide at time $H$, define
$$
\delta_k:=x_k^\star-\widetilde x_k,
\qquad
\delta_H=0.
$$
For $k\geq H$, affine dynamics give
$$
\delta_{k+1}
=
(I+\Delta t A_k)\delta_k
+
\Delta t\,u_k^\star.
$$
Writing
$$
\Phi_{\ell,k}
:=
\prod_{s=k}^{\ell-1}
(I+\Delta t A_s),
\qquad
\Phi_{k,k}:=I,
$$
we obtain
$$
\delta_N
=
\sum_{k=H}^{N-1}
\Delta t\,
\Phi_{N,k+1}u_k^\star.
$$
By the definition of $\kappa_A$,
$$
\|\delta_N\|
\leq
\kappa_A
\sum_{k=H}^{N-1}
\Delta t\,\|u_k^\star\|.
$$
Cauchy--Schwarz then gives
$$
\|\delta_N\|
\leq
\kappa_A
\sqrt{\tau_HE_H}.
$$

Because $\widetilde u$ is feasible for the rollout problem,
$$
\widehat J_{0,H}^\star(x_0)
-
J_0^\star(x_0)
\leq
J(\widetilde{\bx})-J(\bx^\star).
$$
The first $H$ control costs cancel, while the omitted tail energy appears
with a negative sign:
$$
J(\widetilde{\bx})-J(\bx^\star)
=
-\frac{E_H}{2\lambda}
+
r(x_N^\star)-r(\widetilde x_N).
$$
Using the $L_r$-Lipschitz continuity of $r$,
$$
\widehat J_{0,H}^\star(x_0)
-
J_0^\star(x_0)
\leq
L_r\kappa_A\sqrt{\tau_HE_H}
-
\frac{E_H}{2\lambda}.
$$
Completing the square in $\sqrt{E_H}$ yields
$$
L_r\kappa_A\sqrt{\tau_HE_H}
-
\frac{E_H}{2\lambda}
\leq
\frac{\lambda}{2}
\kappa_A^2L_r^2\tau_H.
$$
Therefore,
$$
0
\leq
\widehat J_{0,H}^\star(x_0)
-
J_0^\star(x_0)
\leq
\frac{\lambda}{2}
\kappa_A^2L_r^2
\frac{N-H}{N}.
$$

Combining this estimate with Step~1 and using the global optimality of
$\bx^\star$ gives
$$
0
\leq
J(\widehat{\bx}^{H,\star})
-
J(\bx^\star)
\leq
\frac{\lambda}{2}
\kappa_A^2L_r^2
\frac{N-H}{N}.
$$
Thus,
$$
C_T^{\rm U}
=
\frac{\lambda}{2}\kappa_A^2L_r^2,
$$
which proves the claim.
\end{proof}

\clearpage
\subsection{Proof of Lemma~\ref{lem:c-rhh-affine}}
\label{app:proof-c-rhh-affine}

\begin{proof}
We follow the argument of Lemma~\ref{lem:uc-rhh-affine}, with one additional
step: truncating a feasible long-horizon trajectory need not leave a feasible
endpoint. The uniform Slater condition allows us to restore feasibility by
convexly mixing the truncated trajectory with a strictly feasible one.

For $j+L\leq N$, let
$$
J_{j,L}^{C,\star}(x)
:=
\min_{u_{j:j+L-1}}
\left\{
\sum_{k=j}^{j+L-1}
\frac{\Delta t}{2\lambda}\|u_k\|^2
-r(x_{j+L})
\right\}
$$
subject to the affine dynamics,
$x_j=x$, and the terminal constraint $h(x_{j+L})\leq0$.
We write
$$
J_j^{C,\star}(x)
:=
J_{j,N-j}^{C,\star}(x).
$$

Recall the constants
$$
\chi_h
:=
\frac{L_h(\overline J-\underline J)}{\gamma},
\qquad
G_h:=L_r+\chi_h,
$$
where $\gamma$ is the uniform Slater margin, $\overline J$ uniformly bounds
the objective of the strictly feasible Slater trajectories, and
$\underline J$ is a uniform lower bound on the relevant finite-horizon
objectives. The latter exists by the quadratic control penalty and the
Lipschitz terminal reward, as in the coercivity argument used previously.

\medskip
\noindent\textbf{Step 1: Local constrained horizon sensitivity.}

Fix $j,L,q$ with $j+L+q\leq N$, and set
$$
\tau:=q\Delta t.
$$
We show that
$$
\left|
J_{j,L+q}^{C,\star}(x)
-
J_{j,L}^{C,\star}(x)
\right|
\leq
C_{\mathrm{loc}}^{\rm C}\tau,
$$
where
$$
C_{\mathrm{loc}}^{\rm C}
:=
\max\left\{
G_hM+\frac{\lambda G_h^2}{2},
\frac{M^2}{2\lambda}
\right\}.
$$

Let $u_{j:j+L+q-1}^\star$ be optimal for
$J_{j,L+q}^{C,\star}(x)$, with trajectory $x_k^\star$, and let
$$
E
:=
\sum_{k=j+L}^{j+L+q-1}
\Delta t\,\|u_k^\star\|^2
$$
denote the energy in the omitted tail. Define
$$
d
:=
M\tau+\sqrt{\tau E}.
$$
As in the proof of Lemma~\ref{lem:uc-rhh-affine},
$$
\|x_{j+L+q}^\star-x_{j+L}^\star\|
\leq d.
$$

Although $x_{j+L+q}^\star$ is feasible, the truncated endpoint
$x_{j+L}^\star$ need not be. Since
$h(x_{j+L+q}^\star)\leq0$ and $h$ is $L_h$-Lipschitz,
each component satisfies
$$
h_i(x_{j+L}^\star)
\leq
L_h d.
$$
Set
$$
\delta:=L_hd.
$$

By Assumption~\ref{ass:rh-constrained}, the $L$-step problem admits a
strictly feasible Slater trajectory $(x_k^{\rm s},u_k^{\rm s})$ from the
same initial state $x$, satisfying
$$
h(x_{j+L}^{\rm s})
\leq
-\gamma\mathbf 1,
\qquad
J_{j,L}(x^{\rm s},u^{\rm s})
\leq
\overline J.
$$
Choose
$$
\alpha
:=
\frac{\delta}{\gamma+\delta}.
$$
In particular,
$$
0\leq\alpha\leq\frac{\delta}{\gamma}.
$$

Consider the convexly mixed controls
$$
\widetilde u_k
:=
(1-\alpha)u_k^\star+\alpha u_k^{\rm s},
\qquad
k=j,\ldots,j+L-1.
$$
Because the dynamics are affine and both trajectories start from the same
state, their corresponding states satisfy
$$
\widetilde x_k
=
(1-\alpha)x_k^\star+\alpha x_k^{\rm s}.
$$
By componentwise convexity of $h$,
$$
\begin{aligned}
h_i(\widetilde x_{j+L})
&\leq
(1-\alpha)h_i(x_{j+L}^\star)
+
\alpha h_i(x_{j+L}^{\rm s}) \\
&\leq
(1-\alpha)\delta-\alpha\gamma
=
0.
\end{aligned}
$$
Thus the mixed trajectory is feasible for the $L$-step problem.

Since the dynamics are affine, the quadratic running cost is convex in the
controls, and $-r$ is convex because $r$ is concave. Hence the truncated
objective is convex, so
$$
J_{j,L}(\widetilde x,\widetilde u)
\leq
(1-\alpha)
J_{j,L}(x^\star,u^\star)
+
\alpha
J_{j,L}(x^{\rm s},u^{\rm s}).
$$
Using
$J_{j,L}(x^\star,u^\star)\geq\underline J$ and
$J_{j,L}(x^{\rm s},u^{\rm s})\leq\overline J$ gives
$$
J_{j,L}(\widetilde x,\widetilde u)
\leq
J_{j,L}(x^\star,u^\star)
+
\alpha(\overline J-\underline J).
$$
Therefore,
$$
J_{j,L}^{C,\star}(x)
\leq
J_{j,L}(x^\star,u^\star)
+
\alpha(\overline J-\underline J).
$$

The objective of the truncated prefix and that of the original
$(L+q)$-step optimum differ by
$$
\begin{aligned}
J_{j,L}(x^\star,u^\star)
-
J_{j,L+q}^{C,\star}(x)
&=
r(x_{j+L+q}^\star)
-
r(x_{j+L}^\star)
-
\frac{E}{2\lambda} \\
&\leq
L_r d-\frac{E}{2\lambda}.
\end{aligned}
$$
Moreover,
$$
\alpha(\overline J-\underline J)
\leq
\frac{L_h(\overline J-\underline J)}{\gamma}d
=
\chi_h d.
$$
Combining the preceding inequalities,
$$
J_{j,L}^{C,\star}(x)
-
J_{j,L+q}^{C,\star}(x)
\leq
G_h d-\frac{E}{2\lambda}.
$$
Substituting $d=M\tau+\sqrt{\tau E}$,
$$
J_{j,L}^{C,\star}(x)
-
J_{j,L+q}^{C,\star}(x)
\leq
G_hM\tau
+
G_h\sqrt{\tau E}
-
\frac{E}{2\lambda}.
$$
Completing the square gives
$$
G_h\sqrt{\tau E}
-
\frac{E}{2\lambda}
\leq
\frac{\lambda G_h^2}{2}\tau,
$$
and hence
$$
J_{j,L}^{C,\star}(x)
-
J_{j,L+q}^{C,\star}(x)
\leq
\left(
G_hM+\frac{\lambda G_h^2}{2}
\right)\tau.
$$

For the reverse direction, let an optimal trajectory for
$J_{j,L}^{C,\star}(x)$ terminate at $\bar x$. Since
$h(\bar x)\leq0$, extend it for the next $q$ steps using
$$
u_k
=
-\bigl(A_k\bar x+c_k\bigr),
\qquad
k=j+L,\ldots,j+L+q-1.
$$
This keeps the state fixed at $\bar x$, so feasibility and the terminal
reward are preserved. By the drift bound,
$$
\|u_k\|
\leq M,
$$
and therefore the extension costs at most
$$
\frac{M^2}{2\lambda}\tau.
$$
Thus
$$
J_{j,L+q}^{C,\star}(x)
-
J_{j,L}^{C,\star}(x)
\leq
\frac{M^2}{2\lambda}\tau.
$$
Combining the two directions,
$$
\left|
J_{j,L+q}^{C,\star}(x)
-
J_{j,L}^{C,\star}(x)
\right|
\leq
C_{\mathrm{loc}}^{\rm C}q\Delta t.
$$

\medskip
\noindent\textbf{Step 2: Accumulation along the constrained RH-H trajectory.}

The remainder follows the telescoping argument in
Lemma~\ref{lem:uc-rhh-affine}. We include the main steps for completeness.

Let
$$
m:=N-H,
\qquad
W_j^{\rm C}(x):=J_{j,H}^{C,\star}(x),
\qquad
j=0,\ldots,m.
$$
At $j=m$, the horizon reaches the true terminal time, so
$$
W_m^{\rm C}(x)=J_m^{C,\star}(x).
$$

Let $x_j^{H,\star}$ and $u_j^{H,\star}$ denote the constrained RH-H
closed-loop state and applied control. Bellman's principle for the
$H$-step constrained problem gives
$$
W_j^{\rm C}(x_j^{H,\star})
=
\frac{\Delta t}{2\lambda}
\|u_j^{H,\star}\|^2
+
J_{j+1,H-1}^{C,\star}(x_{j+1}^{H,\star}).
$$
Applying the local horizon estimate with $q=1$,
$$
W_{j+1}^{\rm C}(x_{j+1}^{H,\star})
-
J_{j+1,H-1}^{C,\star}(x_{j+1}^{H,\star})
\leq
C_{\mathrm{loc}}^{\rm C}\Delta t.
$$
Consequently,
$$
\frac{\Delta t}{2\lambda}
\|u_j^{H,\star}\|^2
\leq
W_j^{\rm C}(x_j^{H,\star})
-
W_{j+1}^{\rm C}(x_{j+1}^{H,\star})
+
C_{\mathrm{loc}}^{\rm C}\Delta t.
$$
Summing from $j=0$ to $m-1$ yields
$$
\sum_{j=0}^{m-1}
\frac{\Delta t}{2\lambda}
\|u_j^{H,\star}\|^2
\leq
W_0^{\rm C}(x_0)
-
W_m^{\rm C}(x_m^{H,\star})
+
mC_{\mathrm{loc}}^{\rm C}\Delta t.
$$

From time $m=N-H$ onward, the remaining horizon is at most $H$, so RH-H
coincides with full remaining-horizon re-optimization. Its remaining cost
is therefore
$$
J_m^{C,\star}(x_m^{H,\star})
=
W_m^{\rm C}(x_m^{H,\star}).
$$
Hence
$$
J(\bx^{H,\star})
\leq
W_0^{\rm C}(x_0)
+
mC_{\mathrm{loc}}^{\rm C}\Delta t.
$$

Finally,
$$
W_0^{\rm C}(x_0)
=
J_{0,H}^{C,\star}(x_0),
\qquad
J(\bx^\star)
=
J_{0,N}^{C,\star}(x_0).
$$
Applying the local constrained horizon estimate once more with
$L=H$ and $q=m$ gives
$$
W_0^{\rm C}(x_0)-J(\bx^\star)
\leq
C_{\mathrm{loc}}^{\rm C}m\Delta t.
$$
Therefore,
$$
J(\bx^{H,\star})-J(\bx^\star)
\leq
2C_{\mathrm{loc}}^{\rm C}m\Delta t.
$$
Since $\bx^\star$ is globally optimal among feasible trajectories,
$$
J(\bx^{H,\star})-J(\bx^\star)\geq0.
$$
Using $m=N-H$ and $\Delta t=1/N$,
$$
0
\leq
J(\bx^{H,\star})-J(\bx^\star)
\leq
C_H^{\rm C}\frac{N-H}{N},
$$
where
$$
C_H^{\rm C}
:=
2C_{\mathrm{loc}}^{\rm C}
=
2\max\left\{
G_hM+\frac{\lambda G_h^2}{2},
\frac{M^2}{2\lambda}
\right\}.
$$
This proves the claim.
\end{proof}
\subsection{Proof of Lemma~\ref{lem:c-rht-affine}}
\label{app:proof-c-rht-affine}

\begin{proof}
The proof combines the no-accumulation argument of
Lemma~\ref{lem:uc-rht-affine} with the Slater-mixing argument used in
Lemma~\ref{lem:c-rhh-affine}. The only new issue is that truncating the
control horizon and rolling forward with zero control may violate the
terminal constraint; the uniform Slater condition is used to repair this
candidate without changing the order of the bound.

For $j+L\leq N$, define the constrained rollout value
$$
\widehat J_{j,L}^{C,\star}(x)
:=
\min_{u_{j:j+L-1}}
\left\{
\sum_{k=j}^{j+L-1}
\frac{\Delta t}{2\lambda}\|u_k\|^2
-r(\widehat x_N)
\right\},
$$
where the state follows the controlled affine dynamics for
$k=j,\ldots,j+L-1$, is subsequently propagated with $u_k=0$ until time
$N$, and satisfies
$$
h(\widehat x_N)\leq0.
$$
As before, $J_j^{C,\star}(x)$ denotes the constrained full
remaining-horizon value.

Recall
$$
\chi_h
:=
\frac{L_h(\overline J-\underline J)}{\gamma},
\qquad
G_h:=L_r+\chi_h,
$$
where $\underline J$ is a uniform lower bound on the relevant rollout
objectives and $\overline J$ bounds the objective of the corresponding
strictly feasible Slater trajectories.

\medskip
\noindent\textbf{Step 1: Terminal rollout does not accumulate horizon errors.}

The nesting argument from Lemma~\ref{lem:uc-rht-affine} remains valid in
the constrained case. Indeed, any feasible $L$-step rollout control remains
feasible for the $(L+1)$-step problem by setting the additional control to
zero; the resulting terminal state is unchanged. Hence
$$
J_j^{C,\star}(x)
\leq
\widehat J_{j,L+1}^{C,\star}(x)
\leq
\widehat J_{j,L}^{C,\star}(x).
$$

Let
$$
m:=N-H,
\qquad
\widehat W_j^{\rm C}(x)
:=
\widehat J_{j,H}^{C,\star}(x),
\qquad
j=0,\ldots,m.
$$
At $j=m$, the control horizon reaches the true terminal time, so
$$
\widehat W_m^{\rm C}(x)
=
J_m^{C,\star}(x).
$$

Let $\widehat x_j^{H,\star}$ and $\widehat u_j^{H,\star}$ denote the
constrained RH-T closed-loop state and applied control. Optimality gives
$$
\widehat W_j^{\rm C}(\widehat x_j^{H,\star})
=
\frac{\Delta t}{2\lambda}
\|\widehat u_j^{H,\star}\|^2
+
\widehat J_{j+1,H-1}^{C,\star}
(\widehat x_{j+1}^{H,\star}).
$$
By nestedness,
$$
\widehat W_{j+1}^{\rm C}
(\widehat x_{j+1}^{H,\star})
\leq
\widehat J_{j+1,H-1}^{C,\star}
(\widehat x_{j+1}^{H,\star}),
$$
and therefore
$$
\frac{\Delta t}{2\lambda}
\|\widehat u_j^{H,\star}\|^2
\leq
\widehat W_j^{\rm C}(\widehat x_j^{H,\star})
-
\widehat W_{j+1}^{\rm C}(\widehat x_{j+1}^{H,\star}).
$$
Summing over $j=0,\ldots,m-1$ and using
$\widehat W_m^{\rm C}=J_m^{C,\star}$ yields
$$
J(\widehat{\bx}^{H,\star})
\leq
\widehat J_{0,H}^{C,\star}(x_0).
$$
Consequently,
$$
0
\leq
J(\widehat{\bx}^{H,\star})-J(\bx^\star)
\leq
\widehat J_{0,H}^{C,\star}(x_0)
-
J_0^{C,\star}(x_0).
$$
It therefore remains only to bound this initial constrained rollout gap.

\medskip
\noindent\textbf{Step 2: Bounding the constrained rollout gap.}

Let $u_0^\star,\ldots,u_{N-1}^\star$ be the controls associated with the
full constrained optimum $\bx^\star$, and define
$$
\tau_H
:=
(N-H)\Delta t
=
\frac{N-H}{N},
$$
together with its omitted-tail energy
$$
E_H
:=
\sum_{k=H}^{N-1}
\Delta t\,\|u_k^\star\|^2.
$$

As in Lemma~\ref{lem:uc-rht-affine}, retain the first $H$ optimal controls
and set the remaining controls to zero:
$$
\widetilde u_k
=
\begin{cases}
u_k^\star, & k=0,\ldots,H-1,\\
0, & k=H,\ldots,N-1.
\end{cases}
$$
Let $\widetilde x_N$ denote the resulting terminal state. The affine
perturbation argument of Lemma~\ref{lem:uc-rht-affine} gives
$$
\|x_N^\star-\widetilde x_N\|
\leq
\kappa_A\sqrt{\tau_HE_H}.
$$
Since the full-horizon optimum is feasible,
$h(x_N^\star)\leq0$, and $h$ is $L_h$-Lipschitz. Thus each terminal
constraint violation of the truncated rollout is bounded by
$$
h_i(\widetilde x_N)
\leq
\delta_H,
\qquad
\delta_H
:=
L_h\kappa_A\sqrt{\tau_HE_H}.
$$

The zero-tail rollout need not itself be feasible. By
Assumption~\ref{ass:rh-constrained}, however, the $H$-step rollout problem
admits a strictly feasible trajectory $(x^{\rm s},u^{\rm s})$ from the
same initial state satisfying
$$
h(x_N^{\rm s})
\leq
-\gamma\mathbf 1,
\qquad
\widehat J_{0,H}(x^{\rm s},u^{\rm s})
\leq
\overline J.
$$
Set
$$
\alpha
:=
\frac{\delta_H}{\gamma+\delta_H}
\leq
\frac{\delta_H}{\gamma},
$$
and form the mixed rollout controls
$$
u_k^\alpha
:=
(1-\alpha)\widetilde u_k+\alpha u_k^{\rm s},
\qquad
k=0,\ldots,H-1,
$$
with zero control thereafter.

Because the dynamics are affine, the resulting terminal state satisfies
$$
x_N^\alpha
=
(1-\alpha)\widetilde x_N
+
\alpha x_N^{\rm s}.
$$
By componentwise convexity of $h$,
$$
\begin{aligned}
h_i(x_N^\alpha)
&\leq
(1-\alpha)h_i(\widetilde x_N)
+
\alpha h_i(x_N^{\rm s}) \\
&\leq
(1-\alpha)\delta_H-\alpha\gamma
=
0.
\end{aligned}
$$
Hence the mixed rollout is feasible.

The rollout objective is convex in the active controls, since the dynamics
are affine, the running cost is quadratic, and $-r$ is convex. Therefore,
$$
\widehat J_{0,H}(u^\alpha)
\leq
(1-\alpha)\widehat J_{0,H}(\widetilde u)
+
\alpha\widehat J_{0,H}(u^{\rm s}).
$$
Using the uniform lower and upper bounds
$\widehat J_{0,H}(\widetilde u)\geq\underline J$ and
$\widehat J_{0,H}(u^{\rm s})\leq\overline J$,
$$
\widehat J_{0,H}(u^\alpha)
\leq
\widehat J_{0,H}(\widetilde u)
+
\alpha(\overline J-\underline J).
$$
Moreover,
$$
\alpha(\overline J-\underline J)
\leq
\frac{L_h(\overline J-\underline J)}{\gamma}
\kappa_A\sqrt{\tau_HE_H}
=
\chi_h\kappa_A\sqrt{\tau_HE_H}.
$$

Since $u^\alpha$ is feasible for the constrained rollout problem,
$$
\widehat J_{0,H}^{C,\star}(x_0)
\leq
\widehat J_{0,H}(u^\alpha).
$$
On the other hand, comparing the zero-tail trajectory with the full
optimum gives
$$
\widehat J_{0,H}(\widetilde u)
-
J_0^{C,\star}(x_0)
=
-\frac{E_H}{2\lambda}
+
r(x_N^\star)-r(\widetilde x_N).
$$
Using the $L_r$-Lipschitz continuity of $r$,
$$
\widehat J_{0,H}(\widetilde u)
-
J_0^{C,\star}(x_0)
\leq
L_r\kappa_A\sqrt{\tau_HE_H}
-
\frac{E_H}{2\lambda}.
$$
Combining the reward and feasibility-repair contributions,
$$
\widehat J_{0,H}^{C,\star}(x_0)
-
J_0^{C,\star}(x_0)
\leq
G_h\kappa_A\sqrt{\tau_HE_H}
-
\frac{E_H}{2\lambda}.
$$
Completing the square in $\sqrt{E_H}$ gives
$$
G_h\kappa_A\sqrt{\tau_HE_H}
-
\frac{E_H}{2\lambda}
\leq
\frac{\lambda}{2}\kappa_A^2G_h^2\tau_H.
$$
Therefore,
$$
0
\leq
\widehat J_{0,H}^{C,\star}(x_0)
-
J_0^{C,\star}(x_0)
\leq
\frac{\lambda}{2}\kappa_A^2G_h^2
\frac{N-H}{N}.
$$

Combining this estimate with Step~1 yields
$$
0
\leq
J(\widehat{\bx}^{H,\star})
-
J(\bx^\star)
\leq
\frac{\lambda}{2}\kappa_A^2G_h^2
\frac{N-H}{N}.
$$
Thus
$$
C_T^{\rm C}
=
\frac{\lambda}{2}\kappa_A^2G_h^2,
$$
which proves the claim.
\end{proof}
\subsection{Proof of Lemma~\ref{thm: subproblem convergence}}
\label{app:proof-subproblem-convergence}

We first collect two properties of each frozen RH problem that are used in
the proof. Fix a physical index $j$ and current state $\bar x_j$. Let
$H_j=\min\{H,N-j\}$ and write
$\mathbf{x}_j=(x_{j+1},\ldots,x_{j+H_j})$. Under
Assumption~\ref{ass:rh-affine}, the stacked residuals satisfy
$$
\operatorname{col}(e_j,\ldots,e_{j+H_j-1})
=
B_j\mathbf{x}_j-a_j-S_j\bar x_j,
$$
where
$$
B_j=
\begin{bmatrix}
I_d & 0 & \cdots & 0\\
-(I_d+\Delta t A_{j+1}) & I_d & \ddots & \vdots\\
0 & \ddots & \ddots & 0\\
0 & \cdots & -(I_d+\Delta t A_{j+H_j-1}) & I_d
\end{bmatrix},
\qquad
S_j=
\begin{bmatrix}
I_d+\Delta t A_j\\
0\\
\vdots\\
0
\end{bmatrix},
$$
and
$a_j=\Delta t\,\operatorname{col}(c_j,\ldots,c_{j+H_j-1})$.
Recall
$$
\sigma_j
:=
\frac{\sigma_{\min}^2(B_j)}{\lambda\Delta t}>0.
$$

For RH-T, let $Q_j$ denote the linear part of the terminal map $\Phi_j$;
for RH-H, $Q_j=I$. Thus
$$
L_{g,j}:=L_h\|Q_j\|,
\qquad
L_{d,j}:=\frac{L_{g,j}^2}{\sigma_j}.
$$

\begin{lemma}[Strong convexity and well-posedness]
\label{lem:app-pd-strong-convexity}
Under Assumptions~\ref{ass:rh-affine}--\ref{ass:rh-constrained},
$\varphi_j(\cdot;\bar x_j)$ and
$\mathcal L_j(\cdot,\mu;\bar x_j)$ are $\sigma_j$-strongly convex for every
$\mu\geq0$. Moreover, the primal response
$$
\mathbf{x}_j(\mu)
:=
\argmin_{\mathbf{x}_j}
\mathcal L_j(\mathbf{x}_j,\mu;\bar x_j)
$$
exists and is unique, and the constrained local problem has a unique
optimizer $\mathbf{x}_j^\star$.
\end{lemma}

\begin{proof}
Since $B_j$ is block lower triangular with identity diagonal blocks, it is
nonsingular. The Hessian of the quadratic path cost is
$$
\frac{1}{\lambda\Delta t}B_j^\top B_j
\succeq
\sigma_j I.
$$
Moreover, $\Phi_j$ is affine, $r$ is concave, and each component of $h$ is
convex. Hence $-r\circ\Phi_j$ and $g_{j,i}=h_i\circ\Phi_j$ are convex.
Therefore $\varphi_j$ is $\sigma_j$-strongly convex, and so is
$\mathcal L_j=\varphi_j+\mu^\top g_j$ for every $\mu\geq0$.

It remains to show attainment. Let
$X=\mathbf{x}_j$, $\zeta_j=a_j+S_j\bar x_j$, and write
$$
q_j(X)
:=
\frac{1}{2\lambda\Delta t}\|B_jX-\zeta_j\|^2,
\qquad
\psi_\mu(X)
:=
-r(\Phi_j(x_{j+H_j}))+\mu^\top g_j(X).
$$
The finite-valued convex function $\psi_\mu$ admits an affine lower bound:
for any fixed $\widehat X$ and
$s_\mu\in\partial\psi_\mu(\widehat X)$,
$$
\psi_\mu(X)
\geq
\psi_\mu(\widehat X)+s_\mu^\top(X-\widehat X).
$$
Also,
$$
\begin{aligned}
q_j(X)
&\geq
\frac{1}{4\lambda\Delta t}\|B_jX\|^2
-\frac{1}{2\lambda\Delta t}\|\zeta_j\|^2\\
&\geq
\frac{\sigma_{\min}^2(B_j)}
{4\lambda\Delta t}\|X\|^2
-\frac{1}{2\lambda\Delta t}\|\zeta_j\|^2.
\end{aligned}
$$
Thus
$$
\mathcal L_j(X,\mu;\bar x_j)
\geq
\frac{\sigma_{\min}^2(B_j)}
{4\lambda\Delta t}\|X\|^2
-\|s_\mu\|\,\|X\|-C
$$
for a finite constant $C$, and hence the Lagrangian is coercive.
Therefore its minimum is attained; strong convexity makes the minimizer
unique.

Finally, Assumption~\ref{ass:rh-constrained} gives a nonempty convex
feasible set for the frozen RH problem. Since its objective is strongly
convex and coercive, the constrained minimizer exists and is unique.
\end{proof}

\begin{lemma}[Dual regularity]
\label{lem:app-pd-dual-regularity}
Under Assumptions~\ref{ass:rh-affine}--\ref{ass:rh-constrained},
strong duality holds, $\mathcal M_j^\star$ is nonempty and compact, and
$d_j$ is differentiable with
$$
\nabla d_j(\mu)=g_j(\mathbf{x}_j(\mu)).
$$
Furthermore,
$$
\|\mathbf{x}_j(\mu)-\mathbf{x}_j(\nu)\|
\leq
\frac{L_{g,j}}{\sigma_j}\|\mu-\nu\|,
$$
and
$$
\|\nabla d_j(\mu)-\nabla d_j(\nu)\|
\leq
L_{d,j}\|\mu-\nu\|.
$$
\end{lemma}

\begin{proof}
By Assumption~\ref{ass:rh-constrained}, the local convex problem satisfies
Slater's condition. Standard convex duality therefore gives strong duality
and existence of an optimal multiplier.

To see boundedness, let $\widetilde{\mathbf{x}}_j$ be a Slater point with
$g_j(\widetilde{\mathbf{x}}_j)\leq-\gamma\mathbf 1$. For any
$\mu^\star\in\mathcal M_j^\star$,
$$
\begin{aligned}
p_j^\star
=
d_j(\mu^\star)
&\leq
\mathcal L_j(\widetilde{\mathbf{x}}_j,\mu^\star;\bar x_j)\\
&\leq
\varphi_j(\widetilde{\mathbf{x}}_j;\bar x_j)
-\gamma\|\mu^\star\|_1.
\end{aligned}
$$
Hence
$$
\|\mu^\star\|_1
\leq
\frac{\varphi_j(\widetilde{\mathbf{x}}_j;\bar x_j)-p_j^\star}{\gamma}.
$$
Thus $\mathcal M_j^\star$ is bounded; closedness follows from continuity of
$d_j$, proved below, so it is compact.

We next establish Lipschitzness of the primal response. Strong convexity
and optimality of $\mathbf{x}_j(\mu)$ and $\mathbf{x}_j(\nu)$ give
$$
\mathcal L_j(\mathbf{x}_j(\nu),\mu)
\geq
\mathcal L_j(\mathbf{x}_j(\mu),\mu)
+
\frac{\sigma_j}{2}
\|\mathbf{x}_j(\nu)-\mathbf{x}_j(\mu)\|^2
$$
and
$$
\mathcal L_j(\mathbf{x}_j(\mu),\nu)
\geq
\mathcal L_j(\mathbf{x}_j(\nu),\nu)
+
\frac{\sigma_j}{2}
\|\mathbf{x}_j(\mu)-\mathbf{x}_j(\nu)\|^2.
$$
Adding these inequalities and cancelling the common objective terms yields
$$
\sigma_j
\|\mathbf{x}_j(\mu)-\mathbf{x}_j(\nu)\|^2
\leq
(\mu-\nu)^\top
\bigl(
g_j(\mathbf{x}_j(\nu))
-
g_j(\mathbf{x}_j(\mu))
\bigr).
$$
Using Cauchy--Schwarz and the $L_{g,j}$-Lipschitzness of $g_j$,
$$
\sigma_j
\|\mathbf{x}_j(\mu)-\mathbf{x}_j(\nu)\|^2
\leq
L_{g,j}\|\mu-\nu\|
\|\mathbf{x}_j(\mu)-\mathbf{x}_j(\nu)\|,
$$
which proves
$$
\|\mathbf{x}_j(\mu)-\mathbf{x}_j(\nu)\|
\leq
\frac{L_{g,j}}{\sigma_j}\|\mu-\nu\|.
$$

To identify the dual gradient, optimality of the two primal responses gives
$$
d_j(\nu)
\leq
d_j(\mu)+(\nu-\mu)^\top g_j(\mathbf{x}_j(\mu))
$$
and
$$
d_j(\mu)
\leq
d_j(\nu)+(\mu-\nu)^\top g_j(\mathbf{x}_j(\nu)).
$$
Therefore
$$
\begin{aligned}
&
\left|
d_j(\nu)-d_j(\mu)
-(\nu-\mu)^\top g_j(\mathbf{x}_j(\mu))
\right|\\
&\qquad\leq
\|\nu-\mu\|\,
\|g_j(\mathbf{x}_j(\nu))-g_j(\mathbf{x}_j(\mu))\|\\
&\qquad\leq
\frac{L_{g,j}^2}{\sigma_j}\|\nu-\mu\|^2.
\end{aligned}
$$
Hence
$\nabla d_j(\mu)=g_j(\mathbf{x}_j(\mu))$. Finally,
$$
\begin{aligned}
\|\nabla d_j(\mu)-\nabla d_j(\nu)\|
&\leq
L_{g,j}
\|\mathbf{x}_j(\mu)-\mathbf{x}_j(\nu)\|\\
&\leq
\frac{L_{g,j}^2}{\sigma_j}\|\mu-\nu\|
=
L_{d,j}\|\mu-\nu\|.
\end{aligned}
$$
\end{proof}

\begin{proof}[Proof of Lemma~\ref{thm: subproblem convergence}]
Fix $j$ and suppress the dependence on $\bar x_j$. From
Lemma~\ref{lem:app-pd-dual-regularity},
$$
\begin{aligned}
\langle
\nabla d_j(\mu)-\nabla d_j(\nu),\mu-\nu
\rangle
&=
(\mu-\nu)^\top
\bigl(
g_j(\mathbf{x}_j(\mu))
-
g_j(\mathbf{x}_j(\nu))
\bigr)\\
&\leq
-\sigma_j
\|\mathbf{x}_j(\mu)-\mathbf{x}_j(\nu)\|^2\\
&\leq
-\frac{1}{L_{d,j}}
\|\nabla d_j(\mu)-\nabla d_j(\nu)\|^2.
\end{aligned}
$$

Let
$$
T(\mu)
:=
\operatorname{Proj}_{\mathbb R_+^m}
\bigl(\mu+\eta_\mu\nabla d_j(\mu)\bigr).
$$
A point $z$ maximizes $d_j$ over $\mathbb R_+^m$ if and only if
$T(z)=z$. For $z\in\mathcal M_j^\star$, nonexpansiveness of the projection
and the preceding inequality imply
$$
\begin{aligned}
\|\mu_j^{(s+1)}-z\|^2
&\leq
\|
\mu_j^{(s)}-z
+
\eta_\mu(
\nabla d_j(\mu_j^{(s)})-\nabla d_j(z))
\|^2\\
&\leq
\|\mu_j^{(s)}-z\|^2
-
\eta_\mu
\left(
\frac{2}{L_{d,j}}-\eta_\mu
\right)
\|
\nabla d_j(\mu_j^{(s)})-\nabla d_j(z)
\|^2.
\end{aligned}
$$
Thus, whenever
$0<\eta_\mu<2/L_{d,j}$,
the sequence is Fej\'er monotone with respect to
$\mathcal M_j^\star$ and therefore bounded.

Using firm nonexpansiveness of the projection gives more precisely
$$
\begin{aligned}
&
\|\mu_j^{(s+1)}-z\|^2
+
\|
\mu_j^{(s)}-\mu_j^{(s+1)}
+
\eta_\mu(
\nabla d_j(\mu_j^{(s)})-\nabla d_j(z))
\|^2\\
&\qquad\leq
\|\mu_j^{(s)}-z\|^2
-
\eta_\mu
\left(
\frac{2}{L_{d,j}}-\eta_\mu
\right)
\|
\nabla d_j(\mu_j^{(s)})-\nabla d_j(z)
\|^2.
\end{aligned}
$$
Summing over $s$ shows both squared terms on the left are summable.
Consequently,
$$
\|\mu_j^{(s+1)}-\mu_j^{(s)}\|\to0.
$$
Since the sequence is bounded, it admits a convergent subsequence
$\mu_j^{(s_k)}\to\bar\mu_j$. Passing to the limit in
$\mu_j^{(s+1)}=T(\mu_j^{(s)})$ gives $T(\bar\mu_j)=\bar\mu_j$; hence
$\bar\mu_j\in\mathcal M_j^\star$. Fej\'er monotonicity with respect to
$\bar\mu_j$ then implies convergence of the full sequence:
$$
\mu_j^{(s)}\to\bar\mu_j.
$$
Lemma~\ref{lem:app-pd-dual-regularity} gives continuity of the primal
response, and Lemma~\ref{lem:app-pd-strong-convexity} gives uniqueness.
Therefore
$$
\mathbf{x}_j^{(s)}
=
\mathbf{x}_j(\mu_j^{(s)})
\longrightarrow
\mathbf{x}_j(\bar\mu_j)
=
\mathbf{x}_j^\star.
$$

\medskip
\noindent\textbf{Finite-iteration bounds.}
Assume now $0<\eta_\mu\leq1/L_{d,j}$. The projection optimality condition
gives, for any $z\geq0$,
$$
\left\langle
\nabla d_j(\mu_j^{(k)}),
z-\mu_j^{(k+1)}
\right\rangle
\leq
\frac{
\|z-\mu_j^{(k)}\|^2
-
\|z-\mu_j^{(k+1)}\|^2
-
\|\mu_j^{(k+1)}-\mu_j^{(k)}\|^2
}
{2\eta_\mu}.
$$
By concavity and $L_{d,j}$-smoothness,
$$
\begin{aligned}
d_j(z)-d_j(\mu_j^{(k+1)})
&\leq
\left\langle
\nabla d_j(\mu_j^{(k)}),
z-\mu_j^{(k+1)}
\right\rangle\\
&\quad+
\frac{L_{d,j}}{2}
\|\mu_j^{(k+1)}-\mu_j^{(k)}\|^2.
\end{aligned}
$$
Hence
$$
\begin{aligned}
d_j(z)-d_j(\mu_j^{(k+1)})
&\leq
\frac{
\|z-\mu_j^{(k)}\|^2
-
\|z-\mu_j^{(k+1)}\|^2
}
{2\eta_\mu}\\
&\quad-
\left(
\frac{1}{2\eta_\mu}
-
\frac{L_{d,j}}{2}
\right)
\|\mu_j^{(k+1)}-\mu_j^{(k)}\|^2.
\end{aligned}
$$
Taking $z=\mu_j^{(k)}$ shows that
$d_j(\mu_j^{(k)})$ is nondecreasing. Let
$\mu_j^\dagger\in\mathcal M_j^\star$ satisfy
$$
\|\mu_j^{(0)}-\mu_j^\dagger\|
=
R_j
:=
\operatorname{dist}
(\mu_j^{(0)},\mathcal M_j^\star).
$$
Summing the preceding inequality with $z=\mu_j^\dagger$ gives
$$
\sum_{k=0}^{s-1}
\left(
p_j^\star-d_j(\mu_j^{(k+1)})
\right)
\leq
\frac{R_j^2}{2\eta_\mu}.
$$
Since these gaps are nonincreasing,
$$
0
\leq
p_j^\star-d_j(\mu_j^{(s)})
\leq
\frac{R_j^2}{2\eta_\mu s}.
$$

Finally, strong convexity of
$\mathcal L_j(\cdot,\mu)$ and feasibility of $\mathbf{x}_j^\star$ imply
$$
\begin{aligned}
\frac{\sigma_j}{2}
\|\mathbf{x}_j(\mu)-\mathbf{x}_j^\star\|^2
&\leq
\mathcal L_j(\mathbf{x}_j^\star,\mu)
-
d_j(\mu)\\
&\leq
p_j^\star-d_j(\mu).
\end{aligned}
$$
Taking $\mu=\mu_j^{(s)}$ yields
$$
\|\mathbf{x}_j(\mu_j^{(s)})-\mathbf{x}_j^\star\|
\leq
\frac{R_j}{\sqrt{\sigma_j\eta_\mu s}}.
$$
Since $[\,\cdot\,]_+$ is nonexpansive and
$g_j(\mathbf{x}_j^\star)\leq0$,
$$
\begin{aligned}
\|[g_j(\mathbf{x}_j(\mu_j^{(s)}))]_+\|
&\leq
\|
g_j(\mathbf{x}_j(\mu_j^{(s)}))
-
g_j(\mathbf{x}_j^\star)
\|\\
&\leq
L_{g,j}
\|\mathbf{x}_j(\mu_j^{(s)})-\mathbf{x}_j^\star\|\\
&\leq
\frac{L_{g,j}R_j}
{\sqrt{\sigma_j\eta_\mu s}}.
\end{aligned}
$$
This proves the lemma.
\end{proof}


\subsection{Proof of Theorem~\ref{thm:finite-rh-error}}
\label{app:proof-finite-rh-error}

We next propagate the local optimization error through the receding-horizon
closed loop. For a current state $\xi$, let
$\mathbf{x}_j^\star(\xi)$ denote the exact optimizer of the $j$th RH
problem and define
$$
\pi_j(\xi)
:=
[\mathbf{x}_j^\star(\xi)]_1.
$$
Thus, the exact RH trajectory satisfies
$x_0^{\rm RH}=x_0$ and
$x_{j+1}^{\rm RH}=\pi_j(x_j^{\rm RH})$.

\begin{lemma}[Lipschitz exact RH policy]
\label{lem:app-rh-policy-lipschitz}
Under Assumptions~\ref{ass:rh-affine}--\ref{ass:rh-constrained},
$$
\|\mathbf{x}_j^\star(\xi)-\mathbf{x}_j^\star(\zeta)\|
\leq
\kappa_j\|\xi-\zeta\|,
\qquad
\kappa_j
:=
\frac{\|B_j^\top S_j\|}
{\lambda\Delta t\,\sigma_j}.
$$
Consequently,
$$
\|\pi_j(\xi)-\pi_j(\zeta)\|
\leq
\kappa_j\|\xi-\zeta\|.
$$
\end{lemma}

\begin{proof}
The constraint
$g_j(\mathbf{x}_j)=h(\Phi_j(x_{j+H_j}))$ does not depend explicitly on
the current state. Hence the feasible set is the same for initial states
$\xi$ and $\zeta$.

Since $\varphi_j(\cdot;\xi)$ is $\sigma_j$-strongly convex and
$\mathbf{x}_j^\star(\xi)$ minimizes it over this common convex feasible
set,
$$
\varphi_j(\mathbf{x}_j^\star(\zeta);\xi)
\geq
\varphi_j(\mathbf{x}_j^\star(\xi);\xi)
+
\frac{\sigma_j}{2}
\|\mathbf{x}_j^\star(\xi)-\mathbf{x}_j^\star(\zeta)\|^2.
$$
Exchanging $\xi$ and $\zeta$ gives
$$
\varphi_j(\mathbf{x}_j^\star(\xi);\zeta)
\geq
\varphi_j(\mathbf{x}_j^\star(\zeta);\zeta)
+
\frac{\sigma_j}{2}
\|\mathbf{x}_j^\star(\xi)-\mathbf{x}_j^\star(\zeta)\|^2.
$$
Adding the two inequalities,
$$
\begin{aligned}
\sigma_j
\|\mathbf{x}_j^\star(\xi)-\mathbf{x}_j^\star(\zeta)\|^2
&\leq
\varphi_j(\mathbf{x}_j^\star(\xi);\zeta)
-
\varphi_j(\mathbf{x}_j^\star(\xi);\xi)\\
&\quad+
\varphi_j(\mathbf{x}_j^\star(\zeta);\xi)
-
\varphi_j(\mathbf{x}_j^\star(\zeta);\zeta).
\end{aligned}
$$
The terminal-reward terms cancel, and expanding the two quadratic residual
costs gives
$$
\begin{aligned}
\sigma_j
\|\mathbf{x}_j^\star(\xi)-\mathbf{x}_j^\star(\zeta)\|^2
&\leq
\frac{1}{\lambda\Delta t}
\left\langle
B_j^\top S_j(\xi-\zeta),
\mathbf{x}_j^\star(\xi)-\mathbf{x}_j^\star(\zeta)
\right\rangle\\
&\leq
\frac{\|B_j^\top S_j\|}
{\lambda\Delta t}
\|\xi-\zeta\|
\|\mathbf{x}_j^\star(\xi)-\mathbf{x}_j^\star(\zeta)\|.
\end{aligned}
$$
Cancelling the final factor proves
$$
\|\mathbf{x}_j^\star(\xi)-\mathbf{x}_j^\star(\zeta)\|
\leq
\kappa_j\|\xi-\zeta\|.
$$
Since $\pi_j$ extracts the first block of the optimized path and this
projection has operator norm one, the same bound holds for $\pi_j$.
\end{proof}

\begin{lemma}[Propagation of local solve errors]
\label{lem:app-rh-error-propagation}
Suppose the implemented trajectory satisfies
$$
\|x_{j+1}-\pi_j(x_j)\|
\leq
\epsilon_j,
\qquad j=0,\ldots,N-1.
$$
Then, for every $p=1,\ldots,N$,
$$
\|x_p-x_p^{\rm RH}\|
\leq
\sum_{j=0}^{p-1}
\left(
\prod_{\ell=j+1}^{p-1}\kappa_\ell
\right)
\epsilon_j,
$$
where an empty product equals one. Moreover, if
$h(\pi_{N-1}(x_{N-1}))\leq0$, then
$$
\|[h(x_N)]_+\|
\leq
L_h\epsilon_{N-1}.
$$
\end{lemma}

\begin{proof}
Since $x_0=x_0^{\rm RH}$, the claim is immediate for $p=1$:
$$
\|x_1-x_1^{\rm RH}\|
=
\|x_1-\pi_0(x_0)\|
\leq
\epsilon_0.
$$
Assume the bound holds at time $p$. Then
$$
\begin{aligned}
\|x_{p+1}-x_{p+1}^{\rm RH}\|
&\leq
\|x_{p+1}-\pi_p(x_p)\|
+
\|\pi_p(x_p)-\pi_p(x_p^{\rm RH})\|\\
&\leq
\epsilon_p
+
\kappa_p\|x_p-x_p^{\rm RH}\|\\
&\leq
\epsilon_p
+
\sum_{j=0}^{p-1}
\left(
\prod_{\ell=j+1}^{p}\kappa_\ell
\right)\epsilon_j\\
&=
\sum_{j=0}^{p}
\left(
\prod_{\ell=j+1}^{p}\kappa_\ell
\right)\epsilon_j,
\end{aligned}
$$
which proves the result by induction.

For terminal feasibility, since
$h(\pi_{N-1}(x_{N-1}))\leq0$ and $h$ is $L_h$-Lipschitz,
$$
\begin{aligned}
\|[h(x_N)]_+\|
&=
\|[h(x_N)]_+-[h(\pi_{N-1}(x_{N-1}))]_+\|\\
&\leq
\|h(x_N)-h(\pi_{N-1}(x_{N-1}))\|\\
&\leq
L_h
\|x_N-\pi_{N-1}(x_{N-1})\|\\
&\leq
L_h\epsilon_{N-1}.
\end{aligned}
$$
\end{proof}

\begin{proof}[Proof of Theorem~\ref{thm:finite-rh-error}]
At physical window $j$, Lemma~\ref{thm: subproblem convergence} with
$s=K$ gives
$$
\|\mathbf{x}_j^{(K)}-\mathbf{x}_j^\star(x_j)\|
\leq
\frac{R_j}
{\sqrt{\sigma_j\eta_\mu K}},
$$
where
$$
R_j
=
\operatorname{dist}
\bigl(
\mu_j^{(0)},\mathcal M_j^\star(x_j)
\bigr).
$$
Since the implemented next state is the first block of
$\mathbf{x}_j^{(K)}$, while
$\pi_j(x_j)$ is the first block of
$\mathbf{x}_j^\star(x_j)$,
$$
\|x_{j+1}-\pi_j(x_j)\|
\leq
\frac{R_j}
{\sqrt{\sigma_j\eta_\mu K}}
=:\epsilon_j.
$$

Define
$$
\bar R:=\max_jR_j,
\qquad
\underline\sigma:=\min_j\sigma_j,
\qquad
\bar\kappa:=\max_j\kappa_j.
$$
Lemma~\ref{lem:app-rh-error-propagation} therefore gives
$$
\begin{aligned}
\|x_p-x_p^{\rm RH}\|
&\leq
\sum_{j=0}^{p-1}
\left(
\prod_{\ell=j+1}^{p-1}\kappa_\ell
\right)
\frac{R_j}{\sqrt{\sigma_j\eta_\mu K}}\\
&\leq
\frac{\bar R}
{\sqrt{\underline\sigma\,\eta_\mu K}}
\sum_{j=0}^{p-1}
\bar\kappa^{\,p-1-j}\\
&=
\frac{\bar R\,\Gamma_p(\bar\kappa)}
{\sqrt{\underline\sigma\,\eta_\mu K}},
\end{aligned}
$$
where
$$
\Gamma_p(\bar\kappa)
:=
\sum_{q=0}^{p-1}\bar\kappa^q.
$$

Finally, at the last physical window $j=N-1$, we have $H_{N-1}=1$.
Hence $\pi_{N-1}(x_{N-1})$ is the terminal state of the exact feasible
local solution, so
$$
h(\pi_{N-1}(x_{N-1}))\leq0.
$$
Applying Lemma~\ref{lem:app-rh-error-propagation},
$$
\begin{aligned}
\|[h(x_N)]_+\|
&\leq
L_h\epsilon_{N-1}\\
&\leq
\frac{L_h\bar R}
{\sqrt{\underline\sigma\,\eta_\mu K}}.
\end{aligned}
$$
This proves the theorem.
\end{proof}

\clearpage
\section{Additional Experimental Details and Results}\label{app: experiments}

\subsection{Implementation Details}

\newcommand{\ours}{ConicFlow}

\subsubsection{Low-dimensional experiment implementation details}
\label{app:impl-lowdim}

\paragraph{Model and sampler.}
The pretrained model is a flow map on the two-moons density, parameterised as
$X_\theta(s,t,x) = x + (t-s)\,\sigma_d\,\phi_\theta\!\big(s,\,t,\,x/\sigma_d\big)$
with $\phi_\theta$ an MLP of three hidden layers of $256$ GELU units, input
$(s,t,x/\sigma_d)\in\mathbb{R}^4$ and output in $\mathbb{R}^2$. We use the EMA
weights at $15$k training steps. The data are \texttt{sklearn}'s
\texttt{make\_moons} with noise $0.05$, standardised per coordinate and then
scaled by $0.75$; $\sigma_d = 0.750$ is the standard deviation of that training
set, and the base law is $p_0 = \mathcal{N}(0,\sigma_d^2 I)$. Unlike the image
experiment, time runs \emph{forward}, $t_j = j/N$ with $N=16$ uniform Euler
steps and $\Delta t = 1/16$; Every run uses the same batch of $512$
particles drawn from the same seed, shared across all methods and all $\lambda$, so
each frontier is computed on identical initial conditions.

\paragraph{Rewards.}
Two terminal rewards are used, both bounded in $(0,1]$ so that reward gains are
directly comparable across methods.
The \emph{upper-moon} reward is an RBF envelope over one mode,
$r(x) = \exp\!\big(-\min_j \lVert x - x^{(j)}\rVert^2 / 2\sigma^2\big)$ with
$\sigma = 0.35$ and $400$ reference points $x^{(j)}$ drawn from the upper moon;
it is $\approx 1$ on that mode and decays off it. The uncontrolled sampler
already attains $\bar r = 0.589$, so the attainable headroom is small.
The \emph{off-manifold} reward is a single Gaussian bump
$r(x) = \exp\!\big(-\lVert x - p\rVert^2 / 2\sigma^2\big)$ with $p = (-0.1,\,1.75)$
and $\sigma = 0.5$. Its optimum lies off the data manifold: the upper moon's
apex $(-0.36,\,1.41)$ scores only $0.69$, and the uncontrolled sampler attains
$\bar r = 0.062$. Reaching high reward therefore requires leaving the support of
the learned density, which is what separates the methods.

\paragraph{Control methods.}
All three methods share the sampler, the initial batch and the reward, and no
constraints are active, so \ours{} again reduces to $K=1$ in
Algorithm~\ref{alg:rh-path-space}. Two departures from the image configuration
are specific to this setting. First, the window subproblem is \emph{not}
warm-started across windows: each window is initialised from a fresh
uncontrolled rollout of its own $H$ steps, which isolates the effect of the
horizon from that of solver warm-starting. Second, we report two choices of the
terminal map $\Phi_j$ in~\eqref{eq:terminal-rollout-rh}:
\emph{endpoint}, $\Phi_j = X_{j+H,N}$, which predicts $t=1$ from the window's
terminal state through one flow-map call (the image-experiment choice), and
\emph{horizon}, $\Phi_j = \mathrm{id}$, which evaluates the reward directly at
the window's terminal state $x_{j+H}$ and makes no network call beyond the
rollout. 

HardFlow is run in its unconstrained form, so its terminal subproblem reduces to
proximal reward steering,
$\hat x_N^\star = \arg\min_x -r(x) + \tfrac{\alpha_t^2}{2\Delta t\,\lambda}\lVert x - \bar x_N\rVert^2$,
solved by the augmented-Lagrangian routine with $\rho_0 = 1$ and $\rho \times 2$
per outer round, warm-started across sampling steps. Its released code
parameterises the proximal weight as $w = \lambda_{\mathrm{oc}}\alpha_t^2/(2\Delta t)$,
i.e.\ with the opposite sense: larger $\lambda_{\mathrm{oc}}$ holds the sample
\emph{closer} to the pretrained path. We therefore set
$\lambda_{\mathrm{oc}} = 1/\lambda$ so that a single $\lambda$ axis means the same
thing for all three methods --- larger $\lambda$ is a weaker control penalty and
more reward --- which is what makes the frontiers superimposable. FMRG is run in
both variants, J (differentiating through $X_{j,N}$) and E (stop-gradient at the
map), with the step $z \leftarrow z + (\Delta t/n_{\mathrm{inner}})\lambda\,\nabla r$.

\paragraph{$\lambda$ grids and per-$\lambda$ selection.}
Because the state space is two-dimensional, the inner hyperparameters can be
swept exhaustively rather than range-found. For every (method, variant,
$\lambda$) we run the full inner grid of Table~\ref{tab:impl-lowdim} and report
the configuration minimising the control objective itself,
\begin{equation}
  J \;=\; \frac{1}{2\lambda}\sum_{j}\lVert \Delta x_j\rVert^2 \;-\; \bar r(x_N),
  \label{eq:lowdim-selection}
\end{equation}
where $\Delta x_j = x_{j+1} - \bar x_{j+1}$ is the applied deviation from the
uncontrolled Euler step. Each plotted marker is that single best run, so the
curves are attainment frontiers under a fixed selection rule rather than
non-dominated envelopes. The horizontal axis of both figures is the mean
integrated control effort $\sum_j\lVert \Delta x_j\rVert^2$ averaged over the
batch, which equals $\Delta t\,E$ for the energy $E$ of~\eqref{eq:auc}, i.e.\
$E/16$ at $N=16$.

For the upper-moon reward all three methods share
$\lambda \in \{0.1,\,0.25,\,0.5,\,0.75,\,1,\,2,\,3\}$; beyond $\lambda \approx 3$
every method has saturated, since the reward is already near $1$ on the mode the
sampler visits. For the off-manifold reward HardFlow and FMRG saturate by
$\lambda \approx 5$ and are swept over the $13$ values
$\{0.1,\ldots,100\}$, whereas \ours{} is still improving at $\lambda = 100$ and
is extended geometrically to $\lambda = 10^5$ ($22$ values in total) until its
reward plateaus.

\begin{table}[h]
\caption{Hyperparameters for the low-dimensional experiments.}
  \centering
  \small
  \begin{tabular}{lll}
    \toprule
    \multicolumn{3}{l}{\emph{Shared across methods}} \\
    \midrule
    \multicolumn{2}{l}{Steps $N$ }      & $16$ \\
    \multicolumn{2}{l}{Batch / seed}                    & $512$ particles / $42$, shared across all runs \\
    \multicolumn{2}{l}{Base law $p_0$}                  & $\mathcal{N}(0,\sigma_d^2 I)$, $\sigma_d = 0.750$ \\
    \multicolumn{2}{l}{Backbone}                        & MLP, $3\times256$ GELU; EMA weights at $15$k steps \\
    \midrule
    \multicolumn{3}{l}{\emph{Per method}} \\
    \midrule
    \ours{}  & horizon $H$              & $\{1,2,4,8,16\}$ \\
             & terminal map $\Phi_j$    & \emph{endpoint} $X_{j+H,N}$ or \emph{horizon} $\mathrm{id}$ \\
             & inner optimizer          & Adam \\
             & inner iterations         & $\{25,50\}$, extended to $\{100,200\}$ at large $\lambda$ \\
             & learning rate            & $\{0.01,0.05\}$, extended with $\{0.001,0.005,0.2\}$ at large $\lambda$ \\
             & window initialisation    & fresh uncontrolled rollout (no cross-window warm start) \\
    HardFlow & inner optimizer          & L-BFGS (strong Wolfe) or Adam \\
             & outer $\times$ inner     & $\{3,5,10\} \times \{10,20\}$ \\
             & learning rate            & $\{0.1,0.5\}$ (L-BFGS), $\{0.01,0.05\}$ (Adam) \\
             & penalty $\rho_0$ / growth & $1$ / $\times 2$ per outer round \\
             & proximal weight          & $\alpha_t^2/(2\Delta t\,\lambda)$, i.e.\ $\lambda_{\mathrm{oc}} = 1/\lambda$ \\
             & warm start               & across sampling steps \\
    FMRG     & variant                  & J (through $X_{j,N}$) and E (stop-gradient) \\
             & guidance steps           & $\{1,3,5,10\}$ per sampling step \\
    \bottomrule
  \end{tabular}
  
  \label{tab:impl-lowdim}
\end{table}

\subsubsection{ Unconstrained image experiment implementation details}
\label{app:impl-unconstrained}

\paragraph{Model and sampler.}
We use FLUX.1-dev with the FlowMap LoRA adapter, so one forward pass evaluates the flow map $b_t(z,\,t\!\to\!t')$: sampling advances with $t' = t - \Delta t$, while $t'=0$ estimates the terminal latent in a single call. The $N$ steps are uniform, $t_j = 1 - j/N$, so $\Delta t = 1/N$ for the sampler. Each run draws one latent from a seed, and the uncontrolled sample from that same seed is the reference throughout.

\paragraph{Reward.}

CLIP~\citep{radford2021learning} is the cosine similarity between normalised CLIP-ViT-H-14 image and text features. PickScore~\citep{kirstain2023pickapic} is $\exp(\tau)$ times the same cosine under its own finetuned encoder, $\tau$ being the checkpoint's learned logit scale. ImageReward~\citep{xu2023imagereward} is the score its checkpoint already standardises to zero mean and unit variance. All three are evaluated on the decoded image, so every reward gradient is taken through the VAE decoder as well as the scorer. The three scales are unrelated, so a reward gain is only ever compared within a reward. No constraints are active in this experiment.

\paragraph{Evaluation suite.}
Table~\ref{tab:objective-improvement} and the frontier figures are computed on twelve problems: three reward models $\times$ four (prompt, seed) instances. The instances are \emph{``a portrait of a man''} $\to$ \emph{``a smiling face''} at seeds $42$ and $7$, \emph{``a photo of a sleeping cat''} $\to$ \emph{``a photo of a cat with wide open eyes''} at seed $42$, and \emph{``a field of green grass''} $\to$ \emph{``a field of red poppies''} at seed $42$. Within a problem all three methods share the initial latent and the $\lambda$ grid, so each problem contributes one matched triple of frontiers, and the $\pm$ in Table~\ref{tab:objective-improvement} is the standard error over the four instances of that reward. Figure~\ref{fig:pareto-edit-pickscore} (left) is the PickScore man\,$\to$\,smiling instance at seed $42$; Figure~\ref{fig:pareto-appendix} shows two of the others.

The image rows in Figure~\ref{fig:pareto-edit-pickscore} (right) are selected at matched \emph{achieved reward} rather than at matched $\lambda$, and use three editing pairs: \emph{``a portrait of a man''} $\to$ \emph{``a smiling face''} (seed $7$), \emph{``a portrait of a young man''} $\to$ \emph{``a portrait of an old man''} (seed $317$), and \emph{``a portrait of a woman with straight hair''} $\to$ \emph{``a portrait of a woman with curly hair''} (seed $1289$). The latter two come from a separate editing sweep and do not enter Table~\ref{tab:objective-improvement}.

\paragraph{Control methods.}
All three methods share the sampler, the initial latent and the reward. We run
ConicFlow in its~\eqref{eq:terminal-rollout-rh} form with $H=2$; since this experiment is unconstrained,
Algorithm~\ref{alg:rh-path-space} reduces to $K=1$ and its primal response is approximated by $n_{inner}$ Adam steps at learning rate $10^{-2}$, warm-started from the previous window's solution rather than from the pretrained path. We evaluate $\Phi_j = X_{j+H,N}$ through the recursion, i.e.\ $N-j-H$ pretrained steps. Within the window the drift is evaluated without gradient for $\ell > 0$, dropping the $J_b^\top$ term from the inner gradient.

HardFlow, ported here to FLUX latents, instead optimizes the predicted terminal latent $\hat x = X_{j,N}(x_j)$ obtained from a single call and never differentiates the transformer. FMRG-J takes $n_{\mathrm{inner}}$ Jacobian-guidance steps per sampling step, $z \leftarrow z + (\Delta t/n_{\mathrm{inner}})\,\lambda\,
\nabla_z r\big(X_{j,N}(z)\big)$.

\paragraph{Choosing $\lambda$.}
The useful range of the control-cost weight $\lambda$ moves with the reward's
scale and with the problem, so we range-find it per (reward, prompt) pair at a cheap $N=4$ proxy of the $N=8$ sweep rather than fixing one grid. We step $\times 10$ down from a $\|\nabla_z r\|$-scaled starting value until $\Delta r$ vanishes into the noise and $\times 10$ up until it saturates, then place six $\lambda$ per method across the logistic rise fitted to $\Delta r(\lambda)$, densified with log-midpoints where a finer frontier was wanted. The same grid is reused across seeds. The three per-method grids are merged into one, so every arm runs at identical $\lambda$, at which all three apply the same first-order step $\lambda\,\Delta t\,\nabla r$ per sampling step. The resulting grids differ by orders of magnitude across the suite; the PickScore man\,$\to$\,smiling cell behind Figure~\ref{fig:pareto-edit-pickscore} (left) runs at $\lambda \in \{0.95, 1.5, 2.5, 4.9, 9.8, 15, 23, 38, 62, 130, 280\}$.

\paragraph{Mean reward gain over a shared energy budget.}
Comparing methods at a common $\lambda$ is fragile, since $\lambda$ enters each
method differently and the same value buys different amounts of control. We
therefore summarise each method by its whole frontier instead. 
Define $E = \sum_{k=0}^{N-1} \Delta t\,\lVert u_k \rVert^2$ to denote a run's control energy.
Fix a problem
(reward, prompt, seed). For method $m$ we first discard every run that another
run of the \emph{same} method beats on both axes (i.e., no less energy and no more
reward), leaving a set of non-dominated runs $\{(E_i, r_i)\}$ along which $r$
increases with $E$; write $\hat r_m$ for its piecewise-linear interpolant in
$\log_{10} E$. Because the methods span different energies, we evaluate all of
them on the interval attained by all three,
$E_1 = \max_m \min_i E_i$ and $E_2 = \min_m \max_i E_i$, and report
\begin{equation}
  G_m \;=\; \frac{1}{\log_{10}(E_2/E_1)}
  \int_{\log_{10} E_1}^{\log_{10} E_2}
  \big(\hat r_m(E) - r_{\mathrm{unc}}\big)\; \mathrm{d}\log_{10} E ,
  \label{eq:auc}
\end{equation}
the mean height of the frontier over that budget range. $G_m$ carries the units
of the reward, and $G_m = 0$ means the method does no better than the
uncontrolled sample anywhere in the range. Both endpoints lie inside every
method's frontier by construction, so no extrapolation is involved, and
\eqref{eq:auc} is invariant to any uniform rescaling of the energy axis.
Integrating in $\log_{10} E$ rather than in $E$ weights each decade of budget
equally; on a linear axis the top decade would carry almost all the weight.
ConicFlow is taken at $24$ inner iterations, as plotted in
Fig.~\ref{fig:pareto-edit-pickscore}.

Two properties of \eqref{eq:auc} are worth stating. It is conservative for the
method that reaches furthest: the interval stops at $E_2$, the lowest of the
three maximum energies, so any reward attained beyond it --- typically
ConicFlow's --- is discarded. And it scores the attainment frontier rather than
the sweep itself, so a method whose reward is non-monotone in $\lambda$ is
represented by its upper envelope.

\paragraph{Hardware and cost.}
One run occupies a single NVIDIA RTX A5000 (24\,GB), one process per card,
peaking at $16.3$--$17.7$\,GB. Median wall clock over the 11 values of
$\lambda$: $317$\,s and $626$\,s for \ours{} at $12$ and $24$ inner iterations,
$46$\,s for HardFlow, $43$\,s for FMRG-J, against $3$--$4$\,s for the
uncontrolled sample.

\begin{table}[h]
\caption{Hyperparameters for the unconstrained image experiment (Fig~\ref{fig:pareto-edit-pickscore})}
  \centering
  \small
  \begin{tabular}{lll}
    \toprule
    \multicolumn{3}{l}{\emph{Shared across methods}} \\
    \midrule
    \multicolumn{2}{l}{Steps $N$ / horizon $H$}   & $8$ / $2$ \\
    \multicolumn{2}{l}{Resolution / latent}       & $256\times256$ / $16\times32\times32$ \\
    \multicolumn{2}{l}{Guidance scale}            & $3.5$ \\
    \multicolumn{2}{l}{Precision}                 & bf16 sampler, int8 transformer, fp32 reward \\
    \multicolumn{2}{l}{Reward models}             & CLIP, PickScore (both CLIP-ViT-H-14), ImageReward \\
    \multicolumn{2}{l}{Seeds / batch}             & $42$ and $7$ / $1$ \\
    \multicolumn{2}{l}{$\lambda$ grid}            & range-found per (reward, prompt) \\
    \midrule
    \multicolumn{3}{l}{\emph{Per method}} \\
    \midrule
    \ours{}  & inner iterations  & $12$ and $24$ (reported as two configurations) \\
             & Inner optimizer    & Adam \\
             & Learning rate     & $10^{-2}$ \\
             & Terminal estimate & rollout of the remaining $N-j-H$ steps \\
             & drift term        & detached for $\ell>0$ \\
    HardFlow & inner iterations & $40$ \\
             & Learning rate     & $2.5$ (replaced by $0.5/\mathrm{reg}$ where $\mathrm{lr}\cdot\mathrm{reg}\geq1$)\\
             & terminal estimate & single flow-map call to $t=0$ \\
             & effort penalty    & $\mathrm{reg} = N/(2\lambda)$, active from step $0$ \\
    FMRG-J   & guidance steps    & $3$ per sampling step \\
    \bottomrule
  \end{tabular}
  \label{tab:impl-unconstrained}
\end{table}

\subsubsection{Constrained image experiments implementation details}
\label{app:impl-constrained}

The sampler, resolution, precision, guidance scale, PickScore scorer,
$N=8$, and $H=2$ are those of
Appendix~\ref{app:impl-unconstrained}; only what differs is described here.

\paragraph{Problem.}
There is no reward: $r \equiv 0$, so each method's objective is its control
energy alone and the two methods are compared on what it costs them to reach
the same feasible set. An instance is a base prompt $c$, one to three
\emph{attribute prompts} $c_1,\dots,c_C$, and a seed. Constraint $i$ is
\begin{equation}
  h_i(z) \;=\; \tau_i \;-\; s\big(\mathrm{dec}(z),\,c_i\big) \;\le\; 0,
  \label{eq:pickscore-constraint}
\end{equation}
where $s(\cdot, c_i)$ is the PickScore of the decoded image against
$c_i$, exactly the unnormalised score of
Appendix~\ref{app:impl-unconstrained} with $c_i$ in place of the edit prompt.
Every constraint is evaluated at the predicted terminal latent, and a run is \emph{feasible}
when $\max_i h_i \le 0$ on the final decoded image. Multiple constraints are
independent PickScore heads that share nothing but the image. The generation prompt stays $c$ throughout.

\paragraph{Thresholds.}
Absolute PickScore values are not comparable across prompts or seeds, so
thresholds are set relative to each instance's own uncontrolled sample. For
each (prompt, seed) pair we first draw the uncontrolled sample from that seed and
score it with every constraint to obtain baselines $s_i^0$
($19.2$--$22.5$ across the instances here). A \emph{tightness multiplier} $m$
then sets every threshold of the instance at once,
\begin{equation}
  \tau_i(m) \;=\; s_i^0 + (m-1)\,|s_i^0| \;=\; m\,s_i^0
  \quad (s_i^0>0),
  \label{eq:threshold}
\end{equation}
so $m=1$ places the boundary exactly on the uncontrolled image and
$m=1.5$ demands a $50\%$ higher score on every attribute simultaneously. (The
absolute-value form is what the code implements; it keeps the same grid
meaningful for scorers with negative baselines, which does not arise for
PickScore.) One (prompt, seed, $m$) triple is one comparison: both methods
receive the same $\lambda$, the same thresholds and the same initial latent,
and the two uncontrolled images are checked to be byte-identical before the
figure is drawn.

\paragraph{\ours{}.}
The window problem is~\eqref{eq:terminal-rollout-rh} with the reward term replaced by
the Lagrangian term $\sum_i \mu_i\, h_i\big(\hat e(x_H)\big)$, $\mu_i\ge0$.
After each window solve one projected dual-ascent step is taken on the
constraint value the optimiser saw at its horizon endpoint,
\begin{equation}
  \mu_i \;\leftarrow\; \max\!\big(0,\; \mu_i + \eta\, h_i(\hat e(x_H))\big),
  \label{eq:dual}
\end{equation}
and the multipliers persist across sampling steps. All multipliers start at
$\mu_{\mathrm{init}} = 0.5$ with $\eta = 0.5$; both were chosen on the
pilot instance (``a portrait of a man'' $\to$ ``a smiling face'') and then
frozen for every run reported. The remaining difference from the
unconstrained runs is the terminal estimate \emph{inside} the inner
objective: rather than rolling the remaining $N-j-H$ steps out, $\hat e$ is
a single flow-map call from $x_H$ to $t=0$, which makes the cost of one
sampling step linear rather than quadratic in $N$. Only the surrogate the
optimiser differentiates changes; the applied trajectory, the final image
and everything scored from it are computed as before.

\paragraph{HardFlow.}
Constraints enter HardFlow's inner problem as an augmented Lagrangian on the
predicted terminal latent $\hat x$, as in the released code: the ascent
objective is
$-\sum_i \mu_i h_i(\hat x) - \tfrac{\rho}{2}\sum_i \max(0,h_i(\hat x))^2
 - \mathrm{reg}\,\alpha^2\|\hat x - \bar x\|^2$,
climbed with $40$ explicit gradient steps; every $8$ inner steps the
multipliers take the projected step~\eqref{eq:dual} with $\eta=0.2$, and they
are reset to $0$ at the start of every sampling step (the released code's
semantics, retained). Two settings depart from that code and both were forced
by stability. First, the released code doubles $\rho$ whenever a violation
exceeds $0.02$ at a dual update; over $40$ inner steps that lets $\rho$ grow
from $5$ to $160$ within a single sampling step, and at the tighter
thresholds it produced control energies orders of magnitude above
\ours{}'s on the same instance. We hold $\rho = 5$ fixed. Second, the
inner ascent is an explicit fixed-step scheme whose error contracts by
$1 - 2\,\mathrm{lr}\cdot\mathrm{reg}\cdot\alpha^2$ per step, with
$\alpha\to1$ at the final sampling step; with the authors' $\mathrm{lr}=2.5$
and $\mathrm{reg}=N/(2\lambda)=0.2$ that factor reaches $0$ and the pilot runs
diverged, so the inner learning rate is $0.5$
($\mathrm{lr}\cdot\mathrm{reg}=0.1$). The constraint margin and activation
fraction are $0$, as in Appendix~\ref{app:impl-unconstrained}.

\paragraph{Choosing $\lambda$.}
$\lambda = 20$ for both methods throughout. With no reward, $\lambda$ does not trade reward against effort; it only weights energy against the constraint term and therefore selects \emph{which} feasible point each method settles on, not whether it reaches one.

\paragraph{Instances and sweep.}
Table~\ref{tab:constrained-instances} lists the instances. Nine prompt pairs
carry one or two or three constraints and are swept over a five-point grid of
tightness multipliers, from $m = 1.1$ to $m = 1.5$ in steps of $0.1$. Every instance is run at both seeds,
giving $86$ comparisons, i.e.\ $172$ controlled runs plus one uncontrolled
probe per (prompt, seed). The attribute prompts were written so that the base
prompt does not already imply the attribute (e.g., ``a young man with straight
hair'' $\to$ ``curly hair'').

\begin{table}[h]
 \caption{Hyperparameters for the constrained image experiments. Everything
           not listed (sampler, resolution, precision, guidance scale) is as in
           Table~\ref{tab:impl-unconstrained}.}
  \centering
  \small
  \begin{tabular}{lll}
    \toprule
    \multicolumn{3}{l}{\emph{Shared across methods}} \\
    \midrule
    \multicolumn{2}{l}{Reward}                     & none ($r\equiv0$; objective is control energy) \\
    \multicolumn{2}{l}{Constraint scorer}          & PickScore, one head per attribute prompt \\
    \multicolumn{2}{l}{Constraint form}            & $h_i = \tau_i - s_i \le 0$, evaluated at the endpoint \\
    \multicolumn{2}{l}{Thresholds}                 & $\tau_i = m\,s_i^0$, Eq.~\eqref{eq:threshold} \\
    \multicolumn{2}{l}{Steps $N$ / horizon $H$}    & $8$ / $2$ \\
    \multicolumn{2}{l}{$\lambda$}                  & $20$ (fixed) \\
    \multicolumn{2}{l}{Seeds}                      & $310$, $698$ \\
    \midrule
    \multicolumn{3}{l}{\emph{Per method}} \\
    \midrule
    \ours{}  & inner iterations        & $24$ \\
             & inner optimiser         & Adam, lr $10^{-2}$, warm-started \\
             & terminal estimate       & one flow-map call $x_H \to t=0$ (``jump'') \\
             & drift term              & detached for $\ell>0$ \\
             & multipliers             & $\mu_{\mathrm{init}} = 0.5$, persist across steps \\
             & dual step $\eta$        & $0.5$, once per sampling step, Eq.~\eqref{eq:dual} \\
    HardFlow & inner iterations        & $40$ \\
             & inner lr                & $0.5$ ($\mathrm{lr}\cdot\mathrm{reg} = 0.1$) \\
             & terminal estimate       & single flow-map call to $t=0$ \\
             & effort penalty          & $\mathrm{reg} = N/(2\lambda) = 0.2$, active from step $0$ \\
             & multipliers             & $\mu_{\mathrm{init}} = 0$, reset every sampling step \\
             & dual step $\eta$        & $0.2$, every $8$ inner steps \\
             & penalty weight $\rho$   & $5$, held fixed (released code doubles it) \\
             & margin / activation     & $0$ / $0$ \\
    \bottomrule
  \end{tabular}
 
  \label{tab:impl-constrained}
\end{table}

\begin{table}[h]
\caption{Constrained instances. Each attribute prompt is one PickScore
           constraint~\eqref{eq:pickscore-constraint}; semicolons separate the
           constraints of an instance.}
  \centering
  \small
  \setlength{\tabcolsep}{4pt}
  \begin{tabular}{p{4.6cm}p{6.8cm}}
    \toprule
    base prompt $c$ & attribute prompts $c_1,\dots,c_C$ \\
    \midrule
    a portrait of a man                        & a smiling face \\
    a photo of a golden retriever              & a dog wearing sunglasses \\
    a city street at noon                      & a city street at night \\
    a portrait of a woman                      & a woman with bright red hair \\
    a portrait of a woman                      & sad frowning woman \\
    portrait of a young man with straight hair & man with curly hair; photo of very old man \\
    a portrait of a woman                      & a woman wearing glasses; a woman with short blonde hair \\
    a house by a lake                          & a wooden log cabin; a house in heavy fog \\
    a green garden on an overcast day          & a garden with red flowers; a garden in bright sunshine; a garden with a stone path \\
    \bottomrule
  \end{tabular}
  \label{tab:constrained-instances}
\end{table}

\subsection{Additional Results}

\paragraph{Unconstrained.} We present additional visualizations of the unconstrained image editing comparisons in Figure~\ref{fig:unconstrained-grid}, similar to Figure~\ref{fig:pareto-edit-pickscore} in the main paper. We also present two additional pareto plots of instances in Figure~\ref{fig:pareto-appendix} similar to the one in Figure~\ref{fig:pareto-edit-pickscore} but with different reward/prompts which define a different problem.

\begin{table}[h]
\caption{Per-problem breakdown of Table~\ref{tab:objective-improvement}.}
  \centering
  \footnotesize
  \setlength{\tabcolsep}{2.5pt}
  \begin{tabular}{@{}lrrccc@{}}
    \toprule
    Instance (prompt $\to$ edit, seed) & $r_{\mathrm{unc}}$ & Budget $[E_1,E_2]$ & ConicFlow & HardFlow & FMRG \\
    \midrule
    \multicolumn{6}{l}{\textbf{CLIP}} \\
    \quad man $\to$ smiling, 42 & $0.149$ & $[8.8,\,217]$ & $\mathbf{0.151}$ & $0.135$ & $0.083$ \\
    \quad man $\to$ smiling, 7 & $0.210$ & $[8.2,\,136]$ & $\mathbf{0.094}$ & $0.070$ & $0.046$ \\
    \quad sleeping cat $\to$ open eyes, 42 & $0.247$ & $[8.8,\,127]$ & $\mathbf{0.106}$ & $0.101$ & $0.065$ \\
    \quad green grass $\to$ red poppies, 42 & $0.127$ & $[7.6,\,122]$ & $\mathbf{0.171}$ & $0.166$ & $0.066$ \\
    \cmidrule(l){1-6}
    \quad \textit{mean} $\pm$ s.e. & & & $\mathbf{0.130} \pm 0.018$ & $0.118 \pm 0.021$ & $0.065 \pm 0.008$ \\
    \quad \textit{ConicFlow wins} & & & -- & $4/4$ & $4/4$ \\
    \midrule
    \multicolumn{6}{l}{\textbf{ImageReward}} \\
    \quad man $\to$ smiling, 42 & $-1.611$ & $[9.3,\,25.0]$ & $\mathbf{2.648}$ & $2.203$ & $2.546$ \\
    \quad man $\to$ smiling, 7 & $0.791$ & $[8.0,\,14.1]$ & $0.920$ & $0.738$ & $\mathbf{0.938}$ \\
    \quad sleeping cat $\to$ open eyes, 42 & $-1.285$ & $[9.9,\,254]$ & $\mathbf{1.716}$ & $0.610$ & $0.147$ \\
    \quad green grass $\to$ red poppies, 42 & $-2.282$ & $[7.4,\,1974]$ & $\mathbf{2.908}$ & $1.565$ & $1.483$ \\
    \cmidrule(l){1-6}
    \quad \textit{mean} $\pm$ s.e. & & & $\mathbf{2.048} \pm 0.455$ & $1.279 \pm 0.374$ & $1.278 \pm 0.504$ \\
    \quad \textit{ConicFlow wins} & & & -- & $4/4$ & $3/4$ \\
    \midrule
    \multicolumn{6}{l}{\textbf{PickScore}} \\
    \quad man $\to$ smiling, 42 & $19.426$ & $[11.0,\,1723]$ & $\mathbf{6.011}$ & $2.264$ & $2.532$ \\
    \quad man $\to$ smiling, 7 & $20.178$ & $[9.9,\,1746]$ & $\mathbf{4.911}$ & $1.830$ & $2.561$ \\
    \quad sleeping cat $\to$ open eyes, 42 & $20.304$ & $[9.1,\,2934]$ & $\mathbf{5.475}$ & $0.377$ & $2.822$ \\
    \quad green grass $\to$ red poppies, 42 & $18.646$ & $[7.4,\,2180]$ & $\mathbf{8.236}$ & $2.472$ & $2.766$ \\
    \cmidrule(l){1-6}
    \quad \textit{mean} $\pm$ s.e. & & & $\mathbf{6.158} \pm 0.728$ & $1.736 \pm 0.472$ & $2.670 \pm 0.073$ \\
    \quad \textit{ConicFlow wins} & & & -- & $4/4$ & $4/4$ \\
    \midrule
    \textit{ConicFlow wins, all 12 problems} & & & -- & $12/12$ & $11/12$ \\
    \bottomrule
  \end{tabular}
  \label{tab:objective-improvement-full}
\end{table}

\begin{figure*}[h]
  \centering
  \begin{subfigure}[h]{0.49\textwidth}
    \centering
    \includegraphics[width=\linewidth]{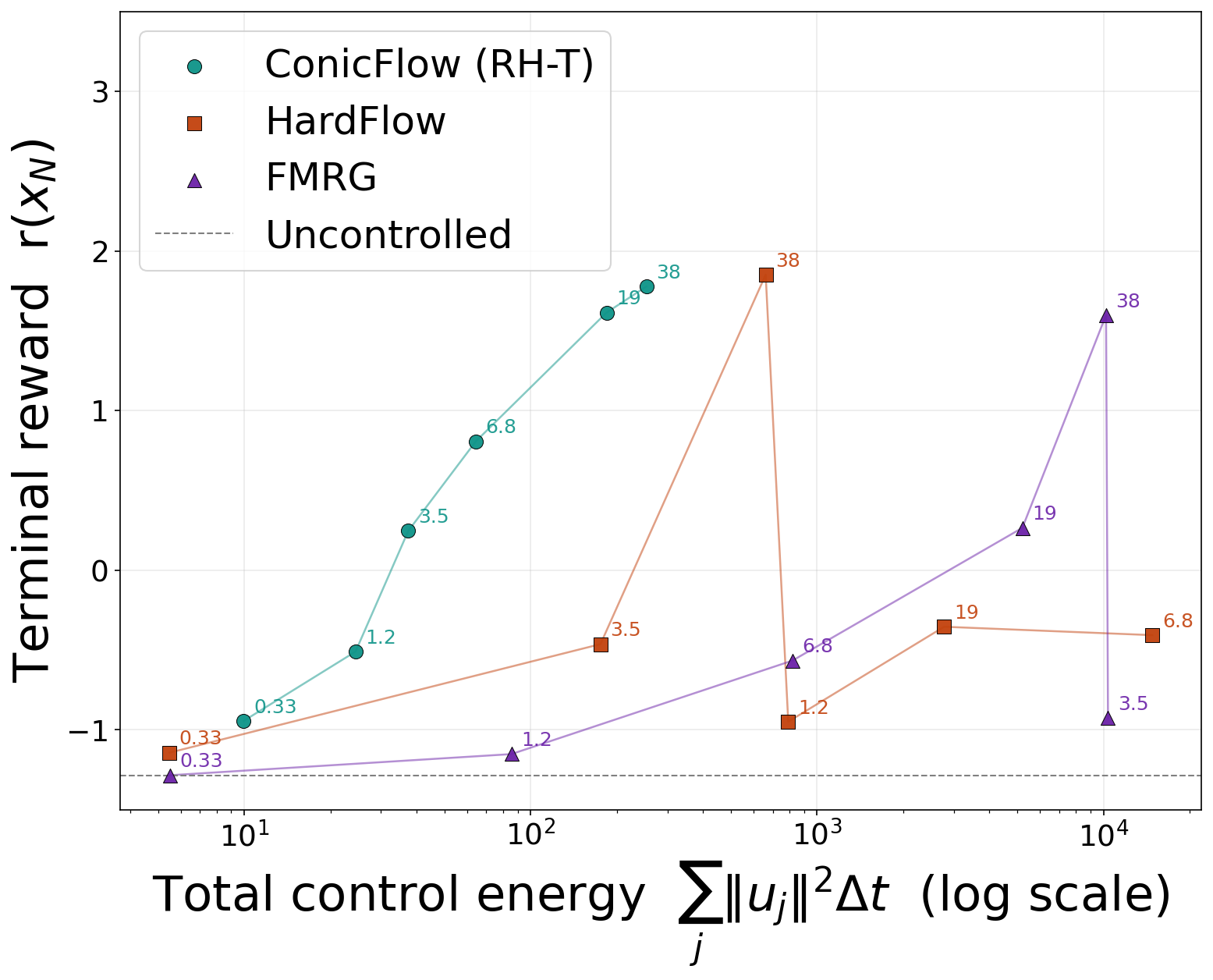}
    \caption{ImageReward, ``a photo of a sleeping cat'' $\rightarrow$ ``a photo
    of a cat with wide open eyes'', seed~42.}
    \label{fig:pareto-app-ir-cat}
  \end{subfigure}\hfill
  \begin{subfigure}[h]{0.49\textwidth}
    \centering
    \includegraphics[width=\linewidth]{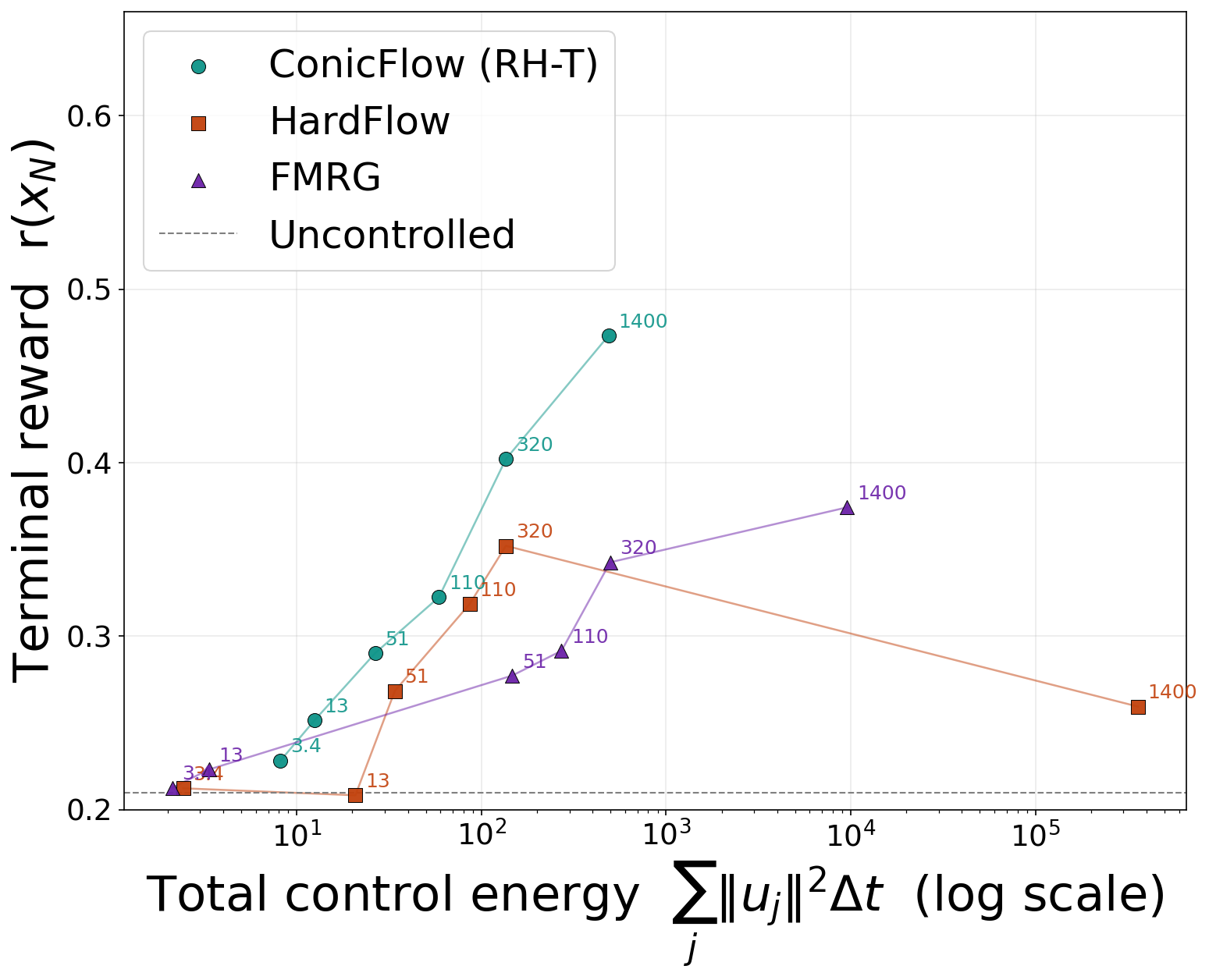}
    \caption{CLIP, ``a portrait of a man'' $\rightarrow$ ``a smiling face'',
    seed~7.}
    \label{fig:pareto-app-clip-portrait}
  \end{subfigure}
  \caption{Reward--effort trade-off on two further editing instances, $N=8$,
  $H=2$. The horizontal axis is control energy
  $\sum_k \lVert u_k \rVert^2 \Delta t$ and the vertical axis the achieved
  reward $r(x_N)$, so up and to the left is better. Each marker is one value of
  $\lambda$.}
  \label{fig:pareto-appendix}
\end{figure*}

\newcommand{\ugridw}{0.49\linewidth}

\begin{figure*}[p]
  \centering
  \begin{subfigure}[t]{\ugridw}
    \centering
    \includegraphics[width=\linewidth]{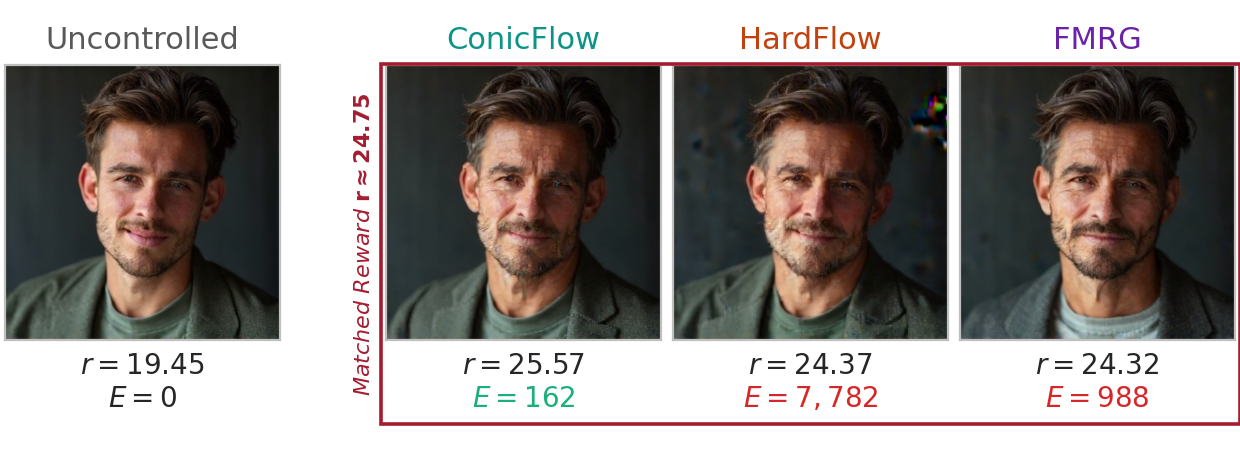}
    \caption{PickScore, seed~$1289$: \emph{``a portrait of a young man''} $\to$ \emph{``a portrait of an old man''}}
    \label{fig:ugrid-u01}
  \end{subfigure}\hfill
  \begin{subfigure}[t]{\ugridw}
    \centering
    \includegraphics[width=\linewidth]{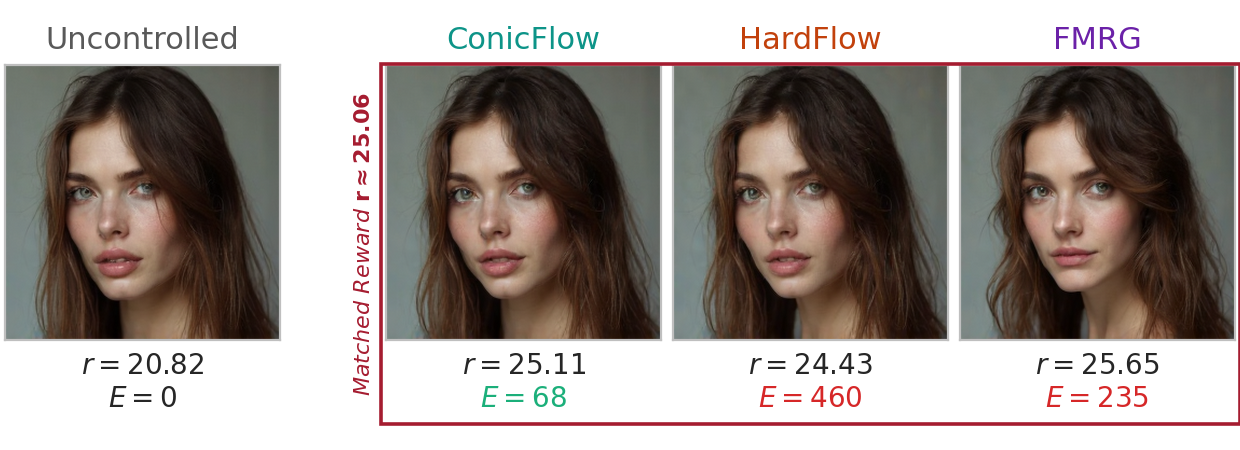}
    \caption{PickScore, seed~$317$: \emph{``a woman with straight hair''} $\to$ \emph{``\ldots with curly hair''}}
    \label{fig:ugrid-u02}
  \end{subfigure}

  \begin{subfigure}[t]{\ugridw}
    \centering
    \includegraphics[width=\linewidth]{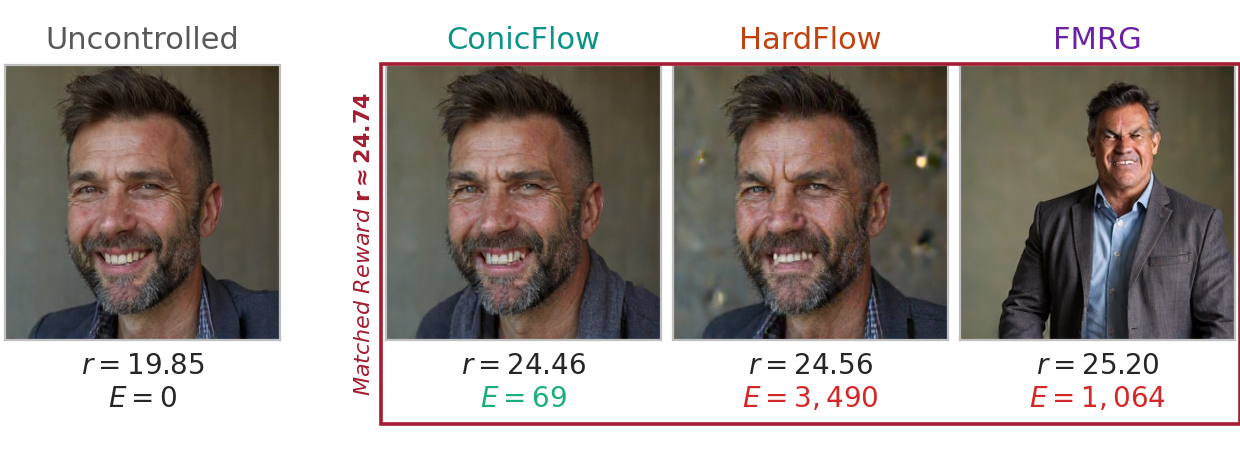}
    \caption{PickScore, seed~$317$: \emph{``a portrait of a smiling man''} $\to$ \emph{``\ldots an angry man''}}
    \label{fig:ugrid-u03}
  \end{subfigure}\hfill
  \begin{subfigure}[t]{\ugridw}
    \centering
    \includegraphics[width=\linewidth]{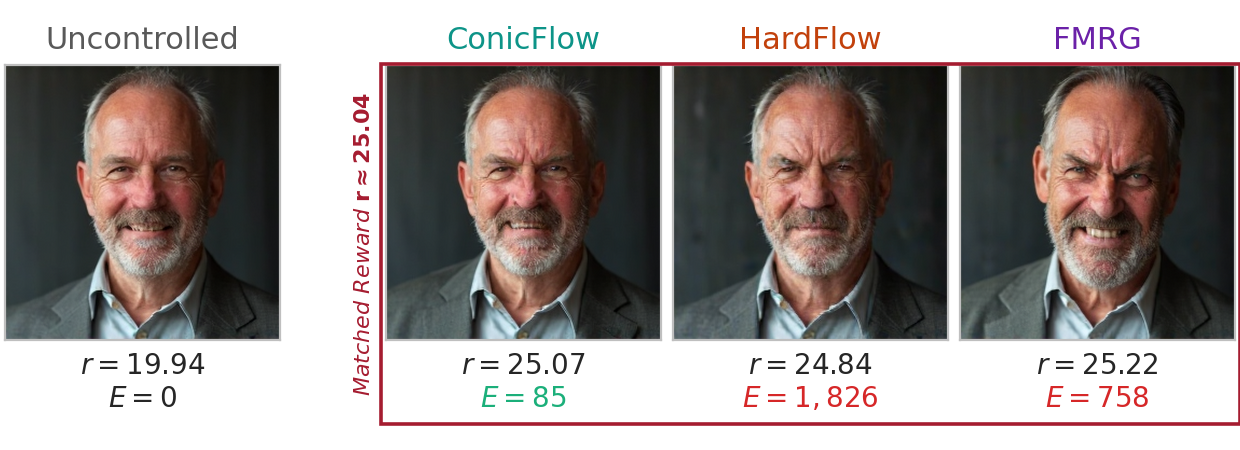}
    \caption{PickScore, seed~$1289$: \emph{``a portrait of a smiling man''} $\to$ \emph{``\ldots an angry man''}}
    \label{fig:ugrid-u04}
  \end{subfigure}

  \begin{subfigure}[t]{\ugridw}
    \centering
    \includegraphics[width=\linewidth]{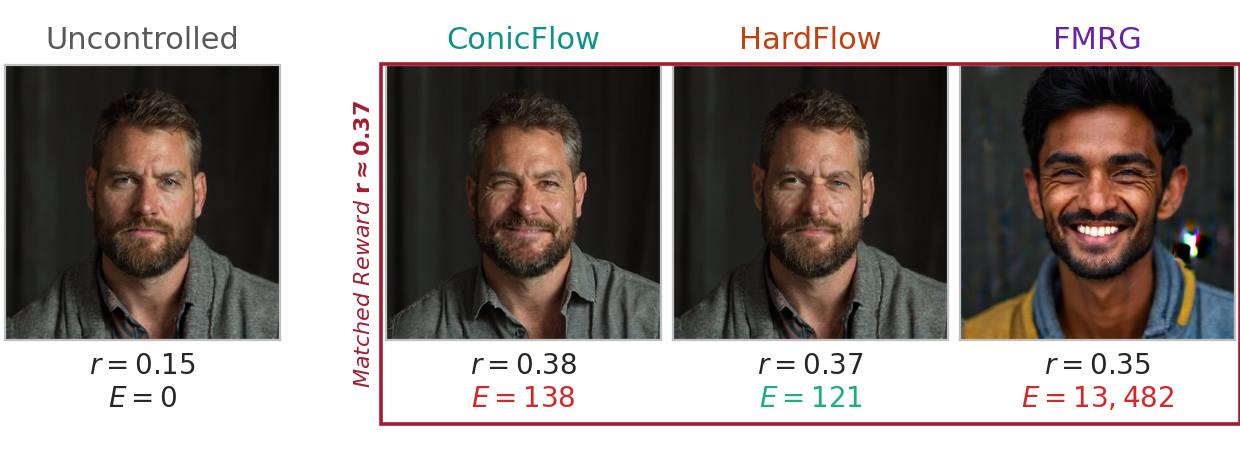}
    \caption{CLIP, seed~$42$: \emph{``a portrait of a man''} $\to$ \emph{``a smiling face''}}
    \label{fig:ugrid-u05}
  \end{subfigure}\hfill
  \begin{subfigure}[t]{\ugridw}
    \centering
    \includegraphics[width=\linewidth]{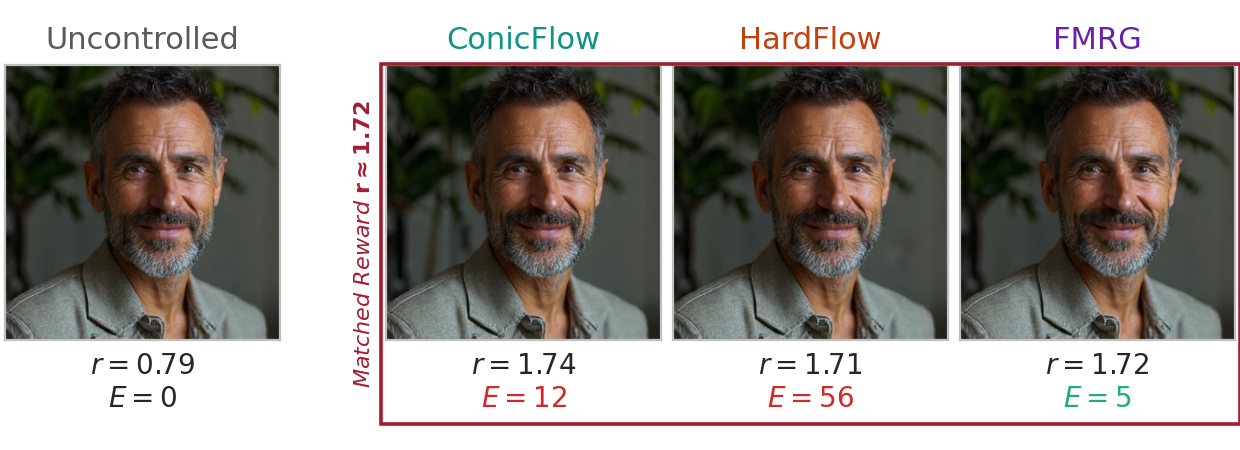}
    \caption{ImageReward, seed~$7$: \emph{``a portrait of a man''} $\to$ \emph{``a smiling face''}}
    \label{fig:ugrid-u06}
  \end{subfigure}

  \begin{subfigure}[t]{\ugridw}
    \centering
    \includegraphics[width=\linewidth]{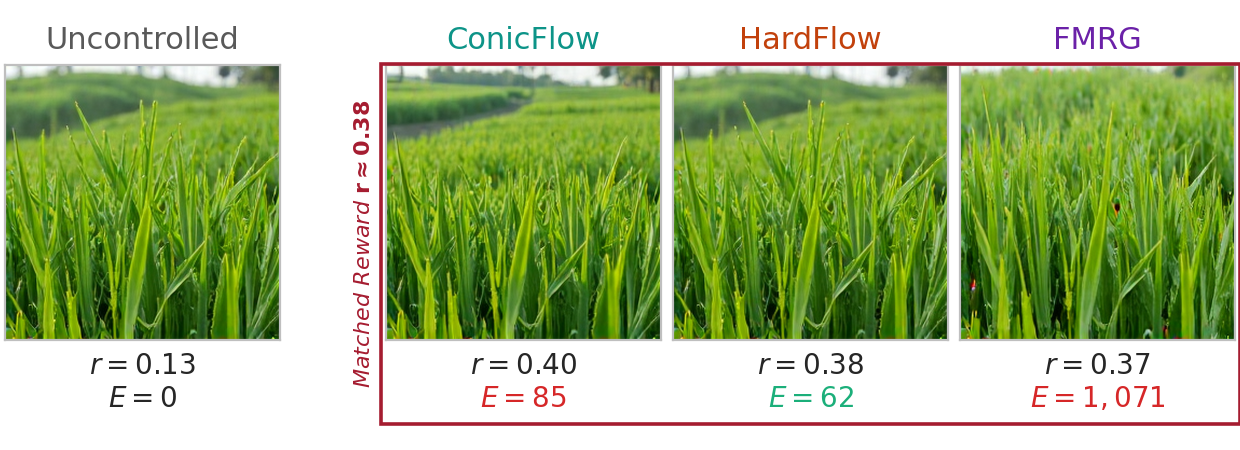}
    \caption{CLIP, seed~$42$: \emph{``a field of green grass''} $\to$ \emph{``a field of red poppies''}}
    \label{fig:ugrid-u07}
  \end{subfigure}\hfill
  \begin{subfigure}[t]{\ugridw}
    \centering
    \includegraphics[width=\linewidth]{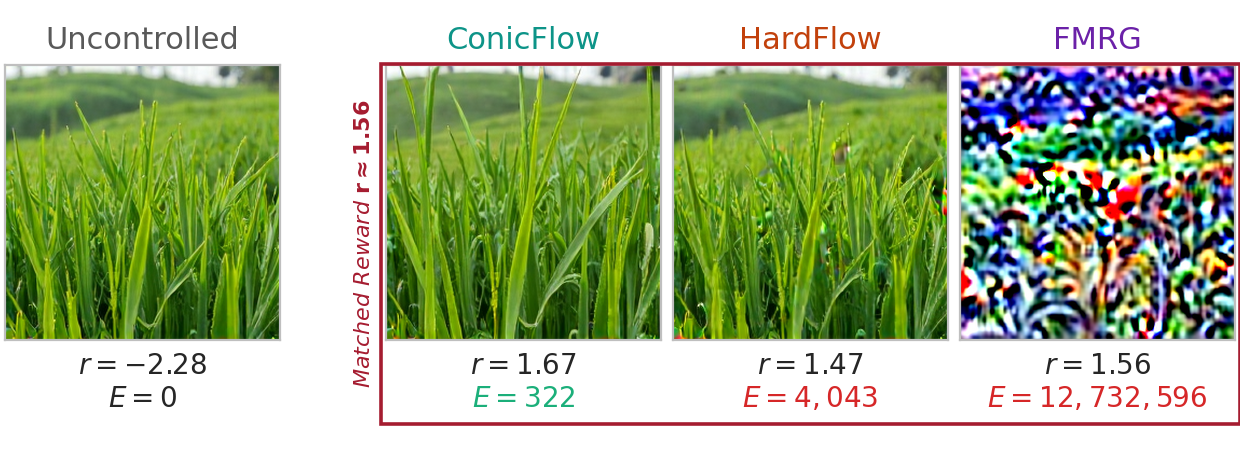}
    \caption{ImageReward, seed~$42$: \emph{``a field of green grass''} $\to$ \emph{``a field of red poppies''}}
    \label{fig:ugrid-u08}
  \end{subfigure}

  \begin{subfigure}[t]{\ugridw}
    \centering
    \includegraphics[width=\linewidth]{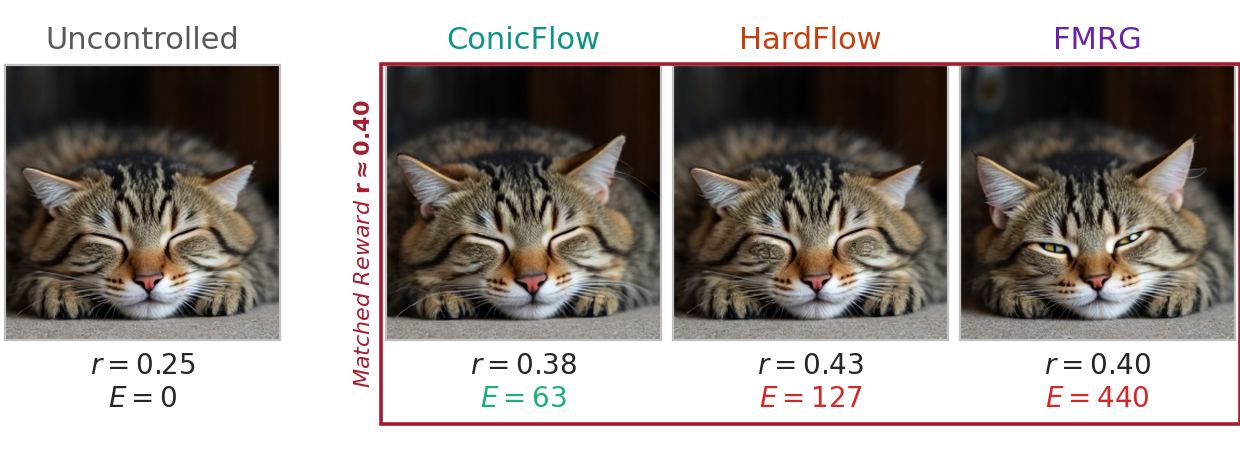}
    \caption{CLIP, seed~$42$: \emph{``a photo of a sleeping cat''} $\to$ \emph{``\ldots with wide open eyes''}}
    \label{fig:ugrid-u09}
  \end{subfigure}\hfill
  \begin{subfigure}[t]{\ugridw}
    \centering
    \includegraphics[width=\linewidth]{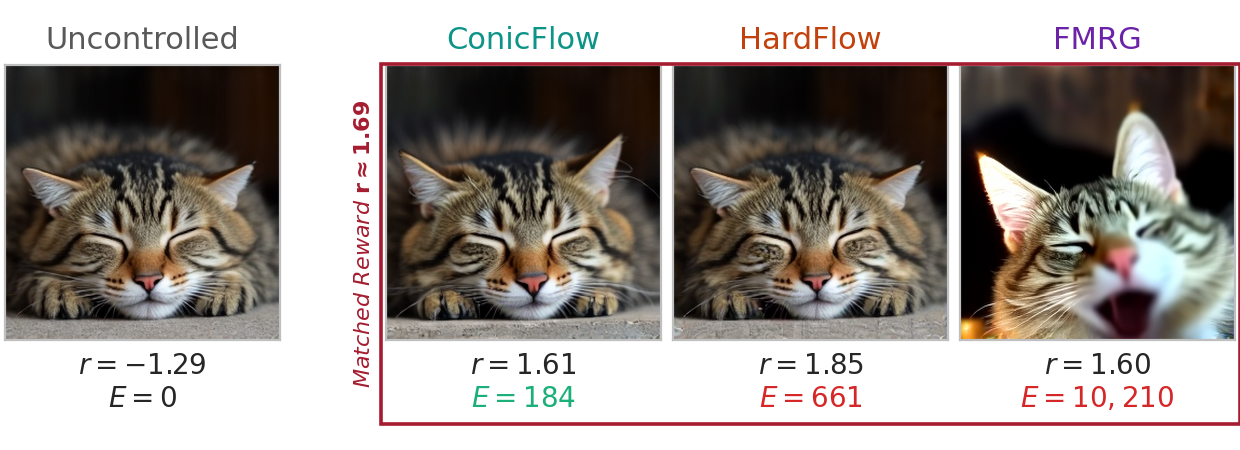}
    \caption{ImageReward, seed~$42$: \emph{``a photo of a sleeping cat''} $\to$ \emph{``\ldots with wide open eyes''}}
    \label{fig:ugrid-u10}
  \end{subfigure}
  \caption{More examples of unconstrained editing. Each panel is one
  (reward, prompt, seed) instance at a matched-reward level. Panels \subref{fig:ugrid-u07}--\subref{fig:ugrid-u10} are the harder case: turning a green field red, or opening a sleeping cat's eyes, asks for the
  image to be restructured rather than adjusted, and no method delivers it.}
  \label{fig:unconstrained-grid}
\end{figure*}

\paragraph{Constrained} We present additional visualizations of the constrained image editing comparisons in Figure~\ref{fig:constrained-grid-1},~\ref{fig:constrained-grid-2}, and~\ref{fig:constrained-grid-3}, similar to Figure~\ref{fig:edit-rows} in the main paper.

\newcommand{\cgridw}{0.49\linewidth}

\begin{figure*}[p]
  \centering
  \begin{subfigure}[b]{\linewidth}
    \centering
    \includegraphics[width=\cgridw]{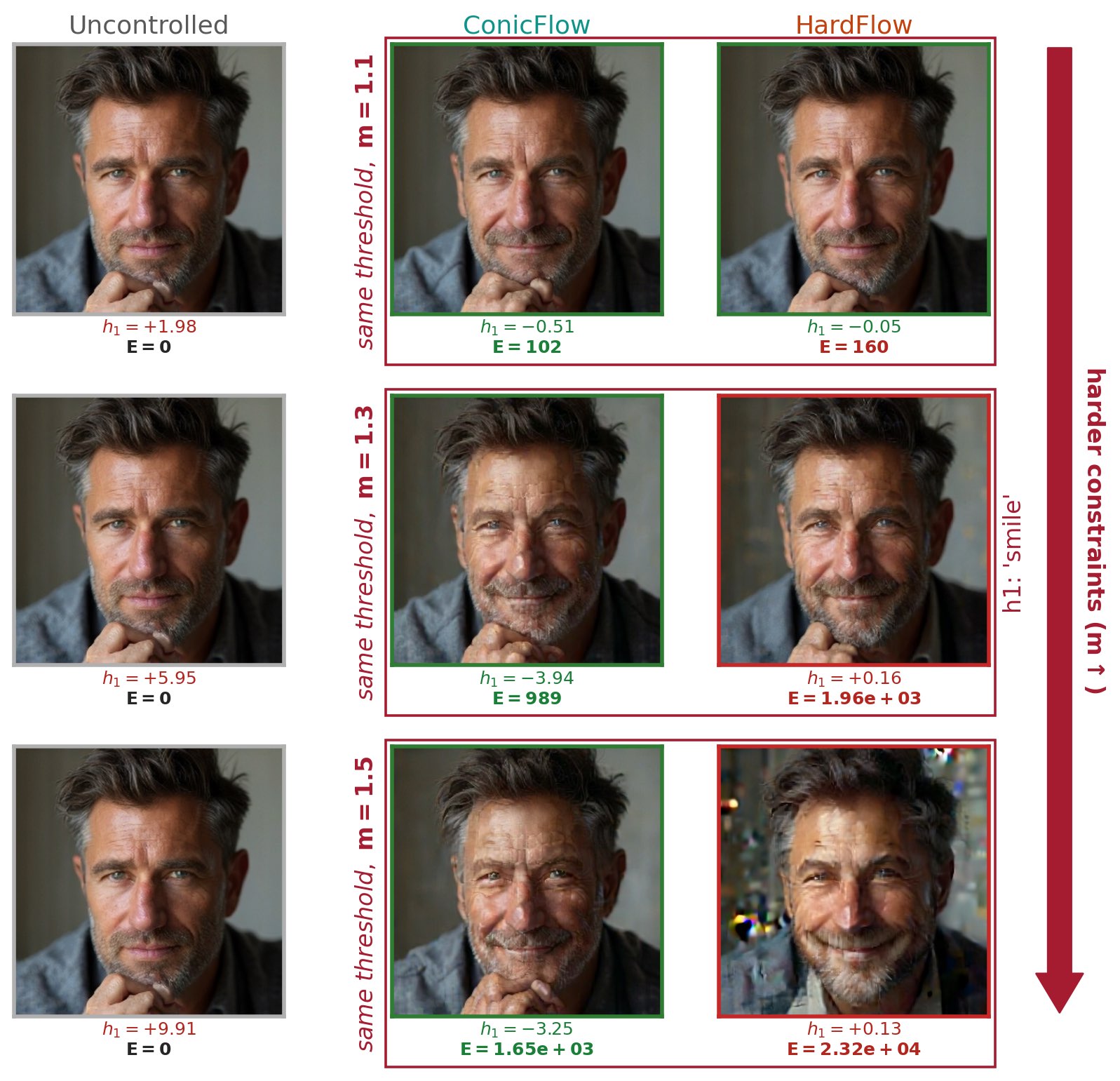}\hfill
    \includegraphics[width=\cgridw]{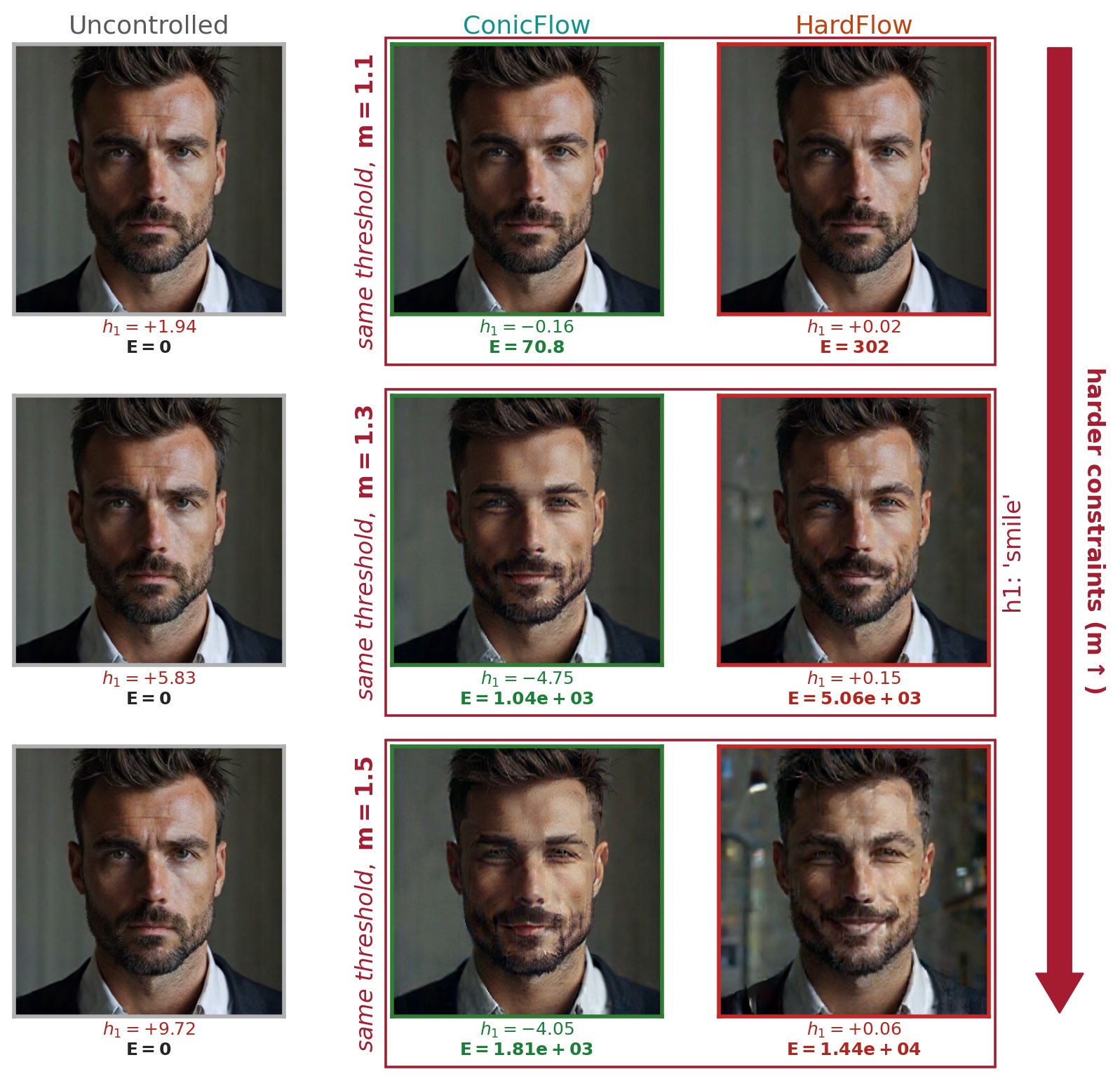}
    \caption{\emph{``a portrait of a man''} $\to$ \emph{``a smiling face''}}
    \label{fig:cgrid-p01_smile}
  \end{subfigure}

  \begin{subfigure}[b]{\linewidth}
    \centering
    \includegraphics[width=\cgridw]{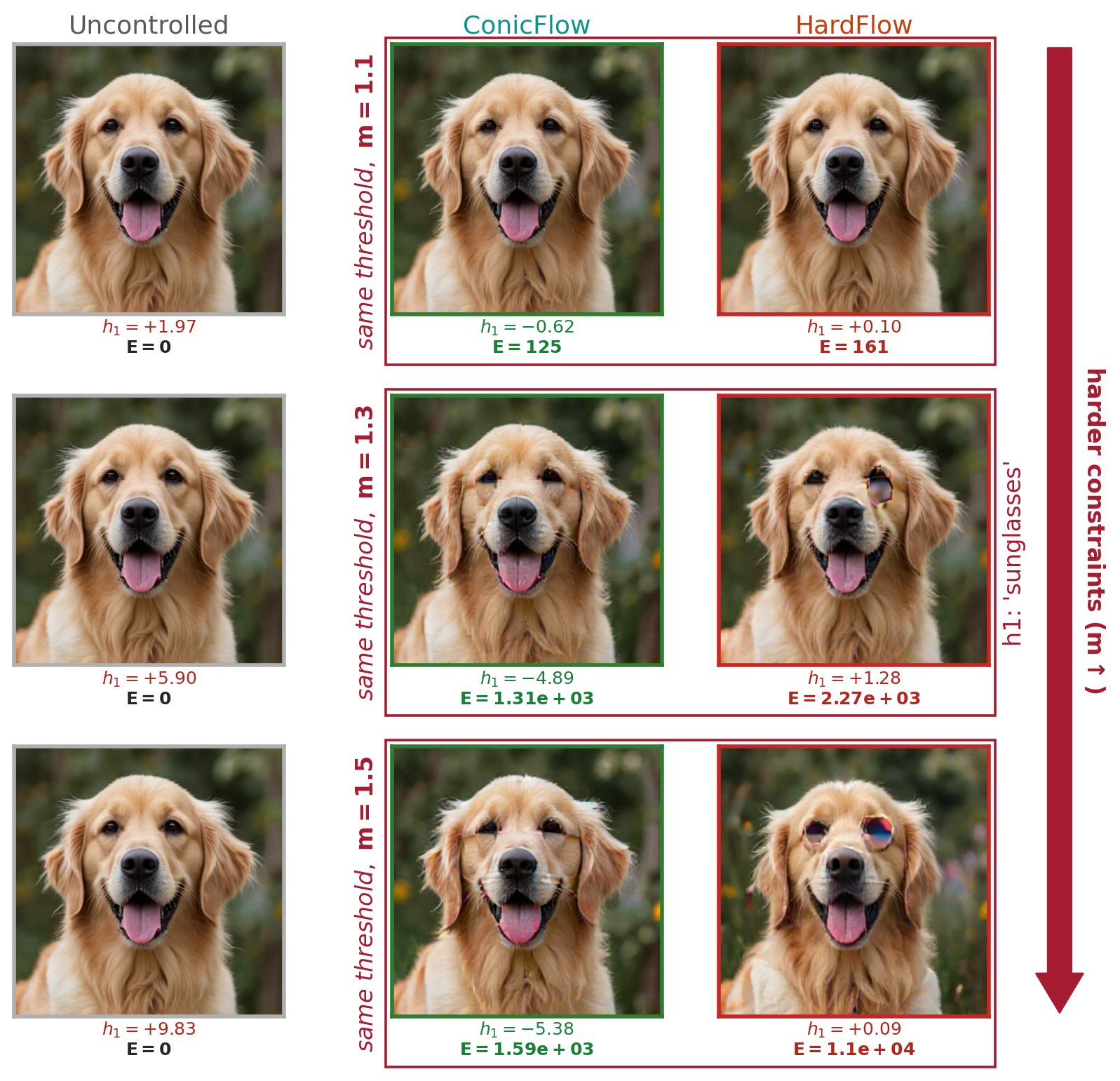}\hfill
    \includegraphics[width=\cgridw]{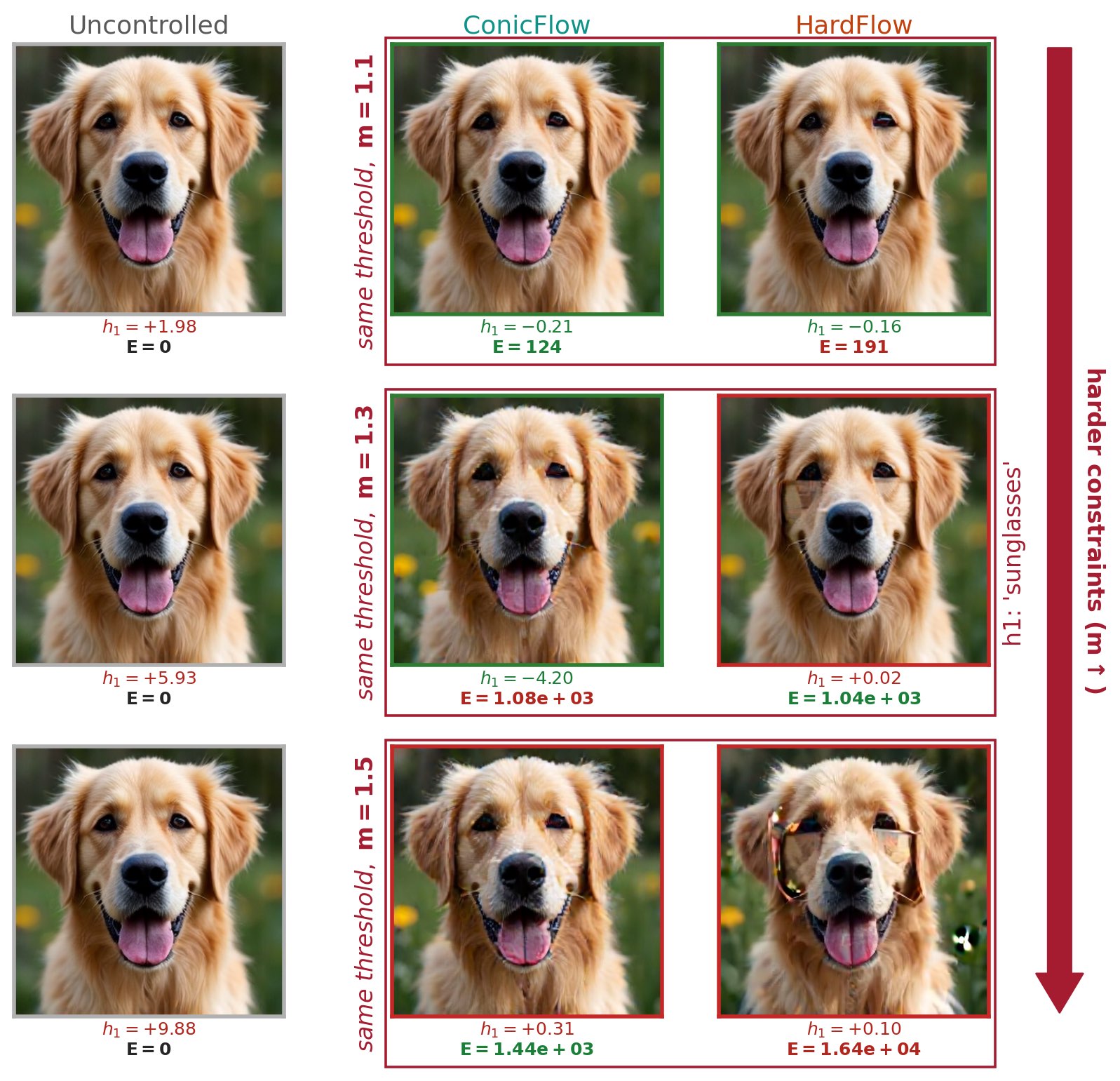}
    \caption{\emph{``a photo of a golden retriever''} $\to$ \emph{``a dog wearing sunglasses''}}
    \label{fig:cgrid-p02_sunglasses}
  \end{subfigure}

  \begin{subfigure}[b]{\linewidth}
    \centering
    \includegraphics[width=\cgridw]{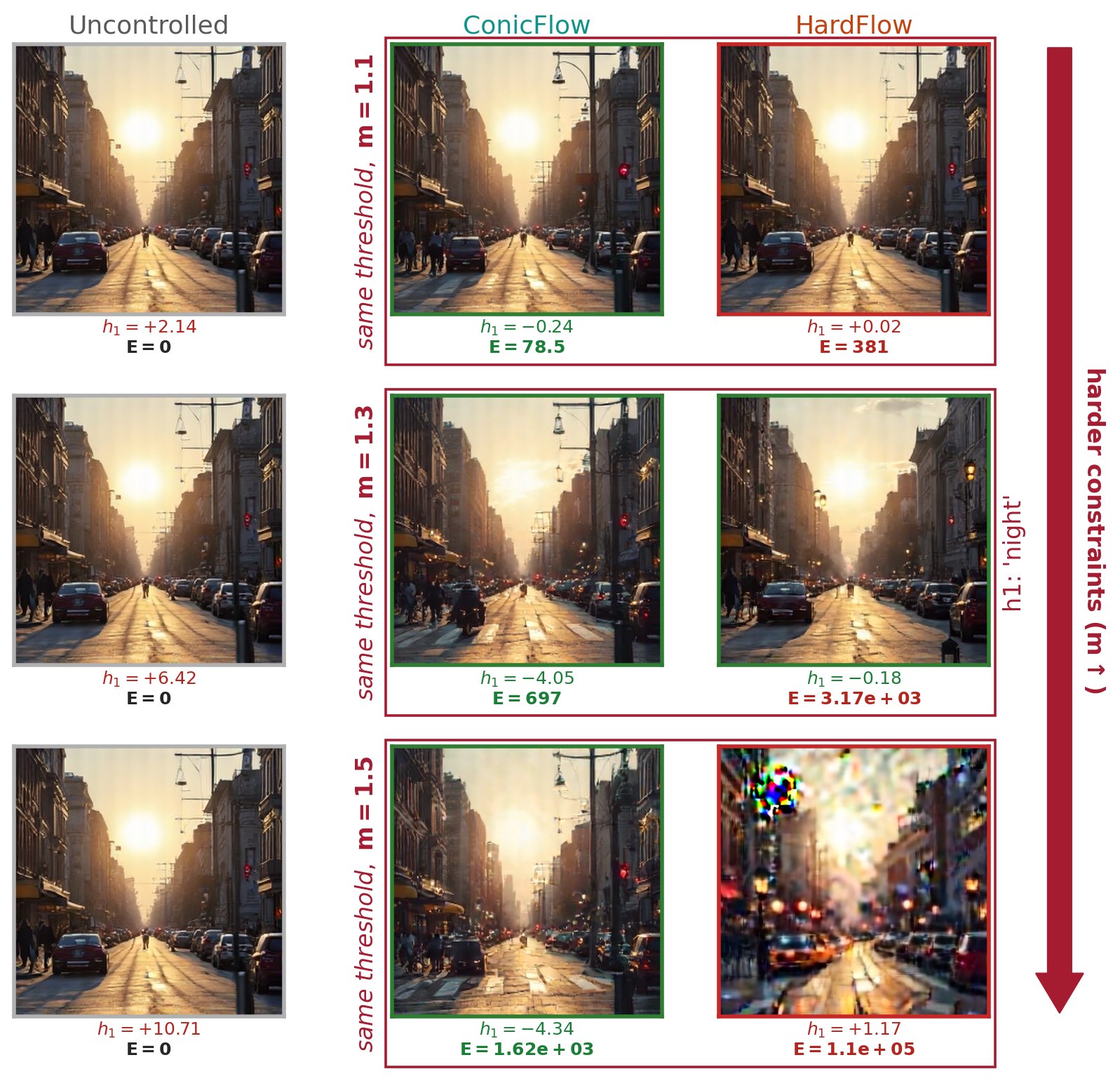}\hfill
    \includegraphics[width=\cgridw]{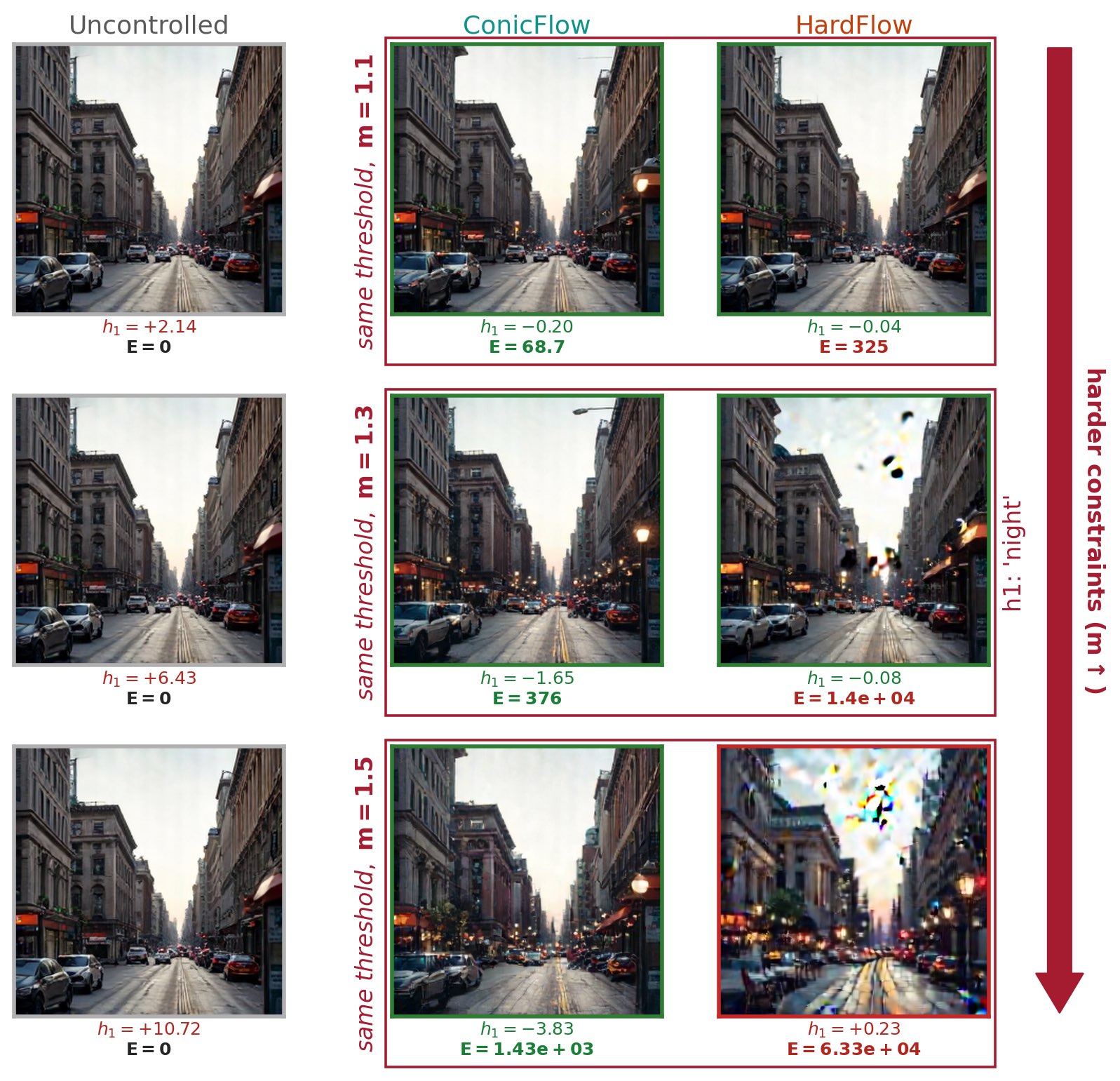}
    \caption{\emph{``a city street at noon''} $\to$ \emph{``a city street at night''}}
    \label{fig:cgrid-p04_night}
  \end{subfigure}
  \caption{More examples of constrained editing. Each panel is one (prompt, seed)
  instance. Instance details are in Table~\ref{tab:constrained-instances}.}
  \label{fig:constrained-grid-1}
\end{figure*}

\begin{figure*}[p]
  \centering
  \begin{subfigure}[b]{\linewidth}
    \centering
    \includegraphics[width=\cgridw]{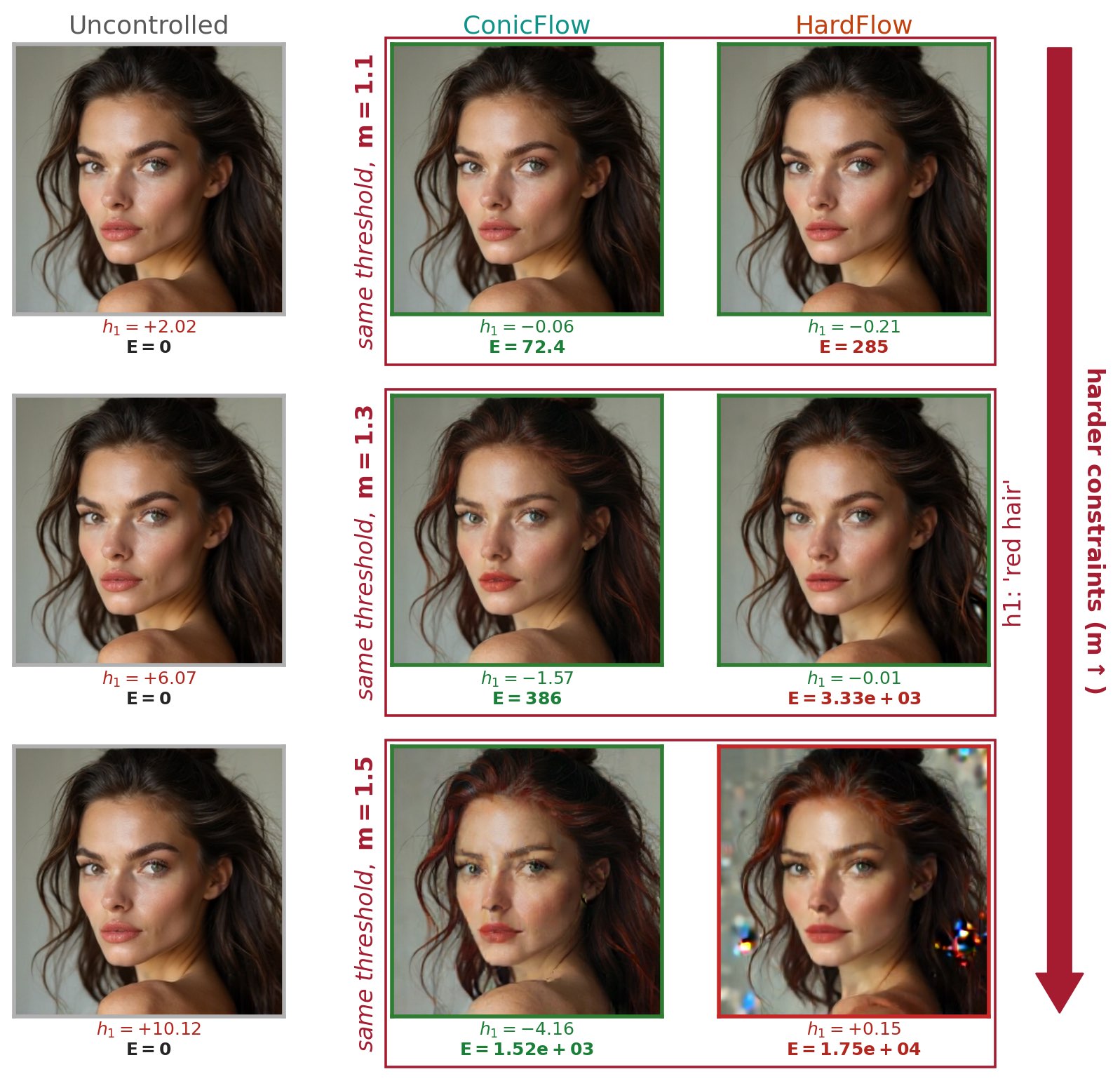}\hfill
    \includegraphics[width=\cgridw]{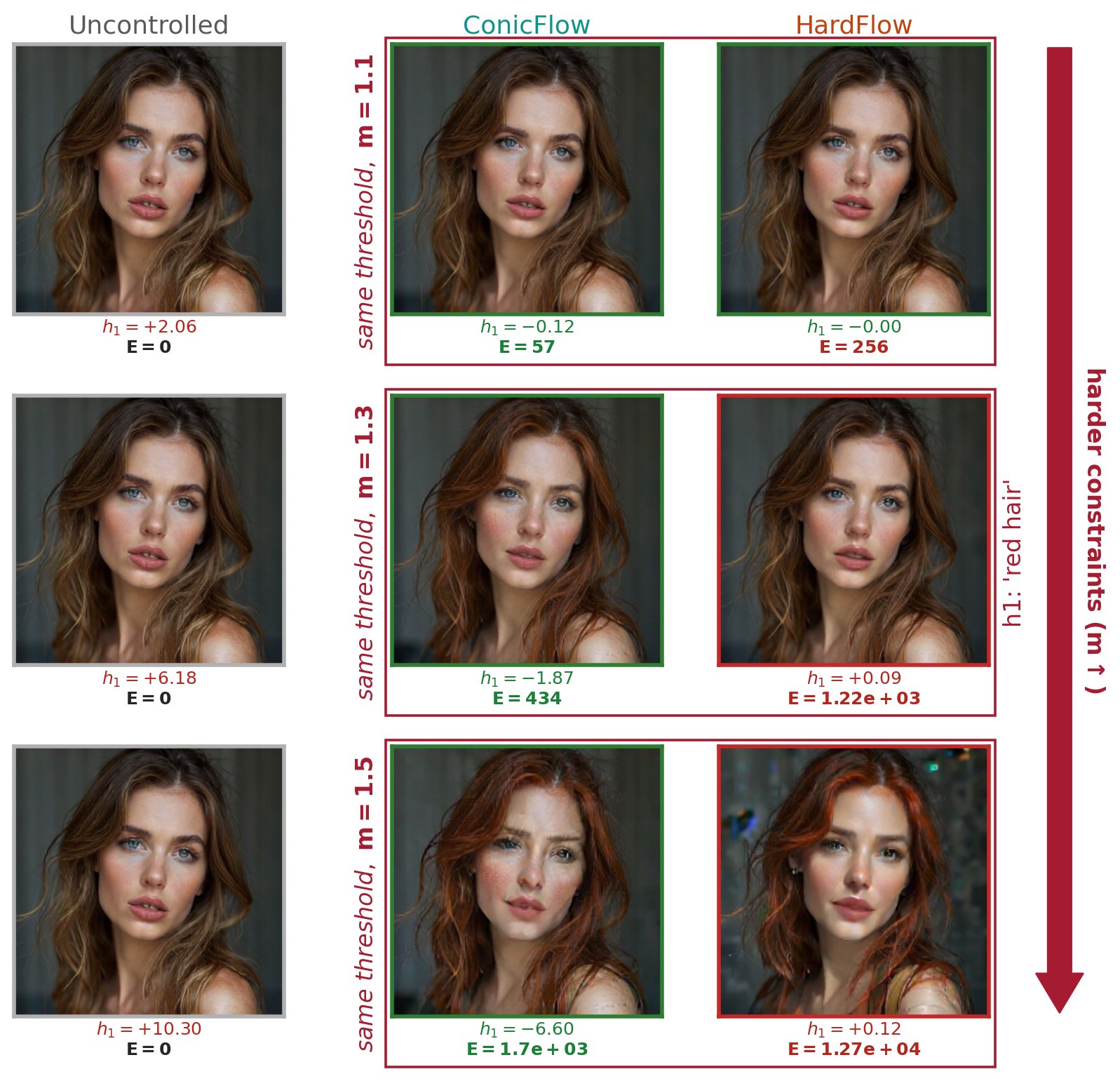}
    \caption{\emph{``a portrait of a woman''} $\to$ \emph{``a woman with bright red hair''}}
    \label{fig:cgrid-p05a_redhair}
  \end{subfigure}

  \begin{subfigure}[b]{\linewidth}
    \centering
    \includegraphics[width=\cgridw]{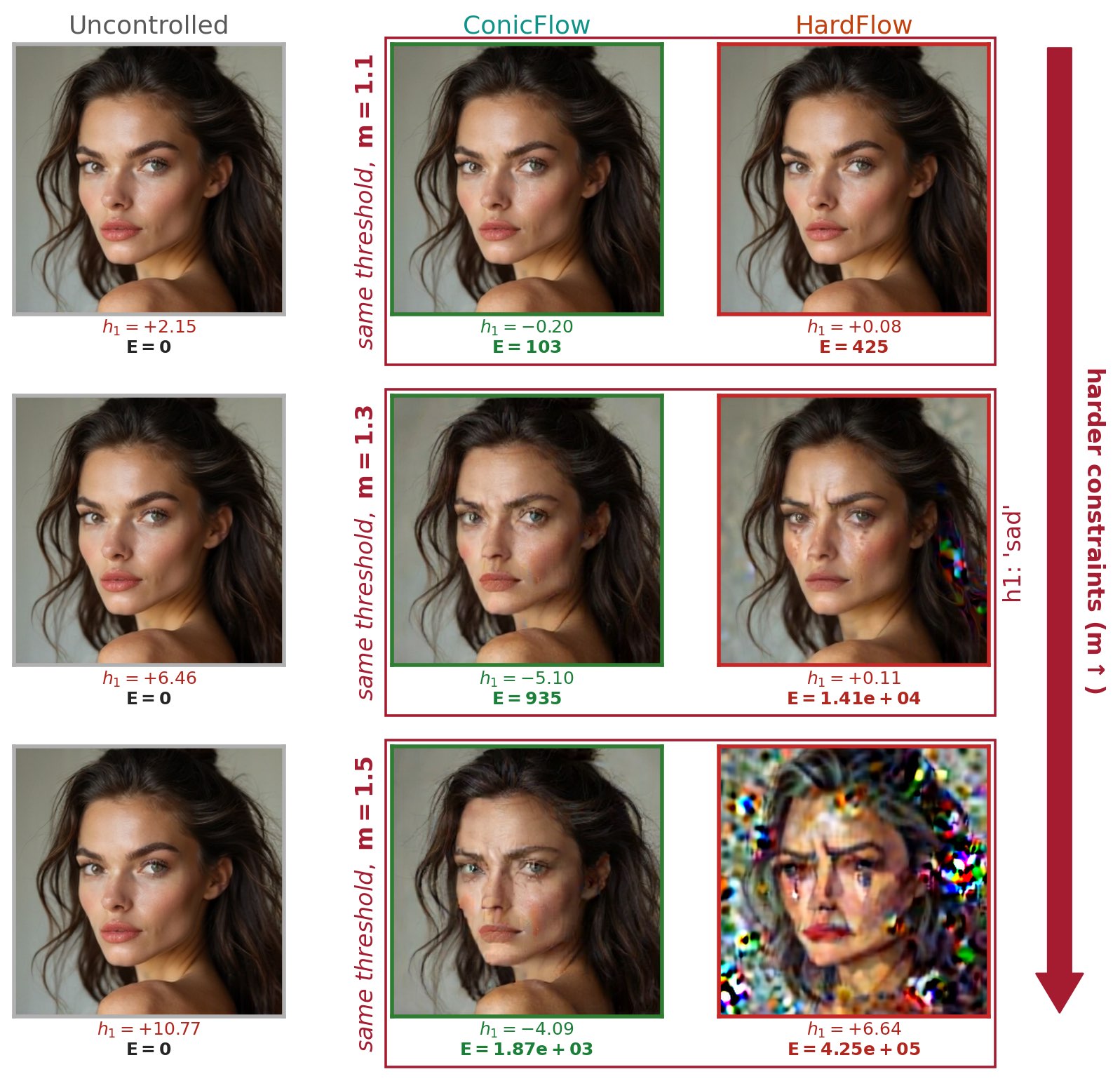}\hfill
    \includegraphics[width=\cgridw]{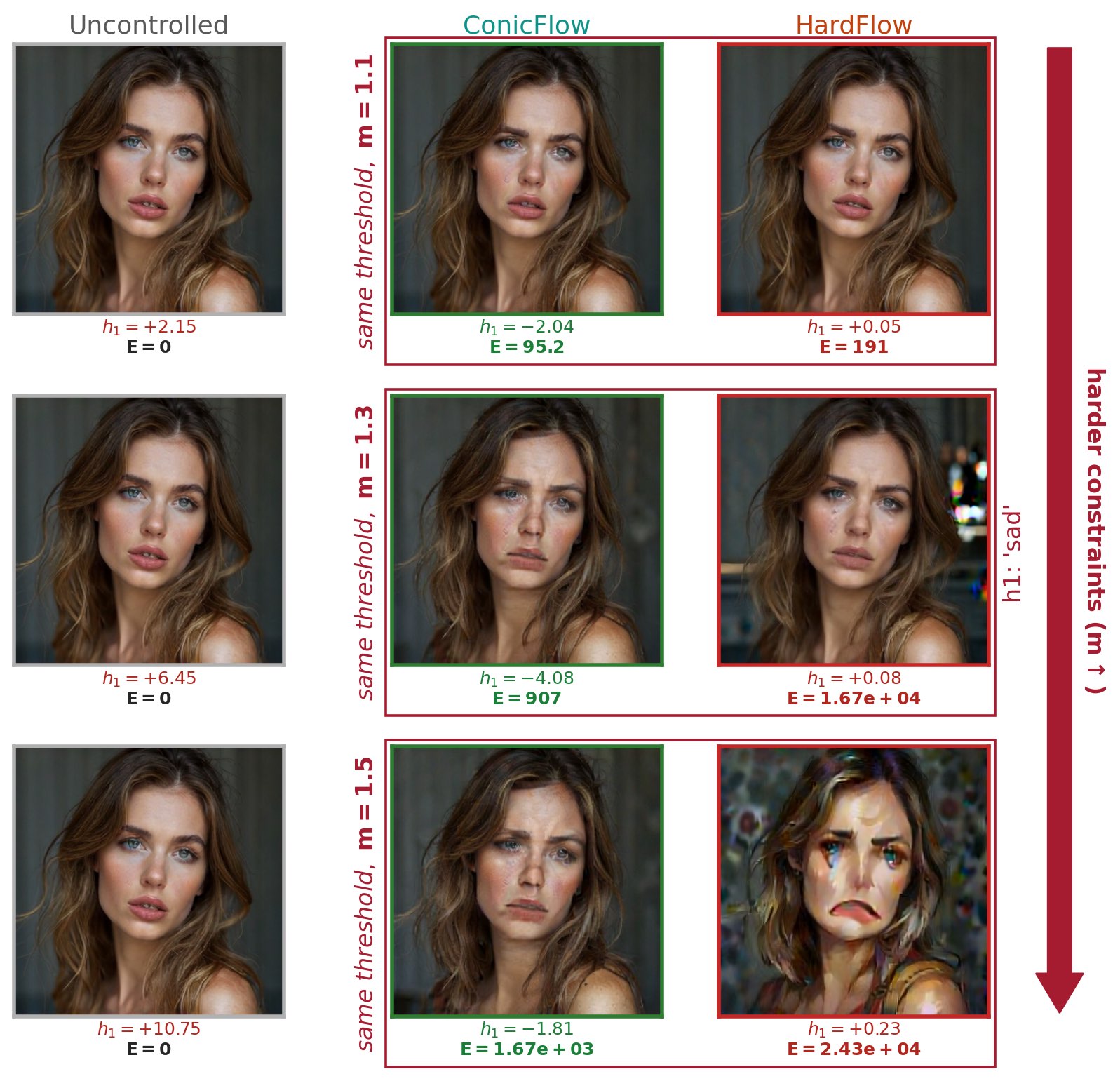}
    \caption{\emph{``a portrait of a woman''} $\to$ \emph{``sad frowning woman''}}
    \label{fig:cgrid-p05b_sad}
  \end{subfigure}

  \begin{subfigure}[b]{\linewidth}
    \centering
    \includegraphics[width=\cgridw]{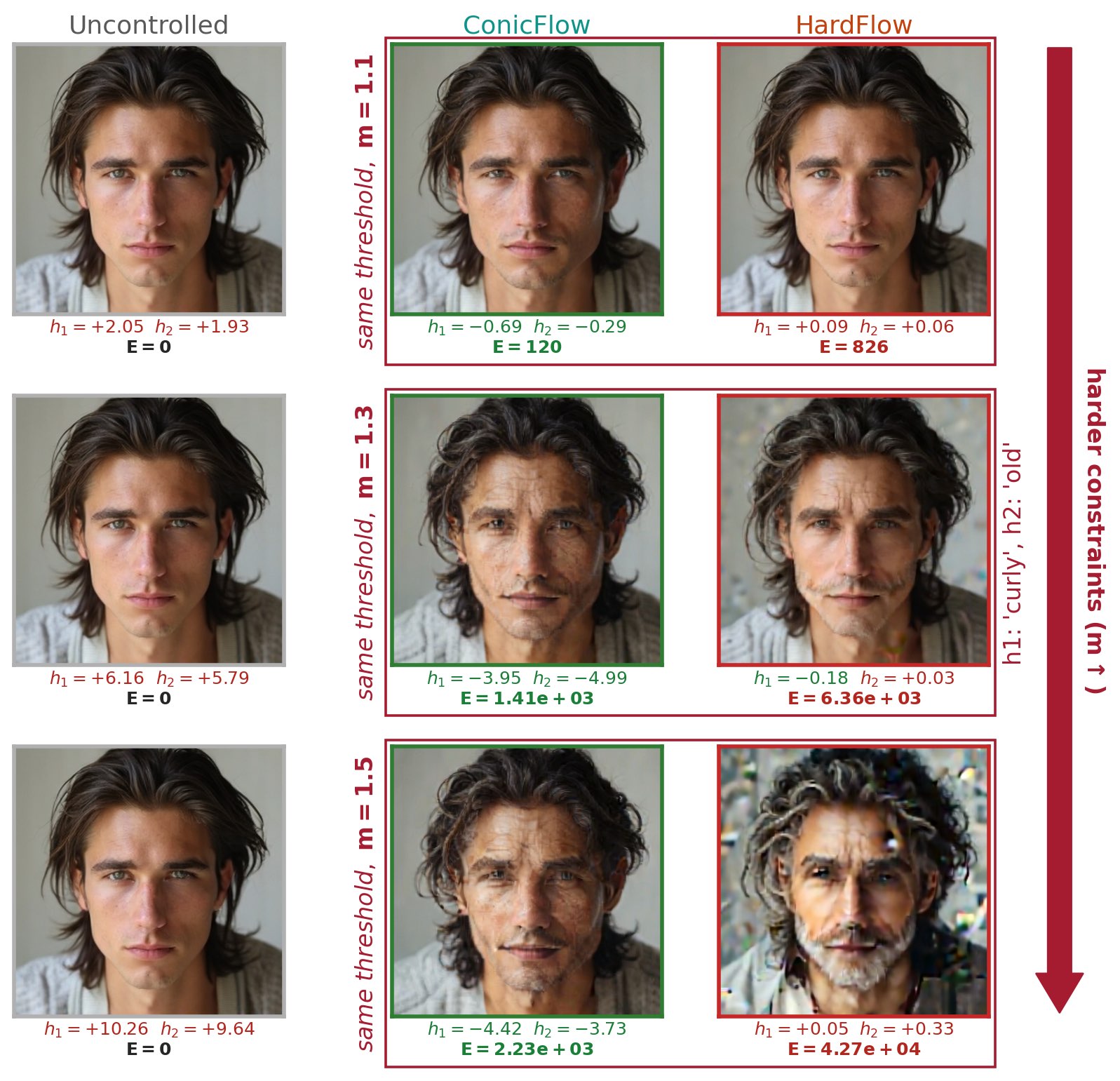}\hfill
    \includegraphics[width=\cgridw]{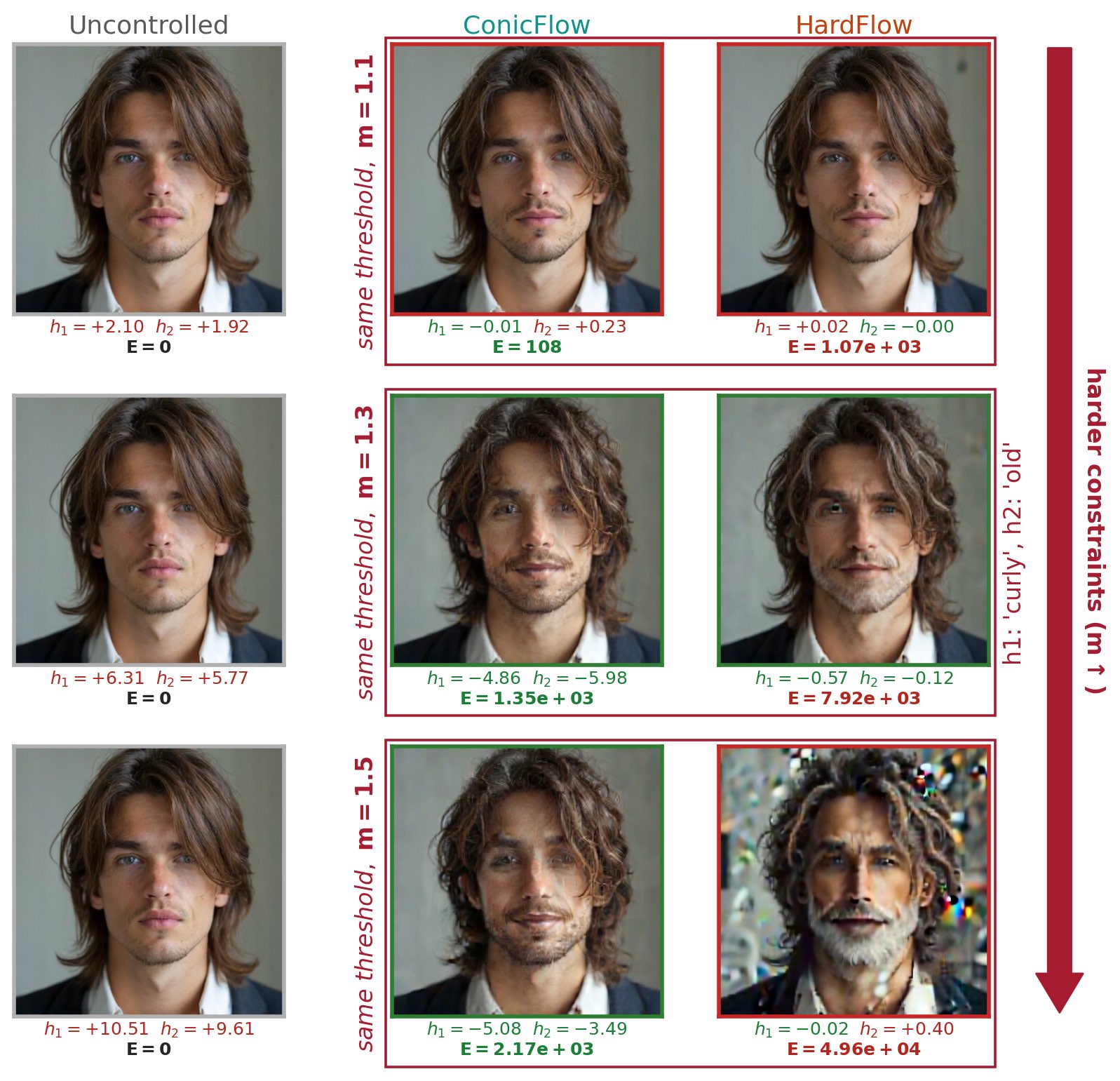}
    \caption{\emph{``portrait of a young man with straight hair''} $\to$ \emph{``man with curly hair; photo of very old man''}}
    \label{fig:cgrid-p09_curly_old}
  \end{subfigure}
  \caption{More examples of constrained editing. Each panel is one (prompt, seed) instance. Instance details are in Table~\ref{tab:constrained-instances}.}
  \label{fig:constrained-grid-2}
\end{figure*}

\begin{figure*}[p]
  \centering
  \begin{subfigure}[b]{\linewidth}
    \centering
    \includegraphics[width=\cgridw]{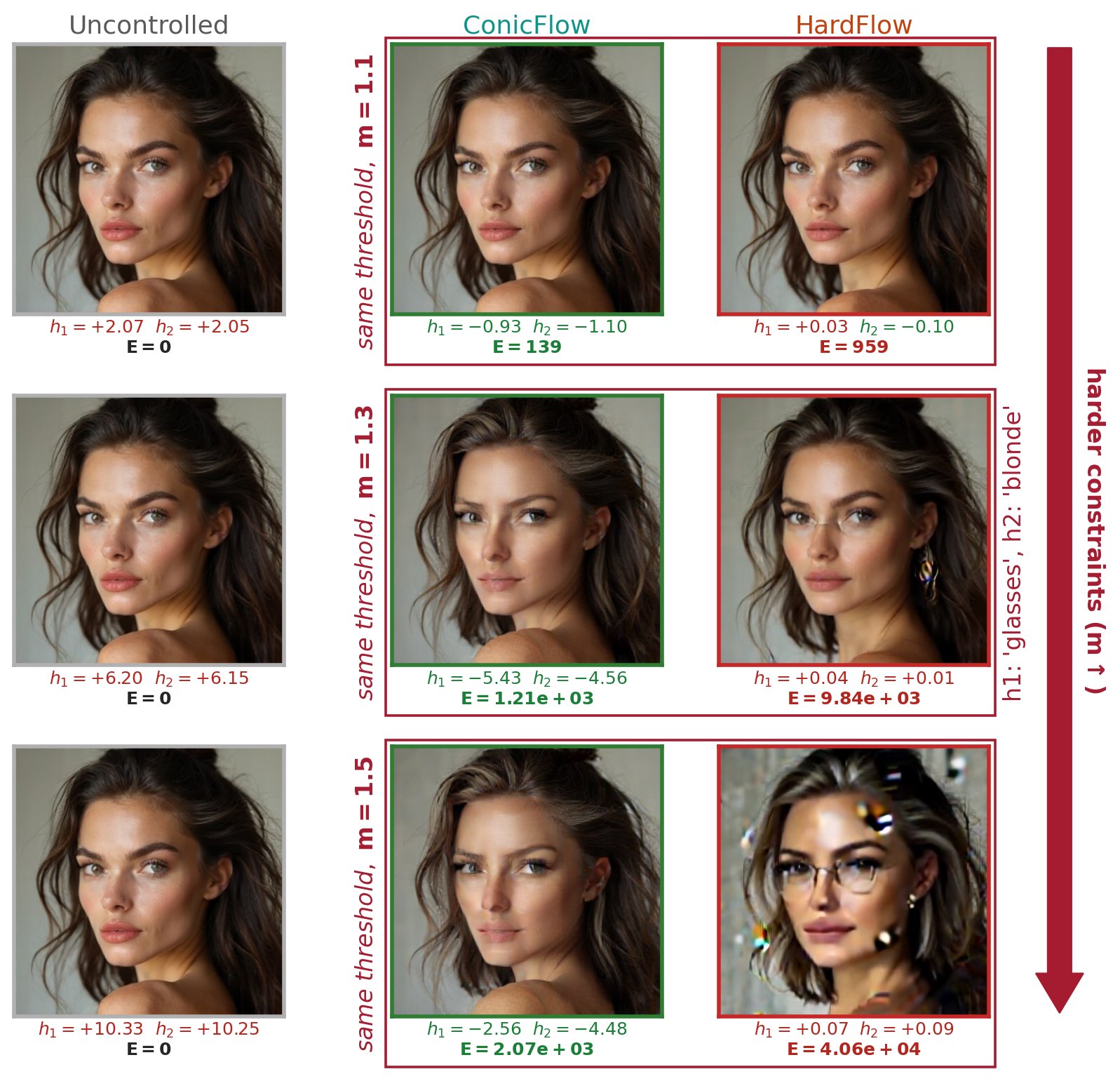}\hfill
    \includegraphics[width=\cgridw]{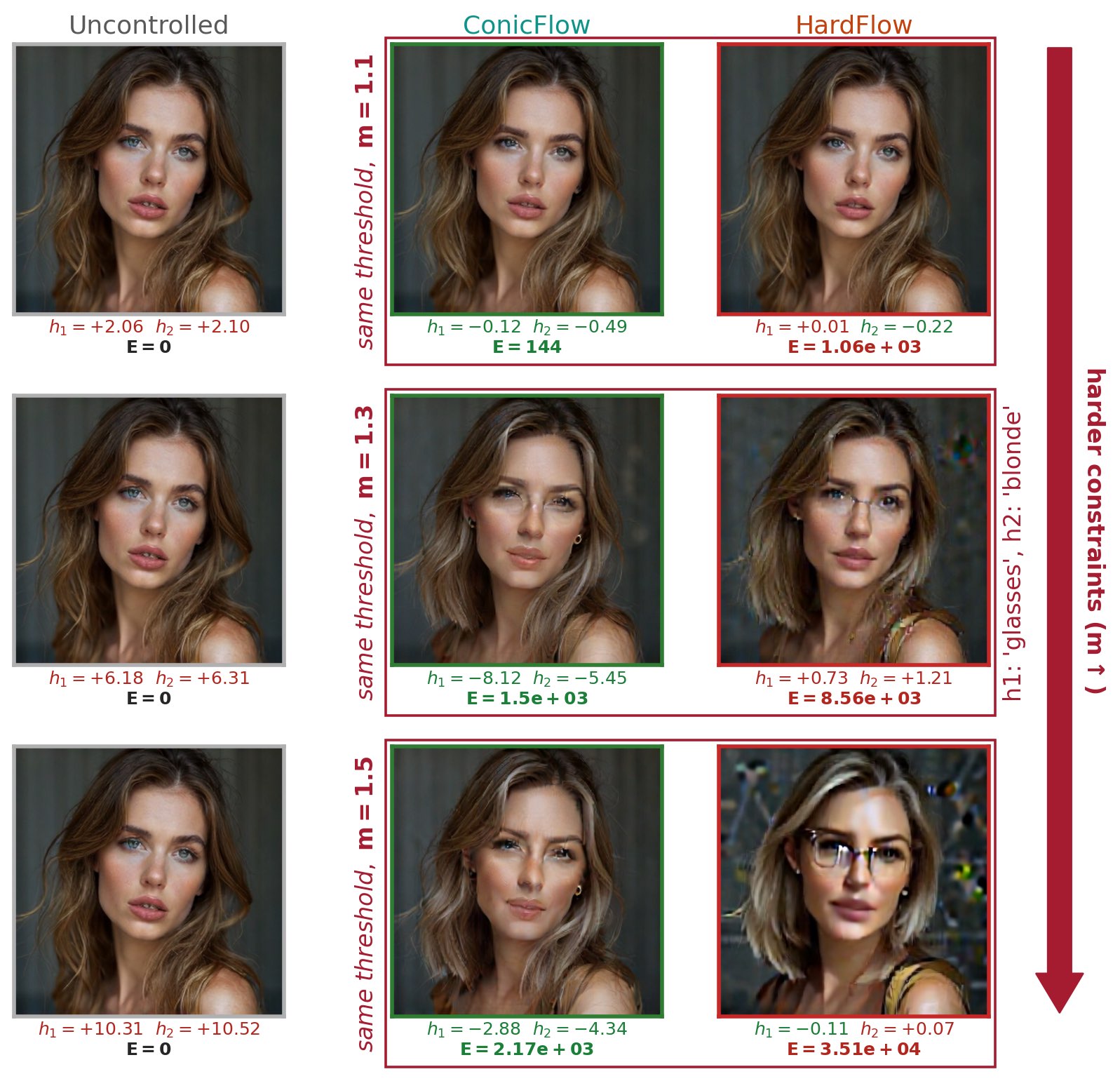}
    \caption{\emph{``a portrait of a woman''} $\to$ \emph{``a woman wearing glasses; a woman with short blonde hair''}}
    \label{fig:cgrid-p10_glass_blonde}
  \end{subfigure}

  \begin{subfigure}[b]{\linewidth}
    \centering
    \includegraphics[width=\cgridw]{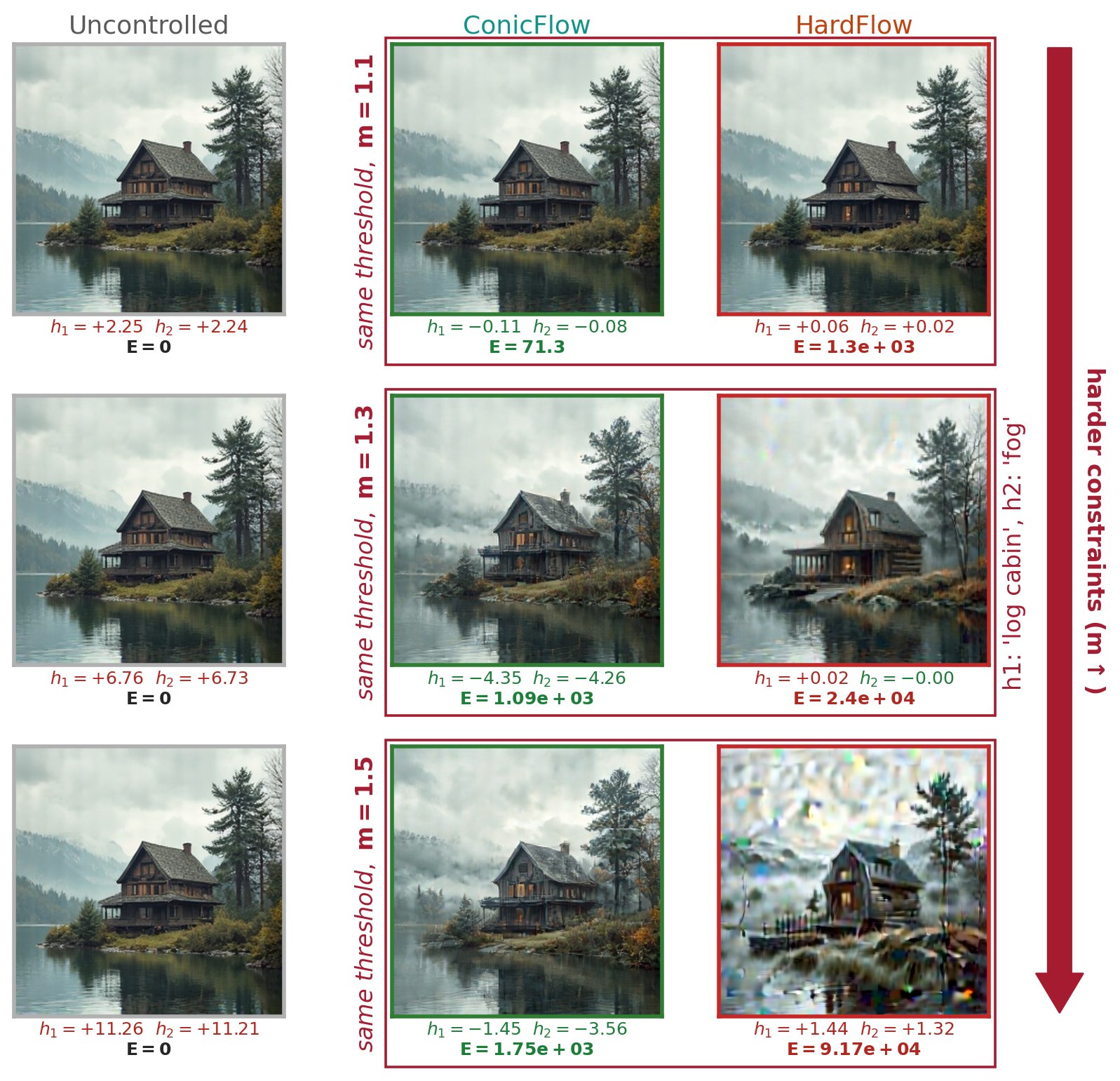}\hfill
    \includegraphics[width=\cgridw]{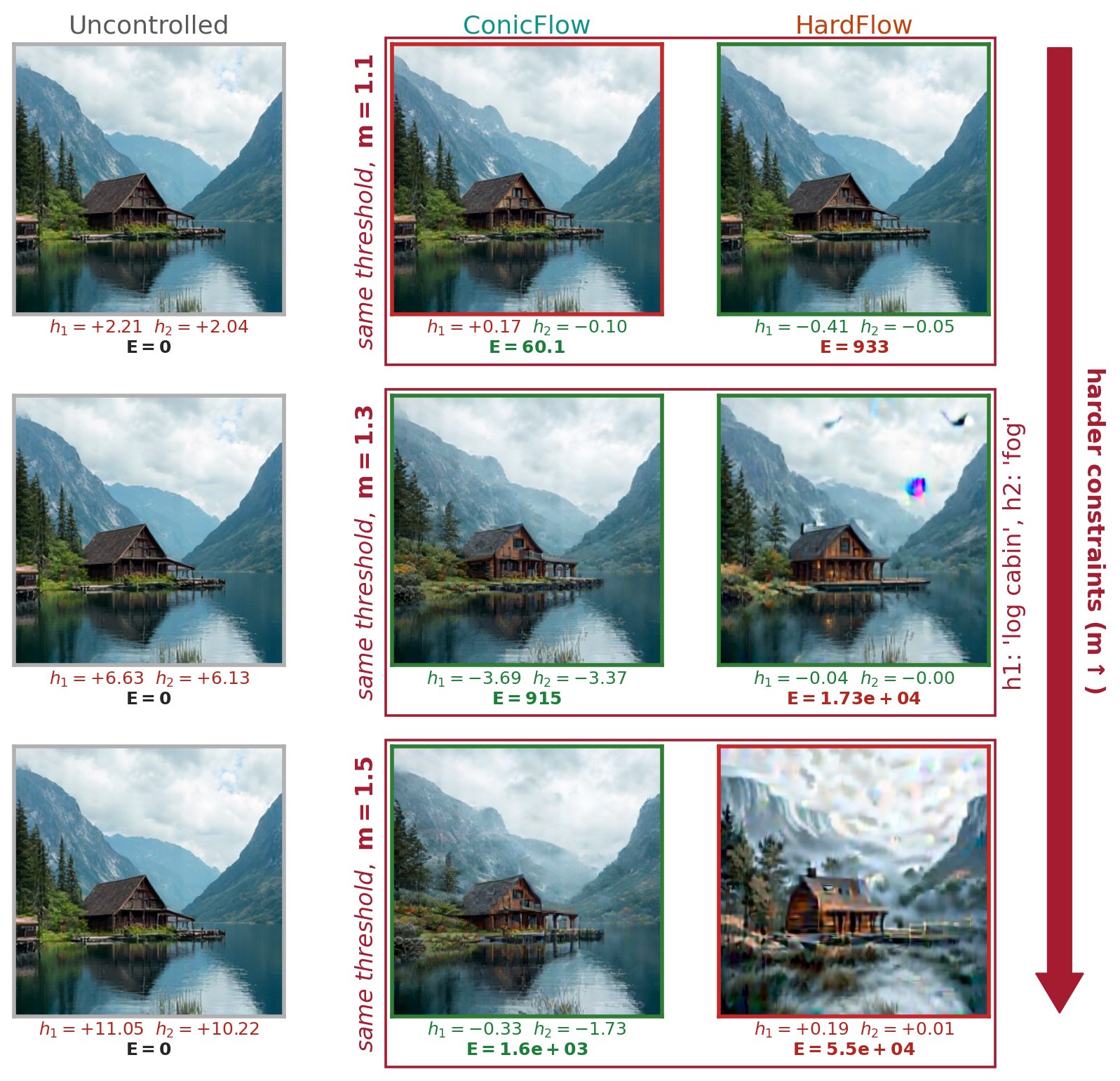}
    \caption{\emph{``a house by a lake''} $\to$ \emph{``a wooden log cabin; a house in heavy fog''}}
    \label{fig:cgrid-p11_cabin_fog}
  \end{subfigure}

  \begin{subfigure}[b]{\linewidth}
    \centering
    \includegraphics[width=\cgridw]{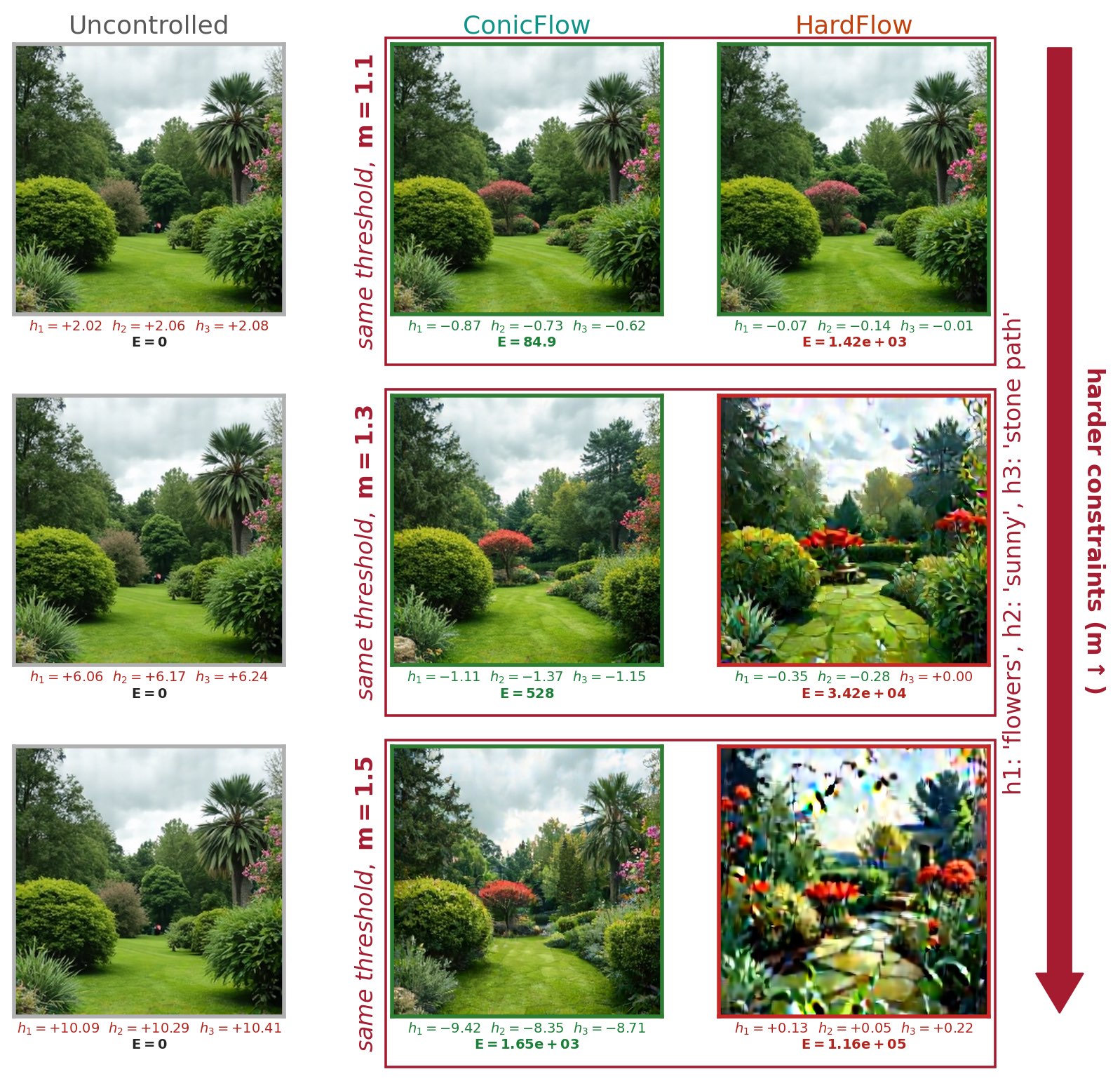}\hfill
    \includegraphics[width=\cgridw]{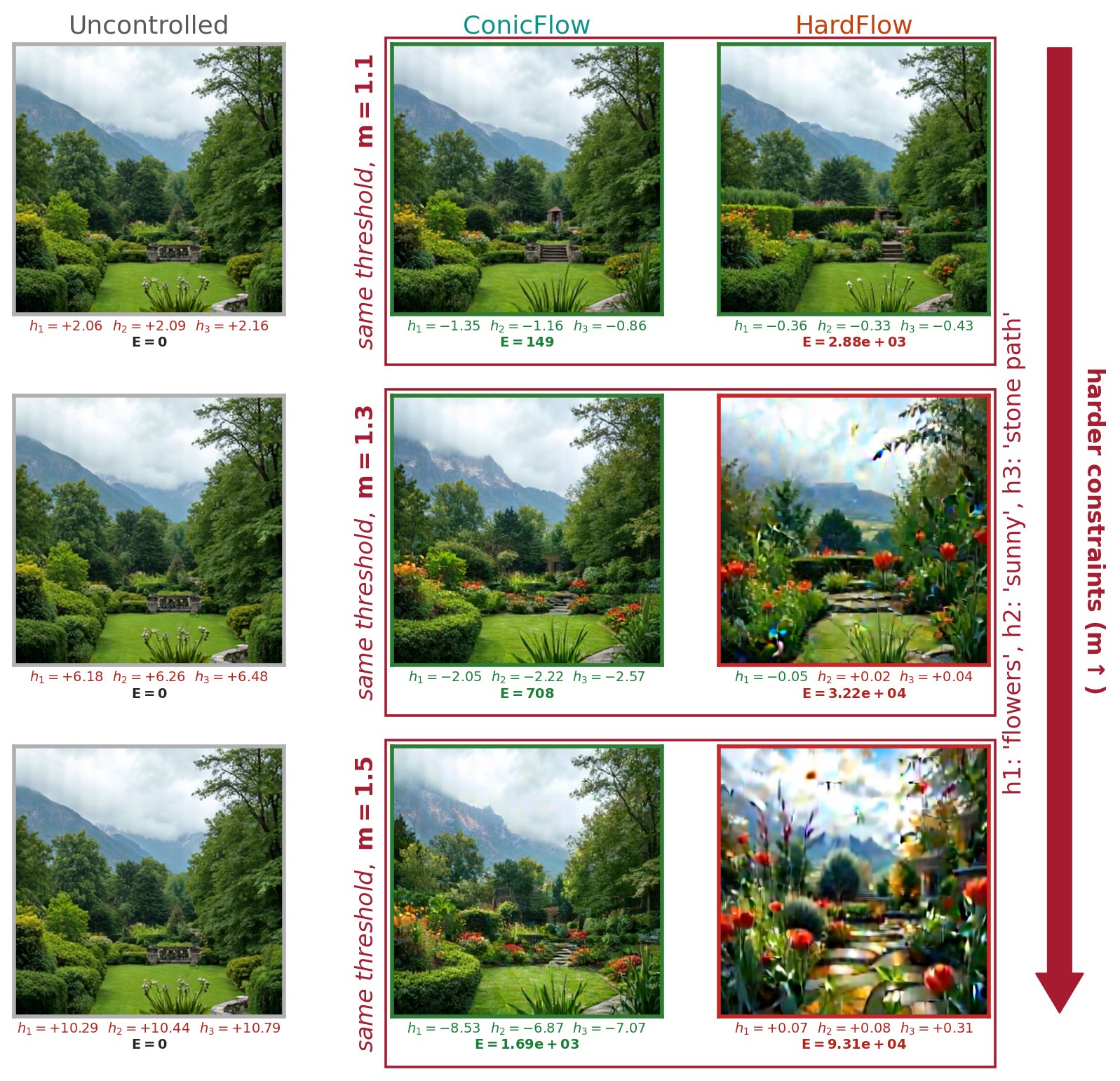}
    \caption{\emph{``a green garden on an overcast day''} $\to$ \emph{``a garden with red flowers; a garden in bright sunshine; a garden with a stone path''}}
    \label{fig:cgrid-t4_red_sun_path}
  \end{subfigure}
  \caption{More examples of constrained editing. Each panel is one (prompt, seed)
  instance. Instance details are in Table~\ref{tab:constrained-instances}.}
  \label{fig:constrained-grid-3}
\end{figure*}

\clearpage
\section{Ablations}
\label{sec:ablations}

Three questions are relevant to the comparison in Section~\ref{sec:experiments}: (1) whether \ours{}'s advantage is an artifact of the planning horizon we happened to choose, (2) whether it survives once the baselines are given the same compute, and (3) how much of it comes from where the horizon objective is evaluated. We address all three in this section on the editing task at $N=8$, on exactly the shared $\lambda$
grids of Section~\ref{sec:experiments}.

\subsection{Planning horizon}
\label{sec:abl-horizon}

\ours{} plans over $H$ steps and applies the first. We re-ran it at
$H \in \{1,4,6,8\}$ against the shipped $H=2$ configuration: four editing
prompts at seed $317$, $n_{\mathrm{inner}}=24$, $\mathrm{lr}=10^{-2}$, terminal
estimate by a single flow-map call (``jump''), six $\lambda$ per cell, so $24$
operating points per $H$. HardFlow and FMRG-J have no planning horizon and are
not re-run.

\begin{table}[h]
\caption{Planning horizon, editing, $N=8$, seed $317$; $\Delta r$, $E$ and time
           are medians over the $24$ operating points in each $H$. Wall clock is
           measured on shared cards and is indicative only.}
  \centering
  \small
  \begin{tabular}{cccc}
    \toprule
    $H$ & $\Delta r$ & $E$ & time \\
    \midrule
    $1$ & $3.98$ & $78.5$  & $207$\,s \\
    $2$ & $4.45$ & $\mathbf{68.7}$  & $249$\,s \\
    $4$ & $4.39$ & $70.6$  & $288$\,s \\
    $6$ & $4.67$ & $90.3$  & $321$\,s \\
    $8$ & $\mathbf{4.95}$ & $101.4$ & $264$\,s \\
    \bottomrule
  \end{tabular}
  \label{tab:abl-horizon}
\end{table}

Moving from a one-step to a two-step horizon improves both quantities at once:
$\Delta r$ rises $12\%$ and $E$ falls $12\%$. Beyond $H=2$ the two diverge.
Reward keeps rising, apart from a plateau at $H=4$, and reaches its maximum at
$H=8$, $24\%$ above $H=1$ for $24\%$ more NFE, where
$\mathrm{NFE} = n_{\mathrm{inner}}\big(N(N-1)/2 + H\big)$. Energy grows much
faster over the same range: from $H=2$ to $H=8$, $E$ rises $48\%$ while
$\Delta r$ gains only $11\%$. Longer horizons therefore buy reward with
disproportionately more control effort, and on the objective $H=2$ is the best
of the five.

\subsection{Matched wall clock}
\label{sec:abl-wallclock}

\ours{} is the only arm that backpropagates through the transformer, so at equal
$\lambda$ it also spends more compute: $239$\,s against $36$\,s for HardFlow
and $33$\,s for FMRG-J in Section~\ref{sec:experiments}. We removed that advantage by
retuning $n_{\mathrm{inner}}$ to $12$, $167$ and $13$ respectively, which brings
all three to $\approx\!122$\,s per run, and re-ran the four editing prompts at
seeds $317$ and $1289$ on the same shared $\lambda$ grids: $48$ operating points
at equal $\lambda$ \emph{and} equal compute.

\begin{table}[h]
\caption{Equal $\lambda$ and equal wall clock, editing, $N=8$, seeds $317$ and
           $1289$; $\Delta r$, $E$ and time are medians.}
  \centering
  \small
  \begin{tabular}{lcccc}
    \toprule
    & $n_{\mathrm{inner}}$ & time & $\Delta r$ & $E$ \\
    \midrule
    \ours{}  & $12$  & $124$\,s & $\mathbf{4.16}$ & $84.6$ \\
    HardFlow & $167$ & $123$\,s & $2.39$ & $555.2$ \\
    FMRG-J   & $13$  & $128$\,s & $1.84$ & $\mathbf{32.5}$ \\
    \bottomrule
  \end{tabular}
  \label{tab:abl-wallclock}
\end{table}

At matched compute \ours{} still delivers the largest reward gain, a median
$\Delta r$ of $4.16$ against $2.39$ for HardFlow and $1.84$ for FMRG-J
(Table~\ref{tab:abl-wallclock}). The two baselines fall short in opposite
directions. HardFlow, given more than three times its shipped compute, reaches
only $57\%$ of \ours{}'s reward gain and pays $6.6\times$ the control energy for
it ($E=555.2$ against $84.6$). FMRG-J spends the least energy of any arm,
$E=32.5$, but only because it barely moves the sample: its reward gain is under
half of \ours{}'s. \ours{} is the only arm that combines a large reward gain with
moderate energy, and extra compute alone does not move either baseline there.

\subsection{Objective evaluation point}
\label{sec:abl-endpoint}

The RH-T variant of \ours{} scores the horizon objective at the flow-map prediction of the terminal latent rather than at the planning-window state $x_H$ it actually reaches. We ran a comparison between RH-H and RH-T variants. We used four editing
prompts at 2 seeds, $n_{\mathrm{inner}}=24$,
$\mathrm{lr}=10^{-2}$, six $\lambda$ values per cell, so $48$
operating points. As seen in Table~\ref{tab:abl-endpoint}, evaluating at the terminal point is worth more than the compute. RH-H runs
$2.6\times$ faster but costs $1.54$ in
median $\Delta r$ while \emph{raising} median energy by a quarter, so the
RH-H is dominated by RH-T on both axes.

\begin{table}[h]
\caption{Where the horizon objective is evaluated, editing, $N=8$, seeds
           $317$ and $1289$; $\Delta r$, $E$ and time are
           medians.}
  \centering
  \small
  \begin{tabular}{llccc}
    \toprule
    & objective at & time & $\Delta r$ & $E$ \\
    \midrule
    \ours{}  & endpoint & $238$\,s & $\mathbf{4.79}$ & $\mathbf{68.7}$ \\
             & horizon  & $90$\,s  & $3.25$ & $86.9$ \\
    \midrule
    HardFlow & ---      & $36$\,s  & $2.00$ & $553.9$ \\
    FMRG-J   & ---      & $33$\,s  & $2.39$ & $87.0$ \\
    \bottomrule
  \end{tabular}
  \label{tab:abl-endpoint}
\end{table}

\end{document}